\documentclass[a4paper,12pt]{article}
\RequirePackage[OT1]{fontenc}
\RequirePackage{amsthm,amsmath,natbib,amssymb,graphicx}
\RequirePackage[colorlinks]{hyperref}

\numberwithin{equation}{section}
\theoremstyle{plain}

\newtheorem{lemma}{Lemma}[section]
\newtheorem{proposition}{Proposition}[section]

\def \E{{\mathbb E}}
\def \P{{\mathbb P}}

\newcommand{\rb}{\mathbb{R}}

\def \reg{ \mu }
\def \hw{ \hat{w}}
\def \w{ \mathbf{w}}
\def \tw{ \tilde{\mathbf{w}}}
\def \tl{ \tilde{ \lambda}}
\def \tmu { \tilde{ \mu}}
\def \ttau { \tilde{ \tau}}
\def \Q{ \mathbf{Q}}
\def \sign{ { \rm sign}}
\def \J{ \mathbf{J}}
\def \K{ \mathbf{K}}
\def \L{ \mathbf{L}}
\def \s{ \mathbf{s}}
\def \t{ \mathbf{t}}
\def \tsig{ \tau }

\def \lmin{ \lambda_{\min}}
\def \lmax{ \lambda_{\max}}

\def \lminQ{ \lmin(Q) }
\def \lmaxQ{ \lmax(Q) }
\def \idm{ { \rm I }}
\def \tr{ { \rm tr }}

\def \hate{ \hat{\varepsilon}}
\def \te{ \tilde{\varepsilon}}
\def \hata{ \hat{\alpha}}
\def \var{ { \rm var }}
\newcommand{\BEAS}{\begin{eqnarray*}}
\newcommand{\EEAS}{\end{eqnarray*}}
\newcommand{\BEA}{\begin{eqnarray}}
\newcommand{\EEA}{\end{eqnarray}}
\newcommand{\BEQ}{\begin{equation}}
\newcommand{\EEQ}{\end{equation}}
\newcommand{\BIT}{\begin{itemize}}
\newcommand{\EIT}{\end{itemize}}
\newcommand{\BNUM}{\begin{enumerate}}
\newcommand{\ENUM}{\end{enumerate}}
\newcommand{\BA}{\begin{array}}
\newcommand{\EA}{\end{array}}
\newcommand{\mysec}[1]{Section~\ref{sec:#1}}

\newcommand{\eq}[1]{Eq.~(\ref{eq:#1})}
\newcommand{\myfig}[1]{Figure~\ref{fig:#1}}

\newcounter{hyp}
\newenvironment{hyp}{\refstepcounter{hyp}\begin{itemize}
  \item[ \hspace*{.4cm} ({\bf{A}\arabic{hyp}})]}{\end{itemize}}

\newcommand{\hypref}[1]{({\bf{A}\ref{hyp:#1}})}
\newcommand{\hypreff}[2]{({\bf{A}\ref{hyp:#1}-\ref{hyp:#2}})}

\newcommand{\minf}[1]{ {\rm m}(#1) }

\author{ Francis Bach \\
Willow Project-team \\
Laboratoire d'Informatique de l'Ecole Normale Sup\'erieure \\
(CNRS/ENS/INRIA UMR 8548)\\
45, rue d'Ulm, 75230 Paris, France \\
\url{francis.bach@mines.org}}
 \title{Model-Consistent Sparse Estimation\\ through the Bootstrap}

\begin{document}

\maketitle

\begin{abstract}

We consider the least-square linear regression problem with regularization by the $\ell^1$-norm, a problem usually referred to as the Lasso.
In this paper, we first present a detailed asymptotic analysis of model consistency of the Lasso in low-dimensional settings. For various decays of the regularization parameter, we compute asymptotic equivalents of the probability of correct model selection. For a specific rate decay, we show that the Lasso selects all the variables that should enter the model with probability tending to one exponentially fast, while it selects all other variables with strictly positive probability. We show that this property implies that if we run the Lasso for several bootstrapped replications of a given sample, then intersecting the supports of the Lasso bootstrap estimates leads to consistent model selection. This novel variable selection procedure, referred to as the Bolasso, is extended to high-dimensional settings by a provably consistent two-step procedure.

\end{abstract}
 
\section{Introduction}
\label{sec:introduction}

 Regularization by the $\ell^1$-norm has attracted a lot of interest in recent years in statistics, machine learning and signal processing. In the context of least-square linear regression, the problem is usually referred to as the \emph{Lasso}~\cite{lasso} or \emph{basis pursuit}~\cite{chen}.
Much of the early effort has been dedicated to algorithms to solve the optimization problem efficiently, either through first-order methods~\cite{shooting,descent}, or through homotopy methods that leads to the entire regularization path (i.e., the set of solutions for all values
 of the regularization parameters) at the cost of a single matrix inversion~\cite{markowitz,osborne,lars}. 

A well-known property of the 
regularization by the $\ell^1$-norm is the \emph{sparsity} of the solutions, i.e.,  it leads to  loading vectors with many zeros, and thus performs model selection on top of regularization. Recent works
\cite{Zhaoyu,yuanlin,zou,martin} have looked precisely
at the model consistency of the Lasso, i.e., if we know that the data were generated from a sparse loading vector, does
the Lasso actually recover the sparsity pattern when the number of observations grows? In the case of a fixed number of covariates (i.e., low-dimensional settings), the Lasso does recover
the sparsity pattern if and only if a certain simple condition on the generating covariance matrices is satisfied~\cite{yuanlin}. 
In particular, in low correlation settings,
the Lasso is indeed consistent. However, in presence of strong correlations between relevant variables and irrelevant variables, the Lasso cannot be model-consistent, shedding light on
potential problems of such procedures for variable selection. Various extensions of the Lasso have been designed to fix its inconsistency, based on thresholding~\cite{yuinfinite}, data-dependent weights~\cite{zou,yuanlin,adaptivehighdim} or two-step procedures~\cite{relaxedlasso}. The main contribution of this paper is to propose and analyze an alternative approach based on resampling. Note that recent work~\cite{stability} has also looked at resampling methods for the Lasso, but focuses on resampling the weights of the $\ell^1$-norm rather than resampling the observations (see \mysec{support} for more details).

In this paper, we first derive a detailed asymptotic analysis of sparsity pattern selection of the Lasso estimation procedure, that extends previous analysis~\cite{Zhaoyu,yuanlin,zou} by focusing on a specific decay of the regularization parameter. Namely, in \emph{low-dimensional} settings where the number of variables $p$ is much smaller than the number of observations $n$, we show that when the decay of $n$ is proportional to $n^{-1/2}$, then the Lasso will select all the variables that should enter the model (the \emph{relevant} variables) with probability tending to one exponentially fast with~$n$, while it selects all other variables (the \emph{irrelevant} variables) with strictly positive probability. If several datasets generated from the same distribution were available, then the latter property  would suggest to consider the intersection of the supports of the Lasso estimates for each dataset: all relevant variables would always be selected for all datasets, while irrelevant variables would enter the models randomly, and intersecting the supports from sufficiently many different datasets would simply eliminate them. However, in practice, only one dataset is given; but resampling methods such as the \emph{bootstrap} are exactly dedicated to mimic the availability of several datasets by resampling from the same unique dataset~\cite{efron}. In this paper, we show that when using the bootstrap and intersecting the supports, we actually get a consistent model estimate, \emph{without} the consistency condition required by the regular Lasso. We refer to this new procedure as the \emph{Bolasso} (\textbf{bo}otstrap-enhanced \textbf{l}east \textbf{a}b\textbf{s}olute \textbf{s}hrinkage \textbf{o}perator).  Finally, our Bolasso framework could be seen as a voting scheme applied to the supports of the bootstrap Lasso estimates;
however, our procedure may rather be considered as a consensus combination scheme,
as we keep the (largest) subset of variables on which \emph{all} regressors agree in terms of variable selection, which is in our case provably consistent and also allows to get rid of a potential additional hyperparameter.

We consider the two usual ways of using the bootstrap in regression settings, namely bootstrapping pairs and bootstrapping residuals~\cite{efron,freedman}. In \mysec{support}, we show that the two types of bootstrap lead to consistent model selection in low-dimensional settings. Moreover, in \mysec{simulations}, we provide empirical evidence that in high-dimensional settings, bootstrapping pairs does not lead to consistent estimation, while bootstrapping residuals still does. While we are currently unable to prove the consistency of bootstrapping residuals in high-dimensional settings, we prove in \mysec{highdim} the model consistency of a related two-step procedure: 
  the Lasso is run once on the original data, with a larger regularization parameter, and then bootstrap replications (pairs or residuals) are run within the support of the first Lasso estimation. We show in \mysec{highdim} that this procedure is consistent. In order to do so, we consider new sufficient conditions for the consistency of the Lasso, which do not rely on sparse eigenvalues~\cite{yuinfinite,ch_zhang}, low correlations~\cite{bunea,lounici} or finer conditions~\cite{tsyb,cohen,tong_zhang}. In particular, our new assumptions allow to prove that the Lasso will select not only a few variables when the regularization parameter is properly chosen, but always the same variables with high probability.
  
In \mysec{algorithms}, we derive efficient algorithms for the bootstrapped versions of the Lasso. When bootstrapping pairs, we simply run an efficient homotopy algorithm, such as \emph{Lars}~\cite{lars}, multiple times; however, when bootstrapping residuals, more efficient ways may be designed to obtain a running time complexity which is less than running Lars multiple times.
Finally, in \mysec{experiments-low} and \mysec{experiments-high}, we illustrate our results on synthetic examples, in low-dimensional and high-dimensional settings. This work is a follow-up to earlier work~\cite{bolasso}: in particular, it refines and extends the analysis to high-dimensional settings and boostrapping of the residuals.

\paragraph{Notations}
For $x \in \rb^p$ and $q>0$, we  denote by $\| x\|_q$ its $\ell^q$-norm, defined as
$\| x\|_q^q = \sum_{i=1}^p |x_i|^q$. We also denote by $\| x\|_\infty =
\max_{i \in \{1,\dots,p\}} | x_i |$ its $\ell^\infty$-norm. For rectangular matrices $A$, we denote
by $\| A \|_2$ its largest singular value, $\|A\|_\infty$ the largest magnitude of all its elements, and $\|A\|_F = ( \tr A^\top A)^{1/2}$ its Frobenius norm. We let denote $\lmaxQ$ and $\lambda_{\min}(Q)$ the largest and smallest eigenvalue of a symmetric matrix $Q$. 

For   $a \in \rb$, $\sign(a)$ denotes the sign of $a$, defined as $\sign(a)=1$ if $a>0$, $-1$ if $a<0$, and $0$ if $a=0$. For a vector $v \in \rb^p$, $\sign(v) \in \{-1,0,1\}^p$ denotes the   vector of signs of elements of $v$. Given a set $H$, $1_H$ is the indicator function of the set $H$. We also denote, for $w\in \rb^p$,
by $\minf{w} = \min_{j \in \{1,\dots,p\}, \ w_j \neq 0 } | w_j|$, the smallest (in magnitude) of non-zero elements of $w$. 
 
 Moreover, given a vector $v \in \rb^p$ and a subset $I$ of $\{1,\dots,p\}$, $v_I$ denotes the vector in $\rb^{{\rm Card}(I)}$ of elements of $v$ indexed by $I$. Similarly, for a matrix $A \in \rb^{p \times p}$, $A_{I,J}$  denotes the submatrix of  $A$ composed of elements of $A$ whose rows are in $I$ and columns are in $J$. Moreover, $|J|$ denotes the cardinal of the set $J$. 
For a positive definite matrix $Q$ of size $p$, and two disjoint subsets of indices $A$ and $B$ included in $\{1,\dots,p\}$, we denote $Q_{A,A|B}$ the matrix $Q_{A,A} - Q_{A,B}Q_{B,B}^{-1} Q_{B,A}$, which is the conditional covariance of variables indexed by $A$ given variables indexed by $B$, for a Gaussian vector with covariance matrix $Q$. Finally, we let denote $\P$ and $\E$ general probability measures and expectations.

\paragraph{Least-square regression with $\ell^1$-norm penalization}
Throughout this paper,  we consider $n$ pairs of observations   $(x_i,y_i) \in \rb^p \times \rb$, $i=1,\dots,n$.
The data are given in the form of a vector $y \in \rb^n$ and a design matrix $X \in \rb^{n \times p}$. 
We consider the normalized square loss function 
$$\frac{1}{2n}
\sum_{i=1}^n ( y_i - w^\top x_i)^2 = \frac{1}{2n} \| y  - X w \|_2^2,
$$
 and the regularization by the $\ell^1$-norm. That is, we look at the following convex optimization problem~\cite{lasso,chen}:
\BEQ
\label{eq:lasso} \min_{w \in \rb^p}  \frac{1}{2n} \|y  - X w \|_2^2 + \reg  \| w\|_1,
\EEQ
where $\reg \geqslant 0$ is the regularization parameter.
We   denote by $\hw$ any global minimum of \eq{lasso}, and $\hat{J} = \{ j  \in \{1,\dots,p\}, \  \hat{w}_j  \neq 0\}$ the support of $\hw$. 

 In this paper, we consider two settings, depending on the value of the ratio of $p/n$. When this ratio is much smaller than one, as in \mysec{lowdim}, we refer to this setting as low-dimensional estimation, while in other cases, where this ratio is potentially much larger than one, we refer to this setting as a high-dimensional problem (see \mysec{highdim}).

\section{Low-Dimensional Asymptotic Analysis}
\label{sec:lowdim}

 We make the following ``fixed-design'' assumptions:
\begin{hyp}
\label{hyp:model}
\emph{Linear model with i.i.d. additive noise}:
 $y = X \w + \varepsilon$,  where $\varepsilon$ is a vector with independent components, identical distributions and zero mean; $\w$ is sparse, with $ \s = \sign(\w)$ and support $\J = \{ j, \w_j \neq 0\}$.
 \end{hyp}
 
\begin{hyp}
\label{hyp:var}
\emph{Subgaussian noise}:
there exists $\tau>0$ such that for all $j \in \{1,\dots,p\}$ and $s \in \rb$,
$\E e^{ s \varepsilon_j } \leqslant e^{\frac{1}{2} \tsig^2 s^2}$. Moreover, the variances of $\varepsilon_j$ are greater than $\sigma^2 >0$.
\end{hyp}

\begin{hyp}
\label{hyp:bounded}
\emph{Bounded design}:
 For all $i\in \{1,\dots,n\}$, $\| x_i \|_\infty \leqslant M$.
\end{hyp}

\begin{hyp}
\label{hyp:inv}
\emph{Full rank design}:
 The matrix   $Q =  \frac{1}{n} X^\top X \in \rb^{ p \times p} $ is invertible.
\end{hyp}

Throughout this paper, we consider normalized constants $\tw = \w M / \sigma$ (normalized population loading vector), $\tmu = \mu / M \sigma$ (normalized regularization parameter), $\tl = \lminQ / M^2$ (condition number of the matrix of second-order moments), and $\ttau = \tau / \sigma$  (always larger than one, and equal to one if and only if the noise is Gaussian, see Appendix~\ref{app:quadratic}).

With our assumptions, the problem in \eq{lasso} is equivalent to
\BEQ
\label{eq:lasso-eq} \min_{w \in \rb^p}  \frac{1}{2} (w -\w)^\top Q (w -\w)- q^\top(w-\w) + \reg  \| w\|_1,
\EEQ
where $Q = \frac{1}{n} X^\top X \in \rb^{p \times p} $ and $q = \frac{1}{n} X^\top \varepsilon \in \rb^p$. Note that under assumption \hypref{inv}, there is a unique solution to \eq{lasso} and \eq{lasso-eq}, because the associated objective functions are then strongly convex.
Moreover, assumption \hypref{inv} implies that $p\geqslant n$, that is, we consider in this section, only ``low-dimensional'' settings  (see \mysec{highdim} for extensions to high-dimensional settings).

In this section, we detail the asymptotic behavior of the (unique) Lasso estimate~$\hat{w}$, both in terms of the difference in norm with the population value $\w$ (i.e., regular
consistency) and of the \emph{sign pattern} $\sign(\hat{w})$, for all asymptotic behaviors of the regularization parameter $\mu$. Note that information about the sign pattern includes information about the \emph{support} $\hat{J}$, i.e., the indices $j \in \{1,\dots,p\}$ for which $\hat{w}_j$ is different from zero; moreover, when $\hat{w}$ is consistent, consistency of the sign pattern is in fact equivalent to the consistency of the support.  We assume that $p$ is fixed and $n$ tends to infinity, the regularization parameter $\mu$ being considered as a function of $n$ (though we still derive non-asymptotic bounds).

Note that for some of our results to be non trivial, we require that $p$ is not only small compared to $n$, but that a power of $p$ is small compared to~$n$. Technically, this is due to the application of multivariate Berry-Esseen inequalities (reviewed in Appendix~\ref{app:berry}), which could probably be improved to obtain smaller powers.

We consider five mutually exclusive possible situations which explain various portions of the regularization path; many of these results appear elsewhere~\cite{yuanlin,Zhaoyu,fu,zou,grouplasso,lounici} but some of the finer results presented below are new (in particular most non-asymptotic results and the $n^{-1/2}$-decay of the regularization parameter in \mysec{medium}). 
These results are illustrated on synthetic examples in \mysec{experiments-low}.

Note that all exponential convergences have a rate that depends on $\minf{\w}$, i.e., the smallest (in magnitude) non zero element of the generating sparse vector $\w$. Thus, we assume a sharp threshold in order to have a fast rate of convergence. Considering situations without such a threshold, which would notably require to estimate errors in model estimation (and not simply exactly correct or incorrect), is out of the scope of this paper (see, e.g.,~\cite{ch_zhang}).

\subsection{Heavy regularization}
If $\mu$ is large enough, then $\hat{w}$ is equal to zero with probability tending to one exponentially fast in $n$. Indeed, we have (see proof in Appendix~\ref{app:proofs-lowdim-heavy}):
\begin{proposition} 
\label{prop:lowdim-heavy}
Assume \hypreff{model}{inv}.
If $\tmu \geqslant 2 \| \tw \|_1$, then the probability that $\hat{w} =0$ is greater than
$ 1 - 2p \exp \left( - \frac{  n \tmu^2}{8 \ttau^2}  \right)$.
\end{proposition}

A well-known property of homotopy algorithms for the Lasso (see, e.g.,~\cite{lars}) is that if $\mu$ is large enough, then $\hat{w}=0$. This proposition simply provides a uniform probabilistic bound.

\subsection{Fixed regularization}

If $\mu$ tends to a finite strictly positive constant~$\mu_0$, then $\hat{w}$ converges in probability to the unique
global minimum of the noiseless objective function $\frac{1}{2} (w-\w)^\top Q (w-\w) + \mu_0 \| w\|_1$. Thus, the estimate $\hat{w}$ never converges in probability to $\w$, while the sign pattern tends to the one of the previous global minimum, which may or may not be the same as the one of the noiseless problem $\w$.
  It is thus possible, though not desirable, to have sign consistency without regular consistency.  See \cite{grouplasso} for examples and simulations of a similar behavior for the group Lasso.
  
  All convergences are exponentially fast in $n$ (proof in Appendix~\ref{app:proofs-lowdim-fixed}). Note that
  here and in the next regime (Proposition~\ref{prop:high}), we do not take into account the pathological cases where the sign pattern of the limit in unstable, i.e., the limit is exactly at a hinge point of the regularization path. When this occurs, all associated sign patterns are attained with positive probability (see also \mysec{highdim}).

\begin{proposition}
\label{prop:fixed}
\label{prop:lowdim-fixed}
Assume \hypreff{model}{inv}. Let $\mu_0> 0$ and $\tmu_0 = \mu_0 / M /sigma$.  Let $w_0$ be the unique solution of
$ \min_{v \in \rb^p} \frac{1}{2} (v-\w)^\top Q (v-\w) + \mu_0 \| v\|_1$. Then, if $| \tmu - \tmu_0 | \leqslant \frac{ \tl  }{4 p^{1/2}} \beta $, we have:
$$\P( \| \hat{w} -  w_0\|_2 \geqslant \beta \sigma / M ) 
\leqslant 
2p  \exp\left( - \frac{\tl^2   \beta^2}{32  \ttau^2 } \frac{n}{p}\right)
\leqslant 2p  \exp\left( - \frac{ ( \tmu - \tmu_0)^2 }{2  \ttau^2 } n \right).
$$
Moreover, assume the minimum $v$ occurs away from a hinge point of the regularization path, i.e., there exists $\eta>0$ such that 
for all $j \in \{1,\dots,p\}$,  $v_j = 0$ implies $ 
| (Q ( w_0  - \w) )_j | \leqslant \mu_0 - \eta M \sigma  $.
 If $  |\tmu - \tmu_0|  \leqslant 
 \tl \min \{  \eta/4, \minf{w_0 M/\sigma} \}$, then
$$
 \P( \sign(\hat{w}) \neq \sign(w_0 ) )
 \leqslant 
 2p
\exp 
\left( - \frac{\tl^2 }{\ttau^{2}}  \min\{ \eta^2 / 4, \minf{w_0 M/\sigma}^2 \} \frac{n}{p}\right).
 $$
\end{proposition}

The proposition above makes no claim in the situation where $\mu$ tends to zero. As we now show, this depends on the rate of decay of $\mu$, slower, faster, or exactly at the rate $n^{-1/2}$.

\subsection{High regularization}
\label{sec:high}
 If $\mu$ tends to zero slower than $n^{-1/2}$, then $\hat{w}$ converges in probability to $\w$ (regular consistency) and the sign pattern converges to the sign pattern of the global minimum of a local noiseless objective function
$\frac{1}{2} \Delta^\top Q \Delta + \Delta_\J^\top \sign(\w_\J) + \| \Delta_{\J^c}\|_1$, the convergence being exponential in $\mu^2 n$   (see proof in Appendix~\ref{app:proofs-lowdim-high}). The local noiseless problem in \eq{AH} is simply obtained by a first-order expansion of the Lasso objective function around $\w$~\cite{fu,yuanlin}.

\begin{proposition}
\label{prop:high}

\label{prop:lowdim-high}
Assume \hypreff{model}{inv}.  Let $\Delta$ be the unique solution of
\BEQ
\label{eq:AH}
\min_{\Delta \in \rb^p}  \frac{1}{2} \Delta^\top Q \Delta + \Delta_\J^\top \sign(\w_\J) + \| \Delta_{\J^c}\|_1. 
\EEQ
Assume that $\tmu \leqslant  \frac{ \minf{\tw} \tl }{2 p^{1/2}}$.
 We have:
$$
\P(  \| \hat{w} - \w - \mu \Delta \|_2 \geqslant \beta \sigma / M)
\leqslant 2p
\exp\left(   -  \frac{\tl^2   \beta^2}{8  \ttau^2}  \frac{n}{p}   \right).
$$
Moreover, assume the minimum $\Delta$ of \eq{AH} occurs away from a hinge point of the regularization path, i.e., there exists $\eta>0$ such that for all $j \in \J^c$, $\Delta_j = 0$ implies $ | (Q\Delta)_j  | \leqslant   1 - \eta $. Then,
 $$
 \P( \sign(\hat{w}) \neq \sign(\w + \mu \Delta ) )\leqslant 
2p \exp\left( - \frac{\minf{\tw} \tl^2 }{8 \ttau^2} \frac{n}{p}\right)
 + 2p  \exp
 \left( -  A  \tmu^2 \frac{n}{p}  \right),
 $$
 where $A =  \ttau^{-2}\tl \min\{ \tl \minf{M^2 \Delta}^2/2, \eta^2 / 8 \} $.
  \end{proposition}

Note that the sign pattern of $\w + \mu \Delta$ is equal to the population sign vector $\s=\sign(\w)$ if and only if the 
problem in \eq{AH} has a solution where $\Delta_{\J^c}$ is equal to zero. A short calculation shows that this occurs if and only if the
following consistency condition is satisfied~\cite{glasso,Zhaoyu,yuanlin,zou,martin}:
\BEQ
\label{eq:cond}
 \| Q_{\J^c ,\J} Q_{\J ,\J}^{-1} \sign(\w_\J) \|_\infty \leqslant 1. 
 \EEQ
Thus, if \eq{cond} is satisfied strictly---which implies that we are not at a hinge point of \eq{AH}---the probability of correct sign estimation is tending to one, and to zero if \eq{cond} is not satisfied (see~\cite{yuanlin} for precise statements when there is equality).
Moreover, when \eq{cond} is satisfied strictly, Proposition~\ref{prop:high} gives an upper bound on the probability of not selecting the correct pattern $\J$.

The first three regimes are not unique to low-dimensional settings; we show in \mysec{highdim} the corresponding proposition related to Proposition~\ref{prop:lowdim-high},
 for high-dimensional settings. However, the last two regimes ($\mu$ tending to zero at rate $n^{-1/2}$ or faster) are specific to low-dimensional settings.

\subsection{Medium regularization}
\label{sec:medium}

 If $\mu n^{1/2}$ is bounded from above and from below, then we show that the sign pattern of $\hat{w}$ agrees on $\J$ with the one of $\w$  with probability tending to one exponentially fast in $n$ (Proposition~\ref{prop:lowdim-medium}), while for all sign patterns consistent on $\J$ with the one of $\w$, the probability of obtaining this pattern is tending to a limit in $(0,1)$ (in particular strictly positive); that is, all sign patterns consistent  with $\w$ on the relevant variables (i.e., the ones in $\J$) are possible with positive probability (Proposition~\ref{prop:lowdim-medium2}). The convergence of this probability follows a  rate of $n^{-1/2}$
(see proof in Appendix~\ref{app:proofs-lowdim-medium} and~\ref{app:proofs-lowdim-medium2}). Note the difference with earlier results~\cite{bolasso} obtained for random designs.

\begin{proposition}
\label{prop:prop1}
\label{prop:lowdim-medium}

Assume \hypreff{model}{inv}
and $\tmu \leqslant  \frac{ \minf{\tw} \tl }{2 p^{1/2}}$. Then for any sign pattern $s\in \{-1,0,1\}^p$ such that $s_\J = \sign(\w_\J)$,  there
exists  $f(s,n^{1/2}\mu p^{1/2}) \in (0,1)$, such that:
$$| \P( \sign(\hat{w}) = s ) - f(s,n^{1/2}\mu p^{1/2})| \leqslant 
    \frac{ 4  C^{\rm BE}_1 \ttau^3  }
 { \tl^{1/2}} \frac{p^{2}}{ n^{1/2}} + 
2p \exp\left( - \frac{\minf{\tw} \tl^2  }{8 \ttau^2  } \frac{n}{p}\right).
$$
\end{proposition}

\begin{proposition}
\label{prop:prop2}
\label{prop:lowdim-medium2}
Assume \hypreff{model}{inv} and $\tmu \leqslant  \frac{ \minf{\tw} \tl }{2 p^{1/2}}$. Then,
 for any pattern $s\in \{-1,0,1\}^p$ such that  $s_\J \neq \sign(\w_\J)$, $$
 \P( \sign(\hat{w}) = s )  \leqslant
2p \exp\left( - \frac{\minf{\tw} \tl^2 }{8\ttau^2 }\frac{ n}{p}\right).
$$
\end{proposition}
The positive real numbers $C^{\rm BE}_1 $ and $C^{\rm BE}_2$ are universal constants related to multivariate Berry-Esseen inequalities (see Appendix~\ref{app:berry} for more details). From the proof in
Appendix~\ref{app:proofs-lowdim-medium}, the constant $f(s,c)$ has specific behaviors when $c = \mu n^{1/2}p^{1/2}$ is small or large: on the one hand, if $c$ tends to infinity, then we tend to the bevahior of the previous section, that is, $f(s,c)$ tends to one if $s$ is the limiting pattern in Proposition~\ref{prop:lowdim-high} and zero otherwise. On the other hand, if $c$ tends to 0, $f(s,c)$ tends to one if $s$ has no zeros, and zero otherwise (see next section).

The last two propositions state that the relevant variables are \emph{stable}, i.e., we get all relevant variables with probability tending to one \emph{exponentially fast}, while we get exactly get all other patterns with probability tending to a limit
\emph{strictly} between zero and one. This stability of the relevant variables is the source of the intersection arguments outlined in \mysec{support}.

Note that Proposition~\ref{prop:lowdim-medium} makes non-trivial statements only for $n$ larger than $p^4$; the fourth power is due to the application of Berry-Esseen inequalities, and could be improved.

\subsection{Low regularization}
 If $\mu$ tends to zero faster than $n^{-1/2}$, then $\hat{w}$ is consistent (i.e., converges in probability to $\w$) but the support of $\hat{w}$ is equal to $\{1,\dots,p\}$ with probability tending to one (the signs of variables in $\J^c$ may then be arbitrarily negative or positive). That is, the $\ell^1$-norm has no sparsifying effect. We obtain two different bounds, with different scalings in $p$ and $n$
 (see proof in Appendix~\ref{app:proofs-lowdim-low}):
 
\begin{proposition}
\label{prop:lowdim-low}
Assume \hypreff{model}{inv} and $\tmu \leqslant  \frac{ \minf{\tw} \tl }{2 p^{1/2}}$. Then the probability of having at least one zero variable is smaller than
$  3^p \!
\left( \!    C^{\rm BE}_1    \frac{ 4  \ttau^3  }
 { \tl^{1/2}} \frac{p^{2}}{ n^{1/2}}   \!+  \! \frac{\tmu n^{1/2}}{\tl^{1/2}} \! \right)
$ and $ \frac{\tmu n^{1/2} p }{\tl^{1/2}}  +    \frac{10 C^{\rm BE}_2}{\ttau^3 \tl} \frac{p^{7/2}}{\tmu n}     +   C^{\rm BE}_2    \frac{ 4  \ttau^3  }
 { \tl^{1/2}} \frac{p^{3}}{ n^{1/2}}    + 2 |\J| \exp\left( - \frac{\minf{\tw} \tl^2 }{8\ttau^2 } \frac{n}{p}\right) $.
\end{proposition}

The first bound simply requires that $\mu$ tends to zero faster than $n^{-1/2}$, but the constant is exponential in $p$, while the second bound required that $\mu$ does not tend to zero too fast, i.e., between $n^{-1/2}$ and $n^{-1}$ (with constants polynomial in $p$). As shown in Appendix~\ref{app:proofs-lowdim-high}, the two bounds correspond to two different applications of Berry-Esseen inequalities, one for all the possible $3^p$ sign patterns, one using a detailed analysis of the non-selection of a given variable (see \mysec{finer}). 
We are currently exploring the possibility of having a bound that shares the positive aspects of our two bounds---polynomial in $p$ and without the term $(\tmu n)^{-1}$.

\vspace*{.5cm}

Among the five previous regimes, the only ones with consistent estimates (in norm) and a sparsity-inducing effect are $\mu$ tending to zero and $\mu n^{1/2}$ tending to a finite or infinite limit. When $\mu n^{1/2}$ tends to infinity,
we can only hope for model consistent estimates if the consistency condition in \eq{cond} is satisfied. This somewhat disappointing result for the Lasso has led to various improvements on the Lasso to ensure model consistency even when \eq{cond} is not satisfied~\cite{yuanlin,zou,relaxedlasso}. Those are based on adaptive weights based on the non regularized least-square estimate or two-step procedures. We propose in \mysec{support}   alternative
 ways which are based on resampling.
 Before doing so, we derive in the next section finer results that allows to consider the presence or absence in the support set $\hat{J}$ of  a specific variable without considering all corresponding consistent sign patterns.

 \subsection{Probability of not selecting a given variable}
\label{sec:finer}
We can lower and upper bound the probability of not selecting a certain irrelevant variable in  $\J^c$
(see proof in Appendix~\ref{app:proofs-lowdim-missing})---see Proposition~\ref{prop:lowdim-medium2} for a related proposition for relevant variables in $\J$:

 \begin{proposition}
 \label{prop:missingone}
  \label{prop:lowdim-correct}
  Assume \hypreff{model}{inv} and $\tmu \leqslant  \frac{ \minf{\tw} \tl }{2 p^{1/2}}$.  Let ${j} \in \J^c$.
 We have:
 \BEAS
 \P( {j} \in \hat{J}) &\!\! \!\geqslant \!\!\! &  \frac{   \tmu n^{1/2} / 4  }{
1+  \tmu n^{1/2} / 2 \tl^{1/2}} 
 \exp \left( -  \frac{ 2 \tmu^2 }{\tl^2} n p    \right) -   \frac{10 C^{\rm BE}_2 }{\ttau^3 \tl^{1}} \frac{p^{5/2}}{\tmu n}    \! -\!   C^{\rm BE}_2    \frac{ 4  \ttau^3  }
 { \tl^{1/2}} \frac{p^{2}}{ n^{1/2}}   , \\
 \P( {j} \in \hat{J}) & \!\! \!\leqslant \! \!\!&   \frac{\tmu n^{1/2}}{\tl^{1/2}} + 
 \frac{8 C^{\rm BE}_2 }{\ttau^3 \tl} \frac{p^{5/2}}{\tmu n}    +   C^{\rm BE}_2    \frac{ 4  \ttau^3  }
 { \tl^{1/2}} \frac{p^{2}}{ n^{1/2}} .
    \EEAS
   \end{proposition}
 This novel proposition allows to consider ``marginal'' probabilities of selecting (or not selecting) a given variable, without considering all consistent sign patterns associated with the selection (or non-selection) of that variable). Note that it makes interesting claims only when $\mu n^{1/2}$ is bounded from above and below (for the lower bound) and when $\mu n^{1/2}$ tends to zero, while $\mu n$ tends to infinity (for the upper bound).
 
 \section{Support Estimation by Intersection}
 \label{sec:support}
 The results from \mysec{medium} exactly show that under suitable choices of the regularization parameter $\mu$, the relevant variables are stable while the irrelevant are unstable, leading to several intersecting arguments to keep only the relevant variables. We first consider the irrealistic situation where we have multiple independent copies, then we consider splitting a dataset in several pieces, and we finally present two usual types of bootstrap (pairs and residuals). Note that an alternative approach is to resample the columns of the design matrix instead of its rows, i.e., draw random weights for each variable from a well-chosen distribution~\cite{stability}.

 The analysis of support estimation is essentially the same for all methods and is based on the following argument: we consider $m$ ``replications'', and $\hat{J}^1,\dots, \hat{J}^m$ the associated active sets. The replications are assumed independent given the original data (i.e., the vector of noise $\varepsilon$). We let denote $\hat{J}^\cap = \bigcap_{i=1}^m \hat{J}^i$ the estimate of the active set (given the original data).
 Once the active set is found, the final estimate of $w$ is obtained by the unregularized least-square estimate, restricted to the estimated active set.

We can upper bound the probability of incorrect pattern selection as follows:
\BEAS
\P( \hat{J}^\cap \neq \J )
& \leqslant & \P( \J^c \cap  \hat{J}^\cap \neq \varnothing )
+
\P( \J \cap  (\hat{J}^\cap)^c \neq \varnothing ) , \\
& \leqslant & \sum_{j \in \J^c} 
\P(  \forall i \in \{1,\dots,m\}, j \in  \hat{J}^i )
+
\P\left(   \bigcup_{i=1}^m
\left[  (\hat{J}^i)^c \cup \J   \right] \neq \varnothing \right), \\
& \leqslant & 
 \sum_{j \in \J^c} \E ( 
\P(  j \in  \hat{J}^\ast | \varepsilon)^m )
+ m 
\P(      (\hat{J}^\ast)^c \cup \J   \neq \varnothing ),
\EEAS
where $ \hat{J}^\ast$ denotes a generic support obtained from one replication.
We now need to upper bound the probability
$\P(      (\hat{J}^\ast)^c \cup \J   \neq \varnothing ) $ of forgetting at least a relevant variable $j\in \J$, and also the probability 
$\P(  j \in  \hat{J}^\ast | \varepsilon)$ that a replication does not include a given irrelevant variable $j \in \J^c$ (given the original data). The first term will always drop as the number of replications gets larger, while the second term increases, leading to a natural trade-off for the choice of the number $m$ of replications. This is to be contrasted with usual applications of the boostrap where $m$ is taken as large as computationally feasible.

   \subsection{Multiple independent copies}
   \label{sec:lowdim-copies}
   Let us assume for a moment that we have $m$ independent copies of similar datasets, with potentially different fixed designs but same noise distribution. We then have $m$ different active sets and we
 denote by $\hat{J}^\cap$ the intersection of the $m$ active sets.
We have the following upper bound on the probability of non selecting the correct pattern
(see proof in Appendix~\ref{app:proofs-lowdim-copies})

 \begin{proposition}
  \label{prop:lowdim-copies}
  Assume \hypreff{model}{inv} for $m$ independent datasets with same noise distribution,
  and $\tmu \leqslant  \frac{ \minf{\tw} \tl }{2 p^{1/2}}$. If $c=\tmu n^{1/2} p^{1/2} >0 $, $\tmu \leqslant  \frac{ \minf{\tw} \tl }{2 p^{1/2}}$ and $n \geqslant p^6 g(c)$, then there exists 
  $f(c ) > 0 $ such that
    $$
  \P( \hat{J}^\cap \neq \J )
  \leqslant p e^{-f(c) m p^{-1/2}} + 2p m   \exp\left( - \frac{\minf{\tw} \tl^2 }{8\ttau^2 } \frac{n}{p}\right). 
  $$
    \end{proposition}
    From the proof of Proposition~\ref{prop:lowdim-copies} in Appendix~\ref{app:proofs-lowdim-copies}, we can get the detailed behavior of $f(c)$ around $c=0$ and $c=\infty$: it goes to zero in both cases, i.e., we actually need (in the bound) a regularization parameter that is proportional to $n^{-1/2}$.

Moreover, we get an exponential convergence rate in $n$ and $m$, where we have two parts: one that states that the number of copies should be as large as possible to remove irrelevant variables (left part), and one that states that $m$ should not be too large, otherwise, some relevant variables would start to disappear (right part). Note that best scaling (for the bound) is $m\approx n$, leading to a probability of incorrect selection that goes to zero exponentially fast in $n$.

Of course, in practice, one is not given multiple  independent copies of the same datasets, but a single one. One strategy is to split it in different pieces, as described in \mysec{splitting}; this however relies on having enough data to get a large number of pieces, which is unlikely to happen in practice.  Our main goal is this paper is to show that by using the bootstrap, we can mimic the availability of having multiple copies. This will come at a price, namely an overall convergence rate of $n^{-1/2}$ instead of exponential in $n$

\subsection{Splitting into pieces}
\label{sec:splitting}
We can cut the dataset into $m$ pieces of the same size, a procedure reminiscent of cross-validation. However, it requires extra-assumption on the design, i.e.,  we  need to assume that the smallest eigenvalues of the data matrices of length $n/m$ are still strictly positive  (see proof in Appendix~\ref{app:proofs-lowdim-cutting}):
 
\begin{proposition}
  \label{prop:lowdim-cutting}
  Assume \hypreff{model}{inv} for $m$ disjoint subdatasets of the original dataset, and $\tmu \leqslant  \frac{ \minf{\tw} \tl }{2 p^{1/2}}$. If $c = \tmu n^{1/2} m^{-1/2} p^{1/2}> 0$, there exists $f(c),a(c)>0$ such that:
  $$
  \P( \hat{J}^\cap \neq \J )
  \leqslant   p \left( 
1 - e^{-f(c) p^{-1/2}}\!\!\!  + h(c) \frac{ p^{5/2} m^{1/2}}{ n^{1/2}}  \right)^m \!\!\!
 +2p m   \exp\left( - \frac{\minf{\tw} \tl^2 }{8\ttau^2 } \frac{n}{mp} \right). $$
    \end{proposition}
  
  The proposition above requires that $m/n$ tends to zero, i.e., there should not be too many pieces (which is also required to allow invertibility of the sub-designs).
  Note that several independent partitions could be considered, and would lead to results similar to the ones for the bootstrap presented in the next two sections~\cite{stability}.

 \subsection{Random pair bootstrap}
\label{sec:bootstrap}
 Given the $n$  observations $(x_i,y_i) \in \rb^p \times \rb$, $i=1,\dots,n$, put together into matrices
$X \in \rb^{ n \times p}$ and $y \in \rb^n$, we consider $m$ \emph{bootstrap} replications of the $n$ data points~\cite{efron}; that is, for $k=1,\dots,m$, we consider a \emph{ghost sample}
$(x^{k}_i,y^k_i)\in \rb^p \times \rb$, $i=1,\dots,n$, given by matrices
$X^k \in \rb^{ n \times p}$ and $y^k \in \rb^n$. For each $k \in \{1,\dots,m\}$, the $n$ pairs $(x^{k}_i,y^k_i)$,
$i=1,\dots,n$, are sampled uniformly and independently at random \emph{with replacement} from the $n$ original pairs in $(X,y)$. Some pairs $(x_i,y_i)$ are not selected, some selected once, some selected twice, and so on.
Note   that we could consider bootstrap replications with more or less points than $n$, but for simplicity, we keep it the same as the original number of data points. 

 The following proposition shows that we obtain a consistent model estimate by intersecting the active sets $\hat{J}^1,\dots, \hat{J}^m$ obtained from running the Lasso on each bootstrap sample
 $(X^1,y^1),\dots,(X^m,y^m)$, a procedure we refer to as the \emph{Bolasso} (see proof in Appendix~\ref{app:proofs-bolasso}):
 \begin{proposition}
  \label{prop:lowdim-bolasso-pairs}
  Assume \hypreff{model}{inv}.  If $c = \tmu n^{1/2}  p^{1/2}> 0$, there exists strictly positive constants $A_0,\dots,A_7$ that may depend on $c$ such that if  $n p^{-6}\geqslant A_6$ and $mp^{-1} \geqslant A_7$, we have, for boostrapping \emph{pairs}:
\begin{multline*}
 \P( \hat{J}^\cap \neq \J )
\leqslant
m p \exp \left( - A_0 \frac{n^{1/2}}{p^{1/2}} \right) 
+  A_4  \left(  A_3  \frac{p^{3}}{ n^{ 1/2}} +  \frac{ \log m }{m} \right)^{
1 +  A_5 \left( 2 \log 
 \left(  A_3 \frac{p^{3}}{ n^{ 1/2}} +  \frac{ \log m }{m}  \right)\right)^{-1/2}} \!\!\!\!.
\end{multline*}
    \end{proposition}
Note that in Proposition~\ref{prop:lowdim-bolasso-pairs}, for any $\eta>0$, if $n$ and $m$ are large enough, then we get an upper bound on the probability of incorrect model selection of the form
$  B_1 m   e^{-B_2 n^{1/2}}
+  \left( \frac{B_3}{ n^{1/2}} + B_4 \frac{ \log m }{m} \right)^{1+\eta}
$, where $B_1,\dots,B_4$ are positive constants. Note that in~\cite{bolasso}, we have derived a bound with better behavior in $n$, i.e., with $\eta = 0$. However, the bound in~\cite{bolasso} holds for random designs and has constants which scale \emph{exponentially} in~$p$ and not polynomially. We are currently trying to improve on the bound in Proposition~\ref{prop:lowdim-bolasso-pairs} to remove the extra factor $\eta >0$. 

As before, the number of replications should be as large as possible to remove irrelevant variables, and   $m$ should not be too large, otherwise, some relevant variables would start to disappear from the intersection. Note that best scaling (for the bound) is $m\approx n^{1/2}$, leading to an overall probability of incorrect model selection that tends to zero at rate $n^{-1/2}$, instead of the exponential rate for the irrealistic situation of having multiple copies (\mysec{lowdim-copies}).

We have not explored yet the optimality (in the minimax sense) of the bound given in Proposition~\ref{prop:lowdim-bolasso-pairs}. While we believe that a rate of $n^{-1/2}$ cannot be improved upon, the rate $p^6$ should be improved with further research.

Finally, we have explored in~\cite{bolasso} the possibility of considering softer ways of performing the intersection, i.e., by keeping all variables that appear in a certain proportion of the active sets corresponding to the various replications. This is important in cases where the decay of the loading vectors does not have sharp threshold as assumed in most analyses (this paper included). However, it adds an extra hyper-parameter and the theoretical analysis of such schemes is out of the scope of this paper.

\subsection{Boostrapping residuals}
An alternative to resampling pairs $(x_i,y_i)$ is to resample only the estimated centered residuals~\cite{efron,freedman}. This is well adapted to fixed-design assumptions, in particular because the design matrix $X$ remains the same for all replications. Note however, that the consistency of this resampling scheme usually relies more heavily on the homoscedasticity assumption \hypref{var} that we make in this paper~\cite{freedman}. Moreover, since the Lasso estimate is biased, the behavior differs slightly from bootstrapping pairs, as shown empirically in \mysec{simulations}.

Bootstrapping residuals works as follows; we let denote $\tilde{\varepsilon}_i = y_i - \hat{w}^\top x_i
= \varepsilon_i - ( \hat{w} - \w )^\top x_i$ the vector of estimated residuals, and $\hate_i$ the centered residuals equal to $\hate_i = \tilde{\varepsilon}_i - \frac{1}{n} \sum_{k=1}^n \tilde{\varepsilon}_k$ . When bootstrapping residuals, for each $i\in \{1,\dots,n\}$, we keep $x_i$ unchanged and we use as data $y_i^\ast = 
\hat{w}^\top x_i +  \hate_{i^\ast}$, where $i^\ast$ is a random index in $\{1,\dots,n\}$---the sampling is uniform and the $n$ indices are drawn independently.

   We obtain a similar bound than when bootstrapping pairs  (see proof in Appendix~\ref{app:bolasso-residuals}):
   \begin{proposition}
   \label{prop:bolasso-residual}
   \label{prop:lowdim-bolassoresiduals}
    Assume \hypreff{model}{inv}.  If $c = \tmu n^{1/2}  p^{1/2}> 0$, there exists strictly positive constants $A_0,\dots,A_7$ that may depend on $c$ such that if  $n p^{-6}\geqslant A_6$ and $mp^{-1} \geqslant A_7$, we have, for boostrapping \emph{residuals}:
\begin{multline*}
 \P( \hat{J}^\cap \neq \J )
\leqslant
m p \exp \left( - A_0 \frac{n^{1/2}}{p^{1/2}} \right) 
+  A_4  \left(  A_3  \frac{p^{3}}{ n^{ 1/2}} +  \frac{ \log m }{m} \right)^{
1 +  A_5 \left( 2 \log 
 \left(  A_3 \frac{p^{3}}{ n^{ 1/2}} +  \frac{ \log m }{m}  \right)\right)^{-1/2}} \!\!\!\!.
\end{multline*}    \end{proposition}

The bound in Proposition~\ref{prop:lowdim-bolassoresiduals} is the same as bootstrapping pairs, but as shown in Appendix~\ref{app:bolasso-residuals}, the constants are slightly better). However, as shown in \mysec{experiments-high}, the behaviors of the two methods differ notably: random-pair bootstrap does not lead to good selection performance in high-dimensional settings, while residual bootstrap does.
While we are currently unable to proof the consistency of bootstrapping residuals in high-dimensional settings, we prove in \mysec{highdim} the model consistency of a related two-step procedure, where the bootstrap replications are performed within the support of the Lasso estimate on the full data.

\section{High-Dimensional Analysis}

\label{sec:highdim}
In   high-dimensional settings, i.e.,  when $p$ may be larger than $n$, we need to change   assumption $\hypref{inv}$ regarding the invertibility of the empirical second order moment, which cannot hold. Various assumptions have been used for the Lasso, based on low correlations~\cite{lounici}, sparse eigenvalues~\cite{yuinfinite} or more general conditions~\cite{tsyb,cohen}. In this paper, we introduce a novel assumption, which not only allows us to consider that the support of the Lasso estimate has a bounded size, but also implies that we obtain the same sign pattern with high probability. The analysis carried out in low-dimensional settings in \mysec{high} is thus also valid in high-dimensional settings.

\subsection{High-dimensional assumptions}
Our analysis relies on the analysis carried out in \mysec{high} for ``high'' regularization, i.e., when $\mu$ tends to zero slower than $n^{-1/2}$. 
 In this setting, we have shown that the Lasso estimate asymptotically behaves as $\w + \mu \Delta$, where $\Delta$ is the unique minimum of
\BEQ
\label{eq:local}
\min_{\Delta \in \rb^p}  \frac{1}{2} \Delta^\top Q \Delta + \Delta_\J^\top \sign(\w_\J) + \| \Delta_{\J^c}\|_1.
\EEQ
We let denote $\K \subset \J^c$ the ``extended'' support of a solution $\Delta_{\J^c}$ of \eq{local-eq} and $\L = \J \cap \K$: that is, we not only keep all indices corresponding to  non zero elements of $\Delta_{\J^c}$, but also the ones for which the optimality condition in \eq{A00} is an equality (i.e., if we are at a hinge point of the regularization path, we take all involved variables)

 We consider the vector $\t \in \{-1,0,+1\}^p$ defined by $t_\J = \sign(\w_\J)$ and $\t_{\J^c} = \sign(\Delta_{\J^c}) $. If we assume that $\lmin(Q_{\L, \L}) >0$, then the solution to \eq{local} is unique~\cite{fuchs}, and is such that $\Delta_\L = 
  -Q_{\L,\L }^{-1} \t_\L$ and the optimality conditions for \eq{local} are simply
$$
\sign(  - [ Q_{\L,\L }^{-1} \t_\L]_\K ) =  \t_\K
\mbox{ and } 
\| Q_{\L^c \L}   Q_{\L,\L }^{-1} \t_\L  \|_\infty \leqslant 1.
$$

We make the following assumptions (note that \hypref{inv-highdim} is essentially equivalent to the lack of hinge point which is also made in Proposition~\ref{prop:lowdim-high}):
\begin{hyp}
\label{hyp:unicity-highdim}
\emph{Unicity of local noiseless problem}: the matrix $\Q_{\L,\L}$ is invertible.
\end{hyp}
\begin{hyp}
\label{hyp:inv-highdim}
\emph{Stability of local noiseless problem}: $\| Q_{\L^c \L}   Q_{\L,\L }^{-1} \t_\L  \|_\infty < 1$.
\end{hyp}
We let denote 
\BEQ \boldsymbol{\theta} = \min \left\{
 1 - \|Q_{\L^c \L}   Q_{\L,\L }^{-1} \t_\L\|_\infty ,
\min_{ k \in \K }  | ( Q_{\L, \L}^{-1} \t_\L)_k Q_{k,k}  |
\right\},
\EEQ
the quantity that will characterize the \emph{stability} of the local noiseless problem; if \hypreff{unicity-highdim}{inv-highdim} are satisfied, then $\boldsymbol{\theta} >0$.
 As shown in Proposition~\ref{prop:highdim}, the quantity $\boldsymbol{\theta}$ dictates the speed of convergence of the probability of not getting $\t$ as a sign pattern for the Lasso problem in \eq{lasso} or \eq{lasso-eq}.

\paragraph{Comparison with consistency condition}
We now relate \hypref{inv-highdim} with the consistency condition for the Lasso in \eq{cond}: if \eq{cond} satisfied,  then $\K = \varnothing$ and the condition \hypref{inv-highdim} simply becomes:
$$ \| Q_{\J^c ,\J} Q_{\J ,\J}^{-1} \sign(\w_\J) \|_\infty < 1, $$
 which is exactly a strict version of \eq{cond}---an assumption commonly made for high-dimensional analysis of the Lasso~\cite{Zhaoyu,martin}. Note that we then have the simplified expression
 $\boldsymbol{\theta} = 1 -   \|Q_{\J^c \J}   Q_{\J,\J }^{-1} \sign(\w_\J)\|_\infty$.

The main goal of this paper is to design a consistent procedure even when \eq{cond} is not satisfied. As we have seen, \hypref{inv-highdim} is weaker than the usual assumptions made for the Lasso consistency; in \myfig{condition} (left and middle), we compare empirically the two conditions for random i.i.d. Gaussian designs, showing that our set of assumptions is weaker, but of course breaks down when $n$ is too small (too few observations) or the cardinal of $\J$ is too large (too many relevant variables). We are currently exploring theoretical proofs of this behavior, extending the current analysis of \cite{martin} for \eq{cond}; in particular, we aim at determining the various scalings between $p$, $n$ and the number of relevant variables for which a Gaussian ensemble leads to consistent variable selection with high probability (according to our assumptions which are weaker than in \cite{martin}). Moreover, in the right plot of 
\myfig{condition}, we show values of $\log \boldsymbol{\theta} $ for various $n$ and $|\J|$, which characterize the convergence rate of our bound. Relying on $\boldsymbol{\theta} $ which is bounded from below is clearly a weakness of our approach to high-dimensional estimation; we are currently exploring refined conditions where we relax the stability, i.e., we allow several (but not too many) patterns to be selected with overwhelming probability.

\begin{figure}
\hspace*{.25cm}
\includegraphics[scale=.385]{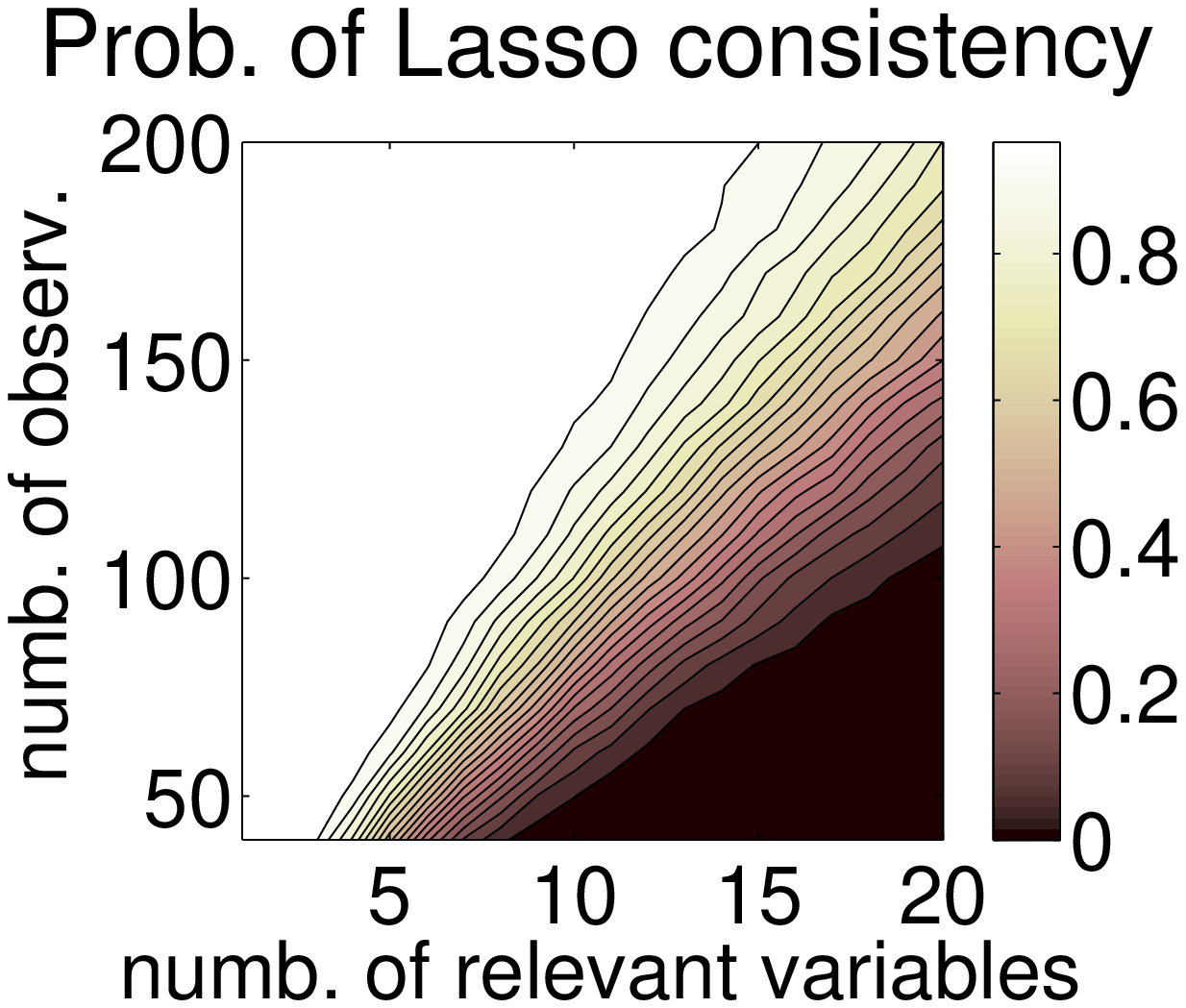} \hspace*{.3cm}
\includegraphics[scale=.385]{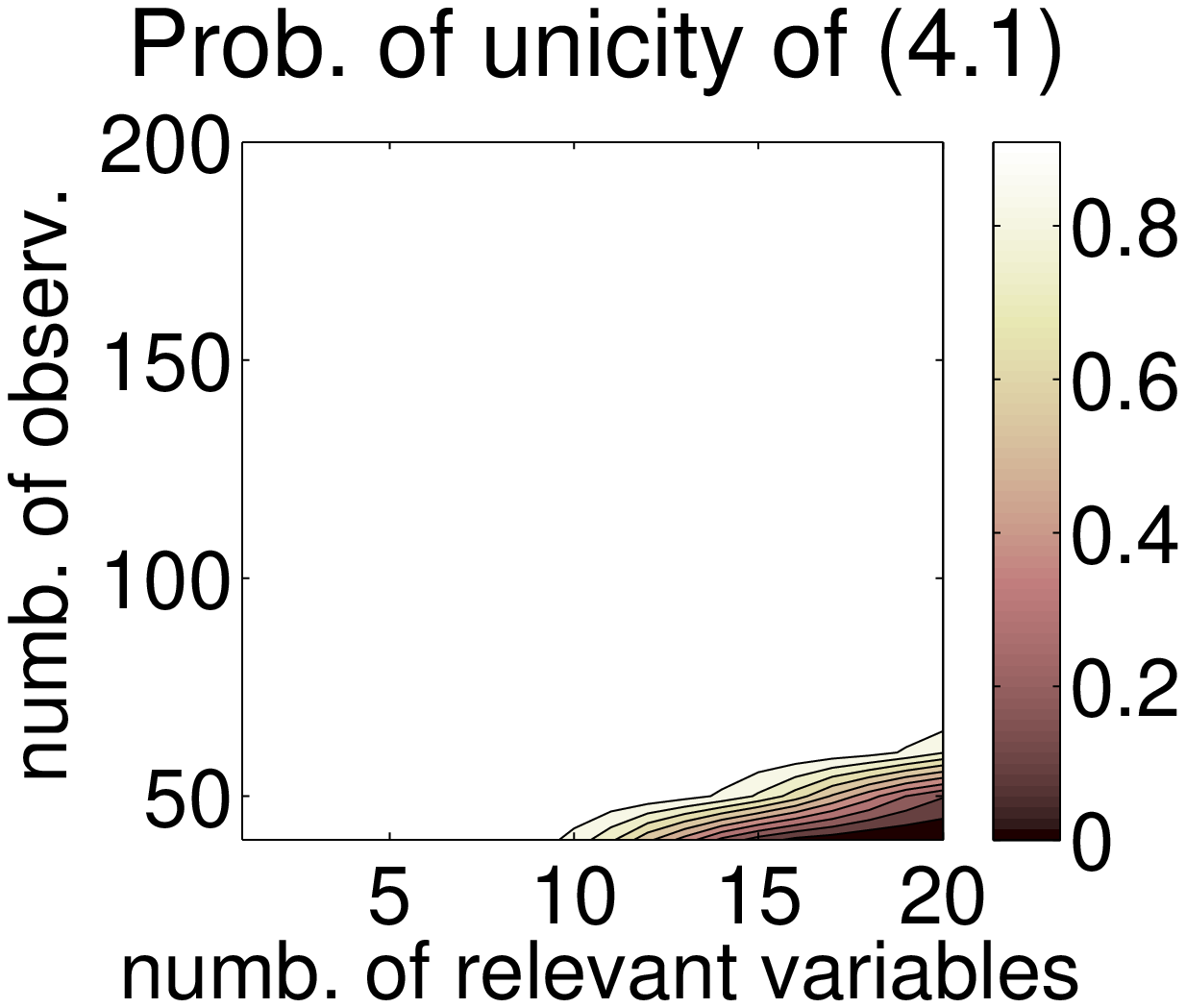} \hspace*{.3cm}
\includegraphics[scale=.385]{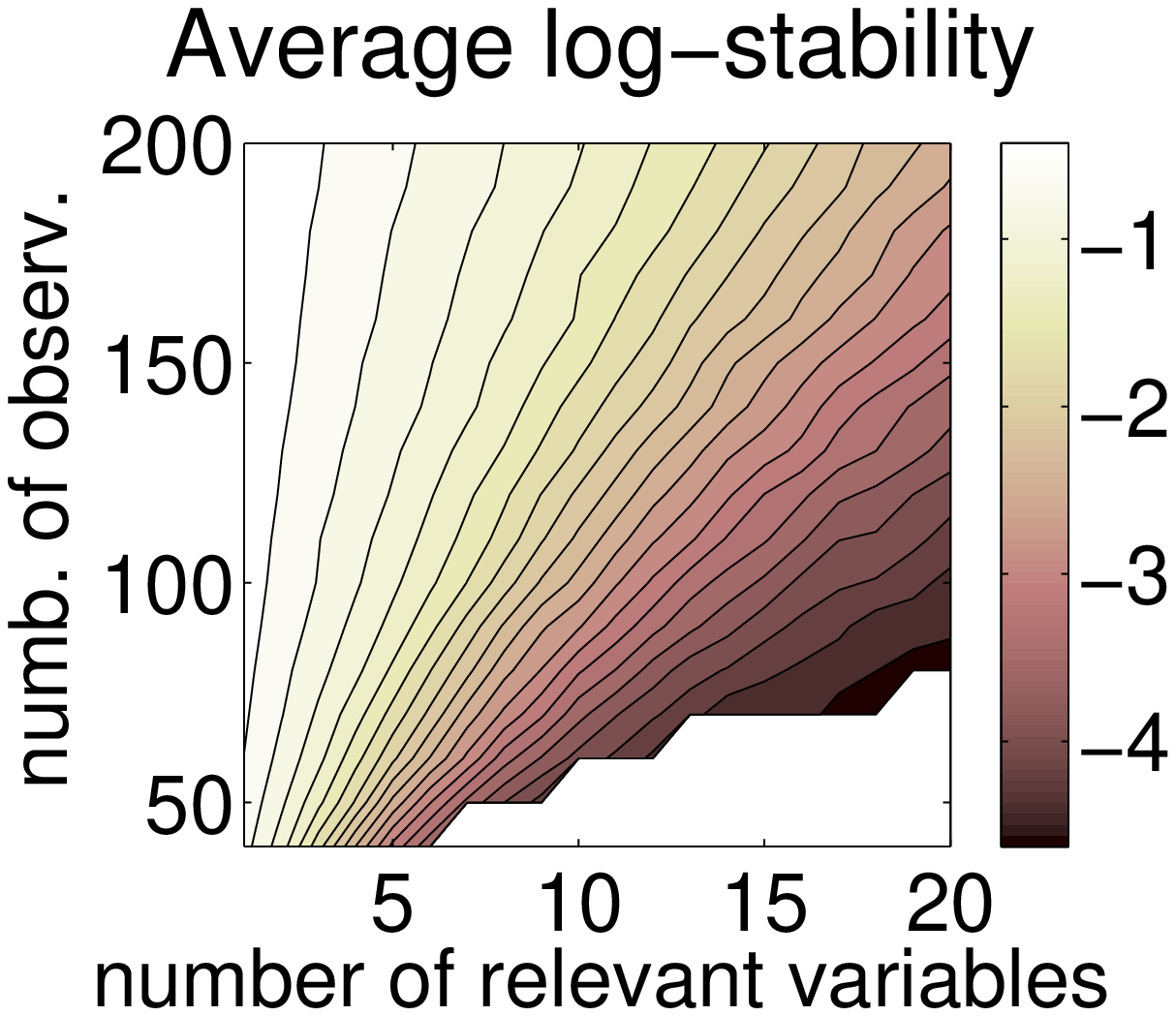} 
\hspace*{-.25cm}
 \caption{Consistency conditions for random Gaussian designs, $p=128$, $n$ from $40$ to $200$ and $| \J |$ from 1 to 20 (all probabilities and averages obtained from 1000 replications). Left: probability that \eq{cond} is satisfied. Middle: probability that \hypref{inv-highdim} is satisfied. Right: expectation of $\log \boldsymbol{\theta}$ (plotted only for the ones for which the local problem is unique with  high probability).
  }
 \label{fig:condition}
\end{figure}

\paragraph{Checking assumptions~\hypreff{unicity-highdim}{inv-highdim}} In \eq{local}, we can optimize in closed form with respect to $\Delta_\J$ as
$\Delta_\J = Q_{\J,\J}^{-1}( - \sign( \w_\J) - Q_{\J , \J^c} \Delta_{\J^c})$, leading to an optimization problem for $\Delta_{\J^c}$:
\BEQ
\label{eq:local-eq}
\min_{\Delta \in \rb^p}  \frac{1}{2} \Delta_{\J^c}^\top Q_{\J^c,\J^c | \J} \Delta_{\J^c} -
\Delta_{\J^c}^\top
Q_{\J^c ,\J } 
Q_{\J,\J}^{-1} \sign( \w_\J) 
+ \| \Delta_{\J^c}\|_1,
\EEQ
which can be solved using existing code for the Lasso. We are currently working on deriving sufficient conditions which do not depend on the sign pattern of the population loading $\w$ (but only on the sparsity pattern, or even its cardinality), as usually done for the consistency condition in \eq{cond}~\cite{Zhaoyu,yuanlin}.

\subsection{Stability of sign selection}
With assumptions \hypref{unicity-highdim} and \hypref{inv-highdim}, we can show that with high-probability, when the regularization parameter is asymptotically greater than $n^{-1/2}$, then the sign of the  Lasso estimate is exactly $\t$ (see proof in Appendix~\ref{app:proofs-highdim}):

\begin{proposition}
\label{prop:highdim}
Assume \hypreff{model}{bounded}, \hypreff{unicity-highdim}{inv-highdim}, and
$
\tmu \leqslant \frac{\tl_\L \minf{\tw}}{ 2 | \L |^{1/2}}
$. Then: 
\BEQ
\label{eq:highdim-lasso}
\P( \sign(\hat{w}) \neq \t)  \! \leqslant  \!
2 p \exp \left( - \frac{ n \tmu^2  \boldsymbol{ \theta}^2 \tl_\L }
{ 8 \ttau^2 |\L|} \right) +  2 |\J| \exp\left(\! - \frac{ n  \minf{\tw}^2 \tl^2_\L }
{ 4 \ttau^2 |\L|} \!
\right).
\EEQ
\end{proposition}
Note that if $\boldsymbol{ \theta}$ is bounded away from zero, then we simply need that $\log p = o(n)$ for our result to hold. Moreover, in \eq{highdim-lasso}, we can see that $\boldsymbol{ \theta}$ dictates the asymptotic behavior of our bound. If it is too small, then in order to have a meaningful bound for this design matrix, we would need to consider sign patterns which are close to $\t$ and show that the sign pattern of the Lasso estimate $\hat{w}$ is with high probability within these sign patterns.  

\subsection{High-dimensional Bolasso}

Proposition~\ref{prop:highdim} suggests to run the Lasso once with a larger regularization parameter (i.e., multiplied by $\log p $) and run the various resampling schemes within the active set of the original Lasso estimation (which is very likely to be the support associated with $\t$).  More precisely, we have the proposition   (see proof in Appendix~\ref{app:proofs-highdim}):

\begin{proposition}
\label{prop:bolasso-highdim}
Assume \hypreff{model}{bounded} and \hypreff{unicity-highdim}{inv-highdim}. 
  If $c = \tmu n^{1/2}  |\L|^{1/2}> 0$, there exists strictly positive constants $A_0,\dots,A_7$ that may depend on $c$ such that if  $n |\L|^{-6}\geqslant A_6$ and $m|\L|^{-1} \geqslant A_7$, we have, for boostrapping \emph{residuals}:
\begin{multline*}
 \P( \hat{J}^\cap \neq \J )
\leqslant
  2 p \exp \left( - \frac{c^2 (\log p)^2  \boldsymbol{ \theta}^2 \tl_\L }
{ 8 \ttau^2 |\L|^2} \right) +  2 |\J| \exp\left(\! - \frac{ n  \minf{\tw}^2 \tl^2_\L }
{ 4 \ttau^2 |\L|}  \right)  + \\
m p \exp \left( - A_0 \frac{n^{1/2}}{|\L|^{1/2}} \right) 
+  A_4  \left(  A_3  \frac{|\L|^{3}}{ n^{ 1/2}} +  \frac{ \log m }{m} \right)^{
1 +  A_5 \left( 2 \log 
 \left(  A_3 \frac{|\L|^{3}}{ n^{ 1/2}} +  \frac{ \log m }{m}  \right)\right)^{-1/2}} \!\!\!\!.
\end{multline*}   
\end{proposition}
Note that the constants depend polynomially on $|\L|$ and $\lmin(Q_{\L,\L})$, and do not depend on $p$. This is thus a high-dimensional result where $p$ may grow large compared to $n$. If we relax \hypref{inv-highdim}, then the original Lasso estimate would have a small set of allowed patterns with high probability (instead of simply one), and a union bound considering all those would need be considered.

\section{Algorithms and Simulations}
\label{sec:simulations}

In this section, we describe efficient algorithms for the boostrapped versions of the Lasso that we present in this paper and we illustrate the various consistency results obtained in previous sections, in low-dimensional and high-dimensional settings. 

\subsection{Efficient Path Algorithms}
\label{sec:algorithms}
We first consider efficient algorithms for the boostrapping procedures, based on homotopy methods~\cite{osborne,lars,garrigues}. Similar developments could be made for first-order methods~\cite{shooting,descent}. For the regular Lasso, one can find the solutions of \eq{lasso} for all values of the regularization parameter $\mu$ that correspond to less than $k$ selected covariates in time which is empirically
$
O(pn + k^2 n )
$: indeed,  computing $\frac{1}{n}X^\top y $ once is $O(pn)$, while computing the relevant elements of $Q
=\frac{1}{n} X^\top X$ and updating various quantities is $O(k^2n)$. Note that our analysis suggests to stop the path when the solution of the problem is not unique anymore, i.e., when the design matrix of selected variables become rank-deficient.

\paragraph{Bootstrapping pairs}

When bootstrapping pairs, we require $m$ applications of the regular Lasso procedure with different design matrices, so we get a complexity of $O(mpn + mk^2n)$, and since the designs are different, there is no immediate possibility of sharing computations between different bootstrap replications

\paragraph{Bootstrapping residuals}

When bootstrapping residuals, we first run the Lasso once, with complexity $
O(pn + k^2 n )
$. Then, for all values of the regularization parameter, naively, we would have to run the Lasso $m$ times. In order to avoid running the Lasso as many times as $m$ times the number of values of $\mu$ we want to consider, one can first notice that there are at most $O(k)$ break points in the original Lasso estimation, and that between break points, one has to minimize an objective function which is composed of a $\ell^1$-penalty, a quadratic term and a linear term whose coefficients depend affinely in $\mu$. This implies that the path is also piecewise linear within this segment and can be followed using an homotopy algorithm very similar to the one for the regular Lasso. Thus it makes $O(mpn + mk^2n)$ per segments when restarting an homotopy method for this segment, i.e., an overall complexity of
 $O(mkpn + mk^3n)$. This can be put down by computing a joint path that goes through all $O(k)$ segments sequentially instead of in parallel, in total time $O(mkpn + mk^2 n)$. Moreover, since  when bootstrapping residuals, the design matrix is the same for all replications and computations of submatrices of $Q$ may be cached, to obtain a complexity of
 $O(mkpn + k^2 n)$.
 
 Similarly, when bootstrapping after projections onto the active set of a single global Lasso run, one can get even get a lower complexity of $O( pn +  mk^2n)$, i.e., one Lasso followed by $m$ Lasso on a reduced data set. This requires however updates (when the first Lasso estimation switches active sets) such as the ones proposed in~\cite{garrigues}.

\begin{figure}
\begin{center}

\vspace*{-.05cm}

\includegraphics[scale=.55]{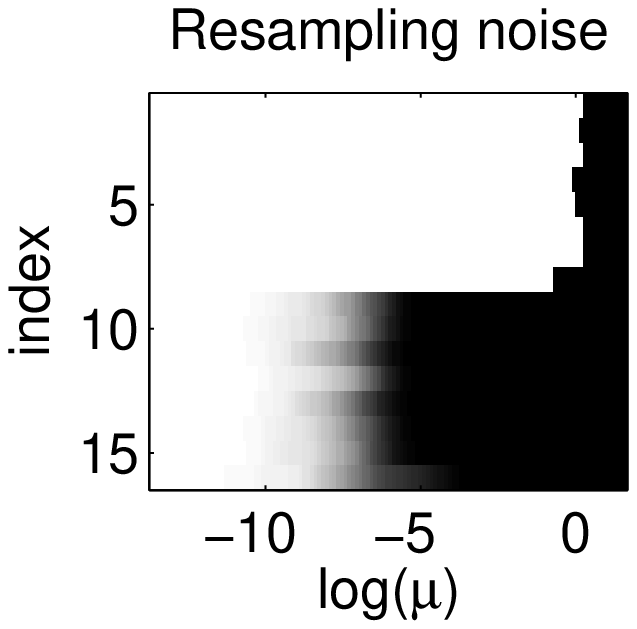}
 \hspace*{.5cm}
 \includegraphics[scale=.55]{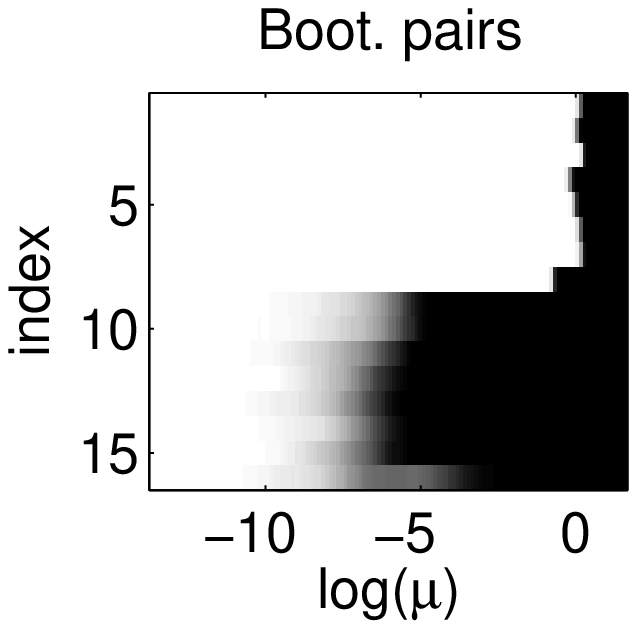}
 \hspace*{.5cm}
 \includegraphics[scale=.55]{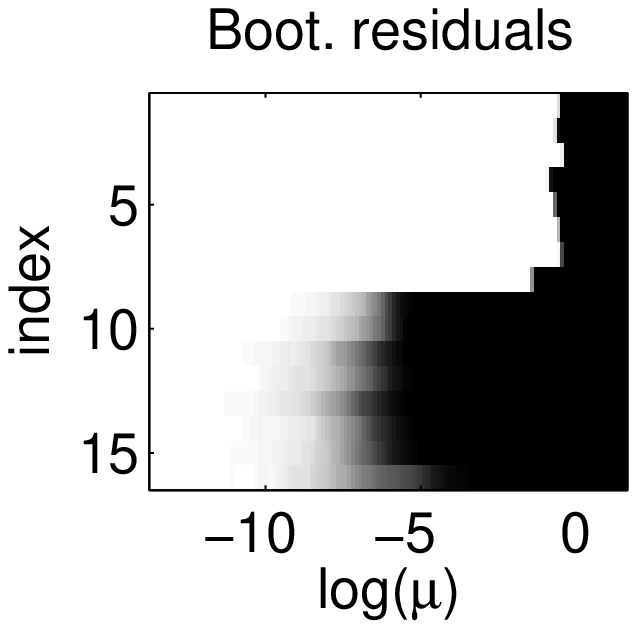}

\vspace*{.15cm}

\includegraphics[scale=.55]{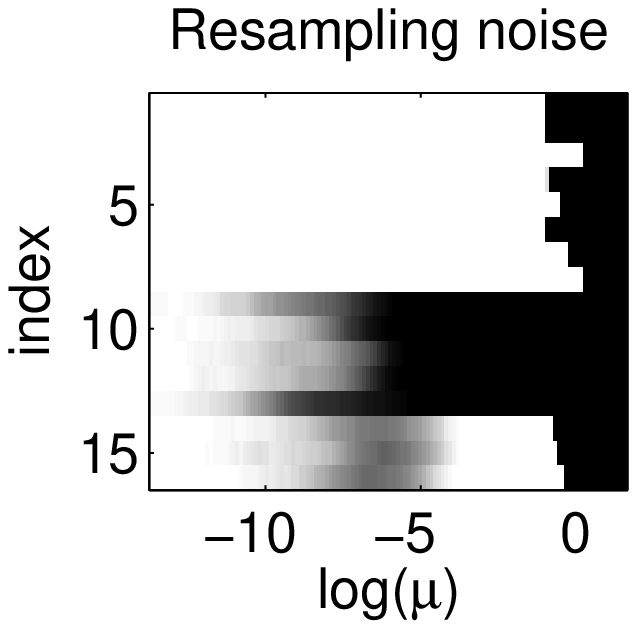}  \hspace*{.5cm}
\includegraphics[scale=.55]{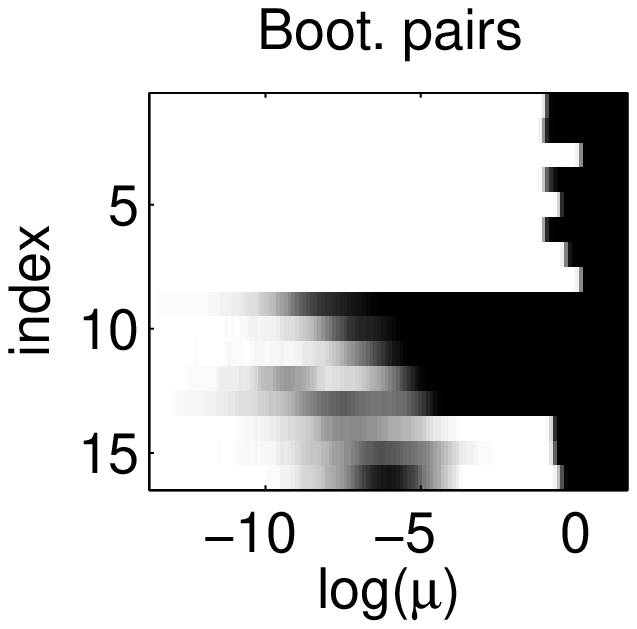}  \hspace*{.5cm}
\includegraphics[scale=.55]{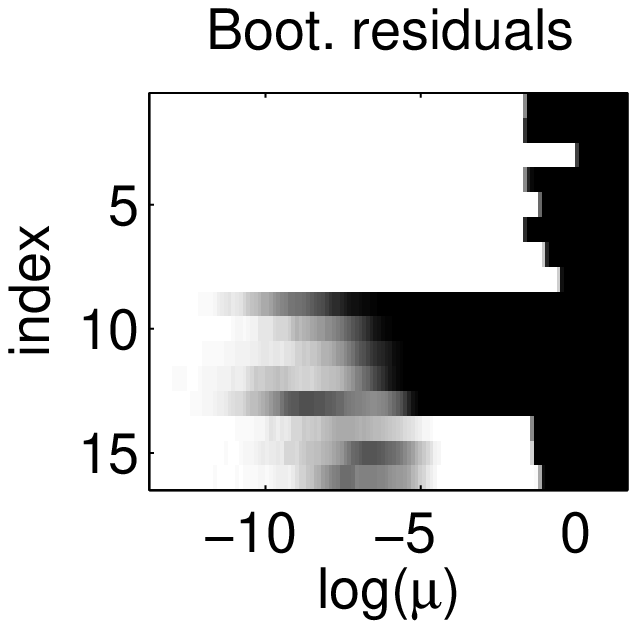}

\vspace*{-.45cm}

\end{center}
\caption{Probability of selecting each variable vs. regularization parameter $\mu$ (low-dimensional setting) for various resampling schemes, \emph{before intersecting}. White values correspond to probability equal to one, and black values correspond to probability equal to zero (model consistency corresponds to while on the top 8 variables and black on the rest). Top: consistency condition of the Lasso is satisfied, Bottom: consistency condition not satistfied. Note the similar behavior of resampling noise (which requires knowing the generating distribution) and the two forms of bootstrapping (which do not). See text for details.}
\label{fig:lowdim}

\end{figure}

\subsection{Experiments - Low-Dimensional Settings}

\label{sec:experiments-low}

We first consider a low-dimensional design matrix, with $p=16$, $n=1024$ and $8$ relevant variables (i.e., $\J
=\{1,\dots,8\}$). The design is sampled from a normal distribution with independent rows, sampled i.i.d. from a fixed covariance matrix. We consider two covariance matrices, one that leads to design matrices which do not satisfy the consistency condition of the Lasso in \eq{cond}, and one that leads to Lasso-consistent design matrices.

 In \myfig{lowdim}, we plot the marginal probabilities (computed from 512 independent replications) of selecting any given of the $p=16$ variables for all values of the regularization parameter $\mu$ and for the various resampling schemes (resampling noise, bootstrapping pairs or bootstrapping  residuals), \emph{without intersecting} (i.e., we are just reporting counts from 512 replications from a single dataset). Note that the left column (resampling noise) exactly corresponds to the various regimes of the Lasso presented in \mysec{lowdim} (these require full knowledge of the generating distributions and are only displayed for illustration purposes): for large values of $\mu$, no variable is selected (Proposition~\ref{prop:lowdim-heavy}), then a fixed pattern is selected ($\mu$ tending to zero faster than $n^{-1/2}$, Proposition~\ref{prop:lowdim-fixed}), then all patterns including the relevant variables ($\mu$ of order $n^{-1/2}$, Propositions~\ref{prop:lowdim-medium} and ~\ref{prop:lowdim-medium2}), and finally, for small values of $\mu$, all variables are selected (Proposition~\ref{prop:lowdim-low}). Note that in the top plots, as expected (since \eq{cond} is not satisfied), some portions of the regularization paths lead to the correct pattern, while in the bottom plots, as expected (since \eq{cond} is satisfied), there is no consistent model selection. It is important to note that using the bootstrap leads to similar behavior than resampling the noise, but does not require extra knowledge (i.e., a single dataset is needed). Note finally, that bootstrapping residuals does alter slightly the regularization paths---because of the bias of the Lasso estimate---and the selected patterns (see other evidence of this behavior in \myfig{lowdim-intersection} and \myfig{lowdim-effectofm}).

\begin{figure}
\begin{center}

\vspace*{-.25cm}

\includegraphics[scale=.55]{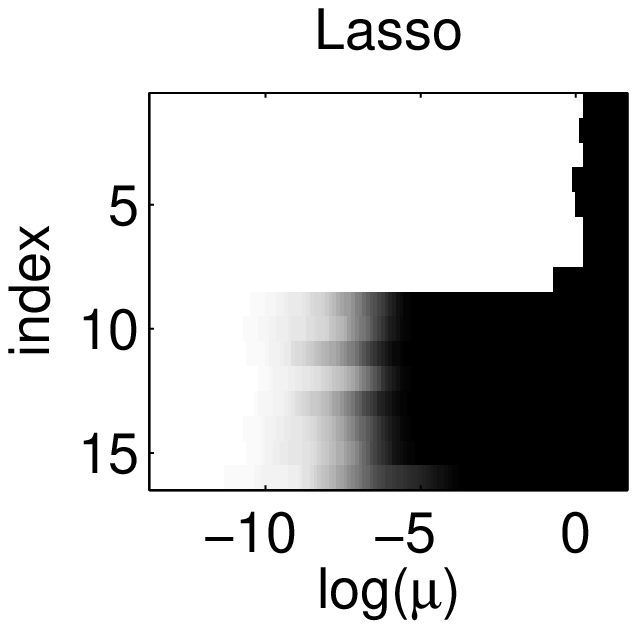}
 \hspace*{.5cm}
 \includegraphics[scale=.55]{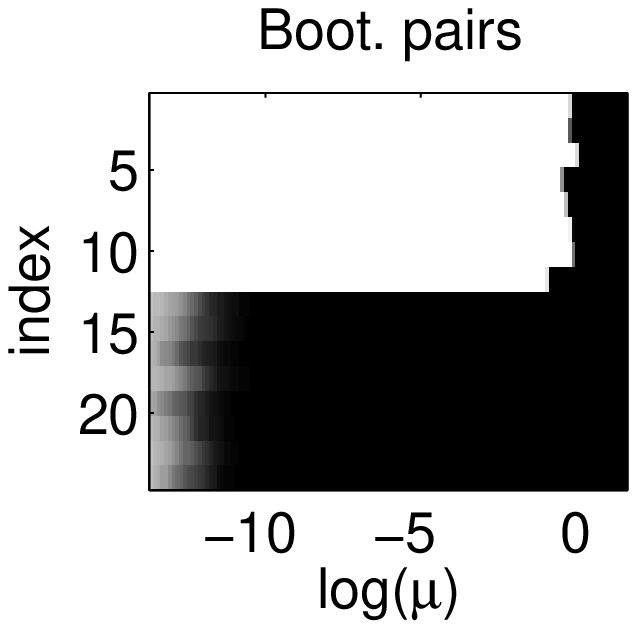}
 \hspace*{.5cm}
 \includegraphics[scale=.55]{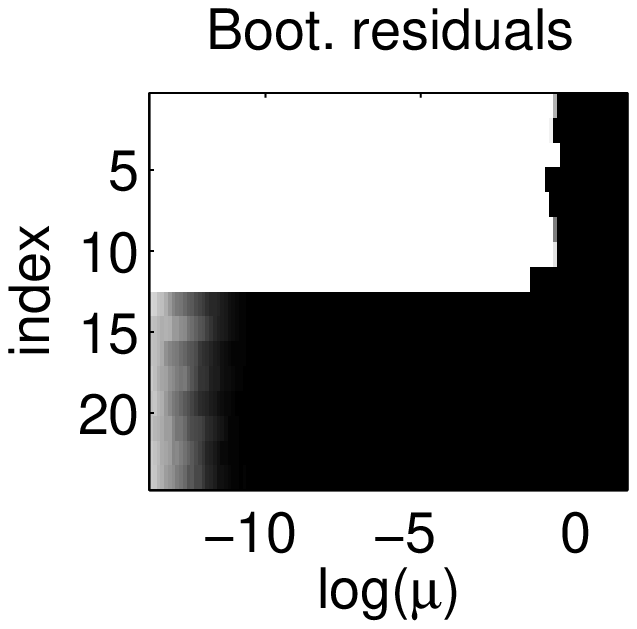}

\vspace*{.05cm}

\includegraphics[scale=.55]{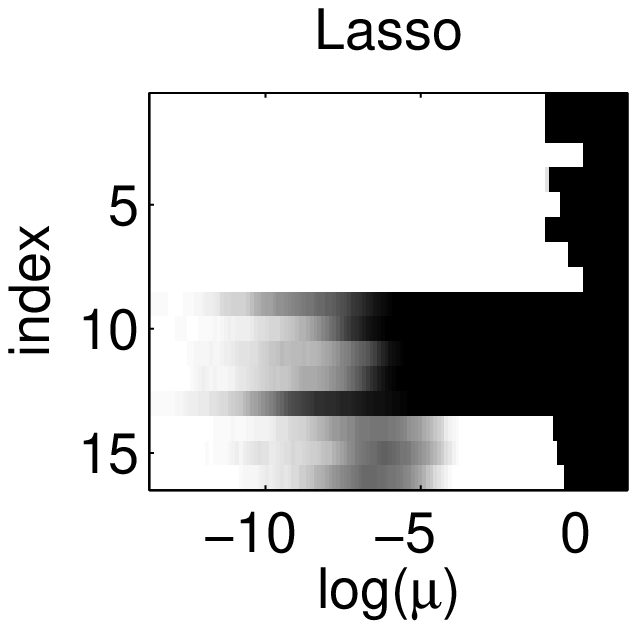}  \hspace*{.5cm}
\includegraphics[scale=.55]{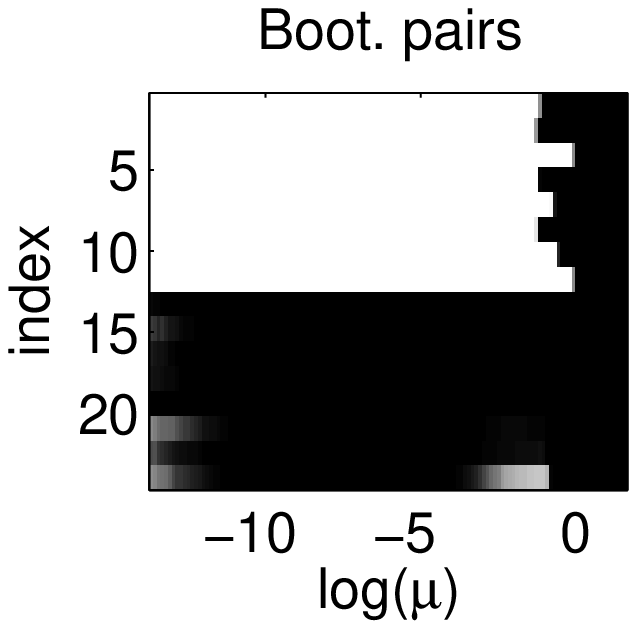}  \hspace*{.5cm}
\includegraphics[scale=.55]{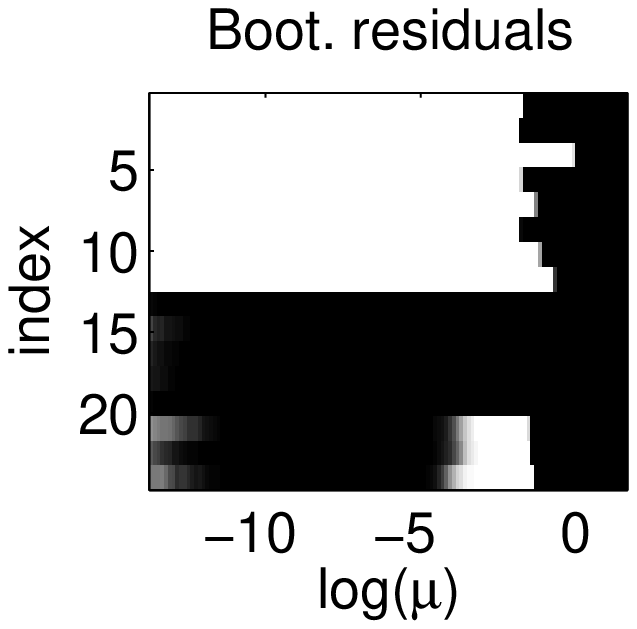}

\vspace*{-.55cm}

\end{center}
\caption{Probability of selecting each variable vs. regularization parameter  $\mu$ (low-dimensional setting) for the Lasso (left column) and the Bolasso (middle and right columns). White values correspond to probability equal to one, and black values correspond to probability equal to zero  (model consistency corresponds to while on the top 8 variables and black on the rest). Top: consistency condition of the Lasso is satisfied, Bottom: consistency condition not satistfied. See text for details.}
\label{fig:lowdim-intersection}

\end{figure}

\begin{figure}
\begin{center}

\vspace*{-.125cm}

\includegraphics[scale=.55]{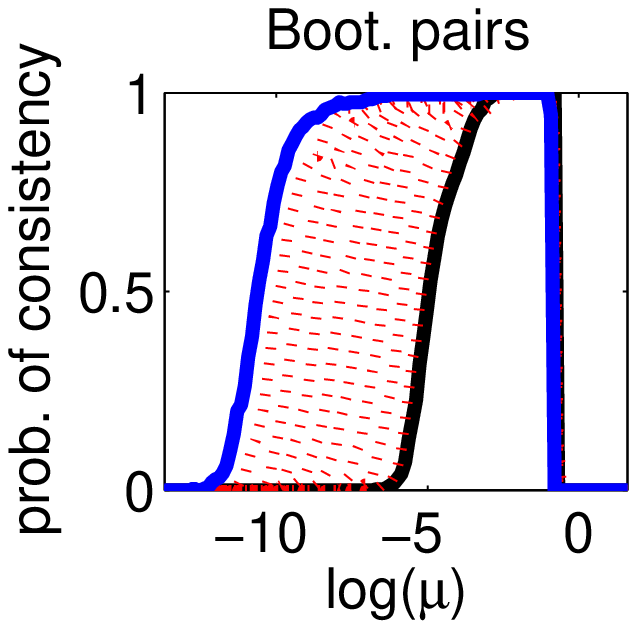} \hspace*{1cm}
\includegraphics[scale=.55]{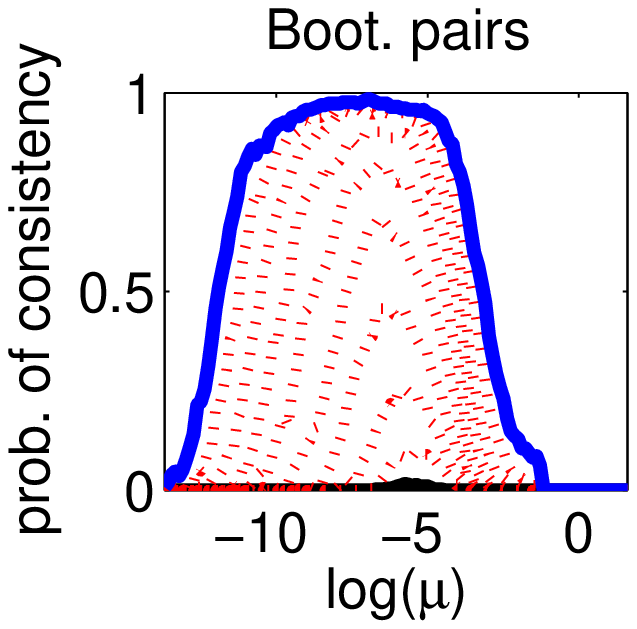}

\vspace*{.05cm}

\includegraphics[scale=.55]{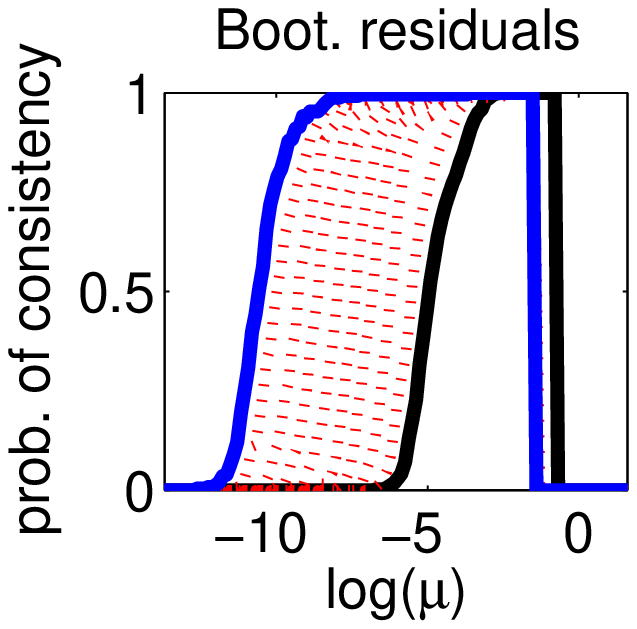} \hspace*{1cm}
\includegraphics[scale=.55]{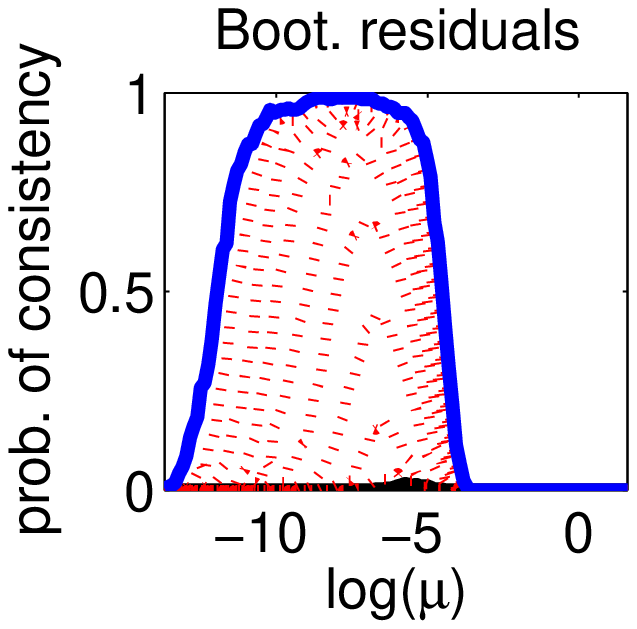}

\end{center}

\vspace*{-.65cm}

\caption{Probability of correct pattern selection with various   numbers $m$ of replications in 
$\{1 \mbox{ in plain black},2,4,8,16,32,64,128,256, \mbox{ all in dashed red}, 512 \mbox{ in plain blue}\}$  (low-dimensional setting). Top: consistency condition of the Lasso is satisfied, Bottom: consistency condition not satistfied. Note that only one replication (plain black) is very similar to the regular Lasso.}
\label{fig:lowdim-effectofm}

\end{figure}

In \myfig{lowdim-intersection}, we compute the marginal probability of selecting the variables for the Lasso (left column) and the various ways of using the Bolasso (boostrapping pairs or residuals), i.e., \emph{after intersecting}. Those are obtained by running the Bolasso with 512 replications, 128 times on the same design but with different noisy observations (thus, a total of $512 \times 128$ Lasso runs are used for each of the plots on the middle and right columns of \myfig{lowdim-intersection}). On the top plots, the Lasso consistency condition in \eq{cond} is satisfied and the two versions of the Bolasso increase the width of the consistency region of the Lasso, while on the bottom plots, it is not, and the Bolasso creates a consistency region. Note that bootstrapping residuals modifies the early parts of the regularization path (i.e., large values of $\mu$), illustrating the effect of the bias of the Lasso when bootstrapping residuals.

In \myfig{lowdim-effectofm}, we consider the effect of the number $m$ of bootstrap replications, in the same two situations. Increasing $m$ seems always beneficial. Note that (1) when $m=1$ (essentially the Lasso), we get some strictly positive probabilities of good pattern selection even in the inconsistent case, illustrating Proposition~\ref{prop:lowdim-medium}, and (2) if $m$ was too large, some of the relevant variables would start to leave the intersection of active sets (but this has not happened in our simulations with only 512 replications).

\begin{figure}
\begin{center}

\vspace*{-.05cm}

\includegraphics[scale=.55]{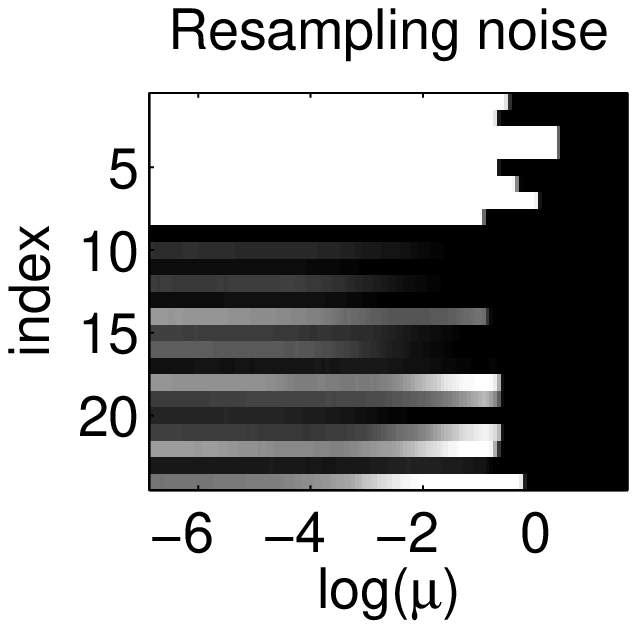}
 \hspace*{.5cm}
 \includegraphics[scale=.55]{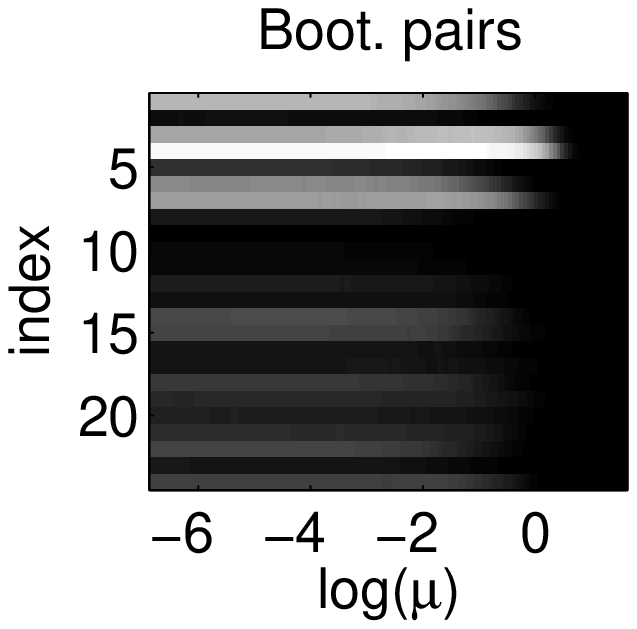}
 \hspace*{.5cm}
 \includegraphics[scale=.55]{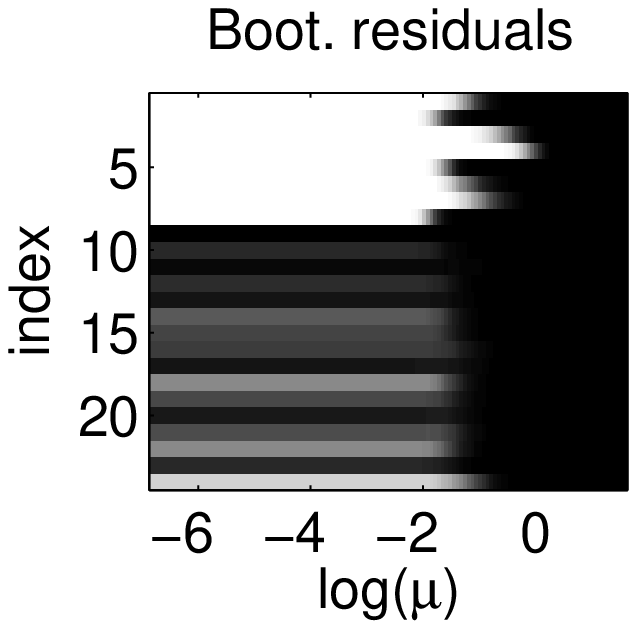}

\vspace*{.15cm}

\includegraphics[scale=.55]{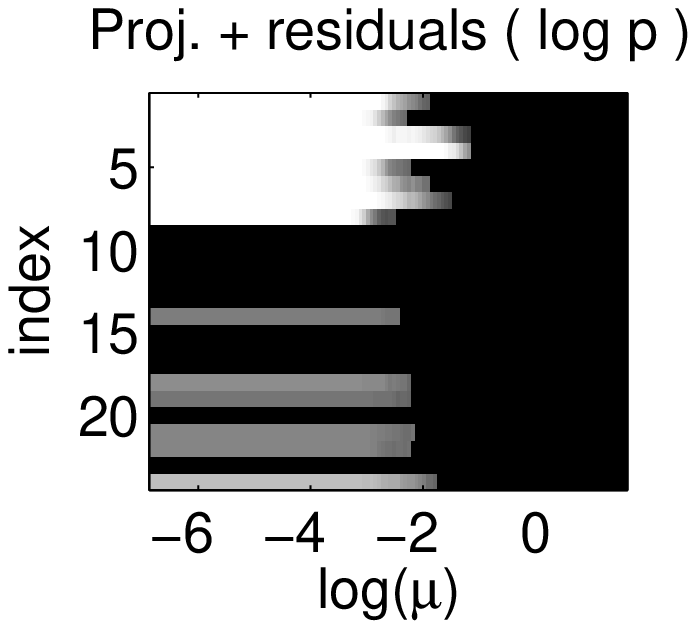}  \hspace*{.5cm}
\includegraphics[scale=.55]{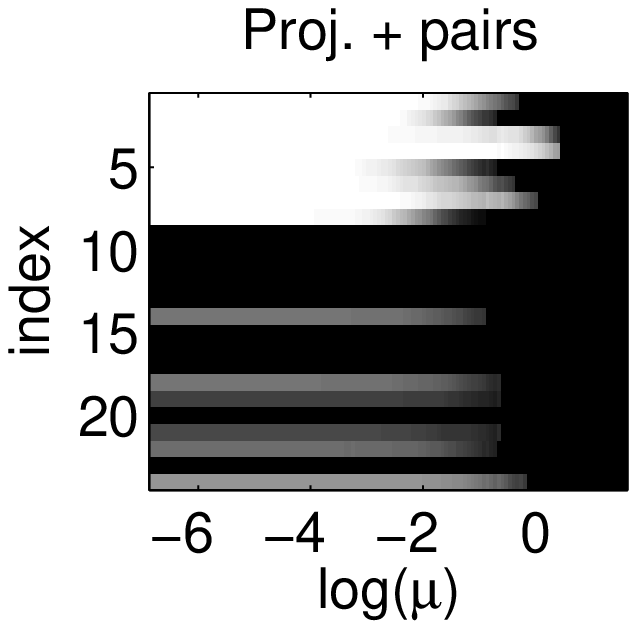}  \hspace*{.5cm}
\includegraphics[scale=.55]{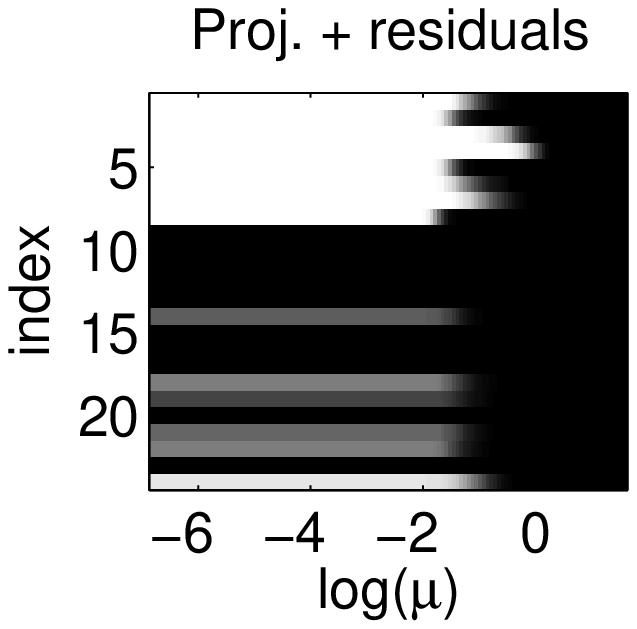}

\vspace*{-.45cm}

\end{center}

 \caption{
 Probability of selecting each variable vs. regularization parameter  $\mu$ (high-dimensional setting) for various resampling schemes, \emph{before intersecting}. Only the first 8 variables and the 16 variables which violates condition in \eq{cond} the most are plotted.
 White values correspond to probability equal to one, and black values correspond to probability equal to zero (model consistency corresponds to while on the top 8 variables and black on the rest). Note the similar behavior of resampling noise (which requires knowing the generating distribution) and all forms of bootstrapping (except for bootsrapping pairs, in the top-middle plot).}
\label{fig:highdim}

\end{figure}

\subsection{Experiments - High-Dimensional Settings}

\label{sec:experiments-high}

We now consider a ``high-dimensional'' design matrix (i.e. such that $p>n$), with $p=128$, $n=64$ and $8$ relevant variables (i.e., $\J = \{1,\dots,8\}$). The design matrix is sampled from a normal distribution  with i.i.d.~elements.
For the sampled design matrix, the condition in \eq{cond} is not satisfied, as for most designs with such $p$, $n$ and $|\J|$, as shown in \myfig{condition} in \mysec{highdim}, but assumptions \hypreff{unicity-highdim}{inv-highdim} are.

We performed the same simulations than in \mysec{experiments-low}, with additional bootstrapping procedures, namely after projecting into the original Lasso estimate, with the same regularization parameter (no consistency result) or with a parameter multiplied by $\log p$ (consistency result in Proposition~\ref{prop:bolasso-highdim}).

In \myfig{highdim}, we consider marginal probabilities \emph{before intersection}, to study the general behavior of various resampling schemes. We see that bootstrapping procedures behave rather differently than resampling the noise (unlike in low-dimensional settings), and that boostrapping pairs does lose some of the relevant variables while boostrapping residuals does not. After projection, all resampling procedures behave correctly. In \myfig{highdim-intersection}, we compare the Lasso and the Bolasso (for several ways of performing the bootstrap): boostrapping residuals consistently leads to better performance. Note that while the top right plot behaves correctly, we currently have no proofs for it. In \myfig{highdim-effectofm}, we consider the effect of various numbers of replications. Note that in the bottom-right plot, 512 replications are indeed too many (i.e., when too many replications are used, we start to  lose some of the relevant variables).

\section{Conclusion} 

 We have presented a detailed analysis of the variable selection properties of a boostrapped version of the Lasso. The model estimation procedure, referred
  to as the Bolasso, is provably consistent under general assumptions, in low-dimensional and high-dimensional settings. We have considered the two types of bootstrap for linear regression, and have shown empirically and theoretically better properties for the bootstrap of residuals.
  This work brings to light that poor variable selection results of the Lasso may be easily enhanced  thanks to a simple parameter-free resampling procedure. 
Our contribution also suggests that the use of bootstrap samples by L. Breiman
in Bagging/Arcing/Random Forests~\cite{arcing} may have been so far slightly overlooked and considered a minor feature, while using boostrap samples may actually be 
a key computational feature in such algorithms for good model selection performances, and eventually good prediction performances on real datasets.

\begin{figure}
\begin{center}

\vspace*{-.05cm}

\includegraphics[scale=.55]{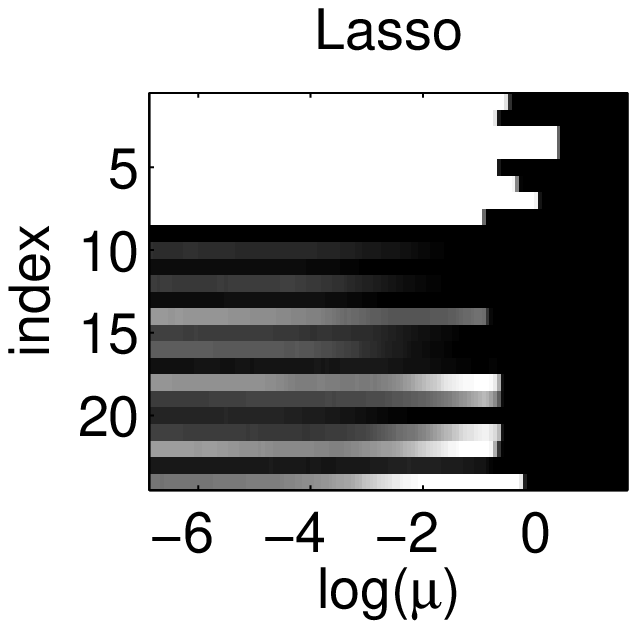}
 \hspace*{.5cm}
 \includegraphics[scale=.55]{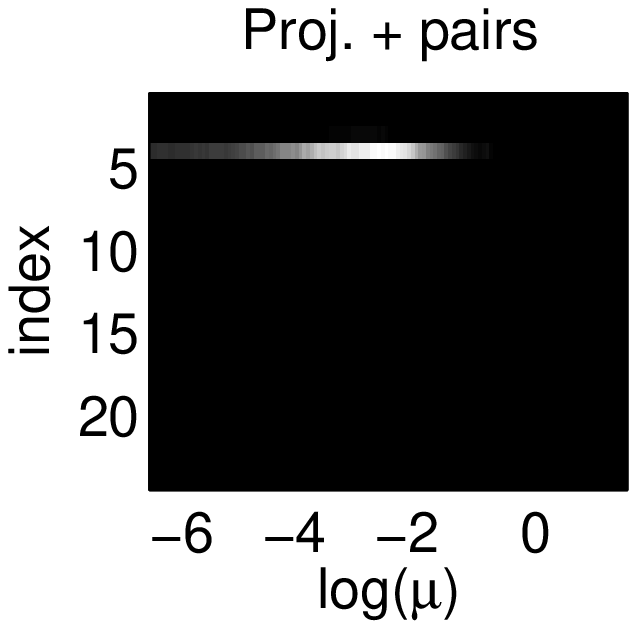}
 \hspace*{.5cm}
 \includegraphics[scale=.55]{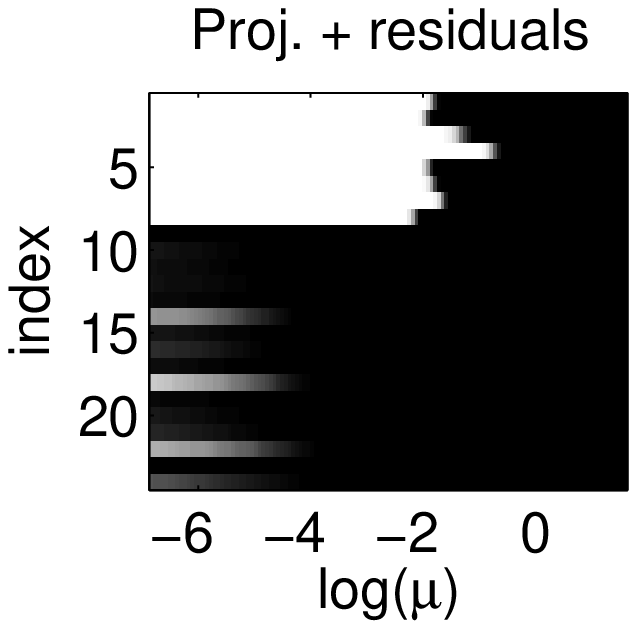}

\vspace*{.05cm}

\includegraphics[scale=.55]{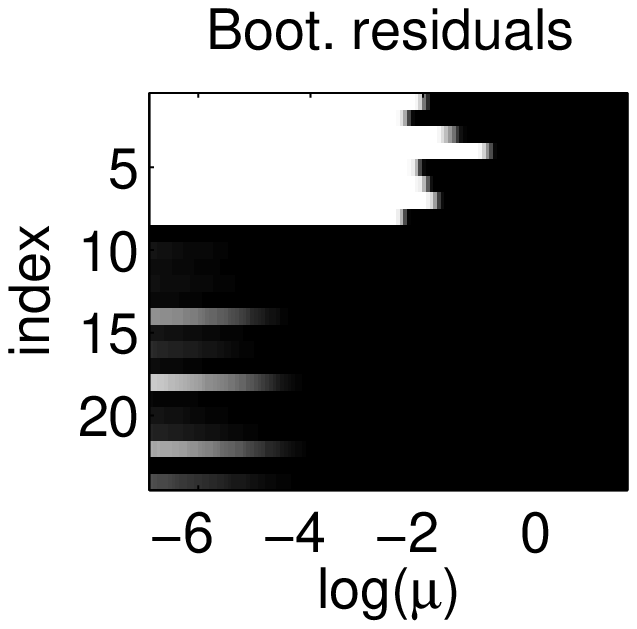}  \hspace*{.5cm}
\includegraphics[scale=.55]{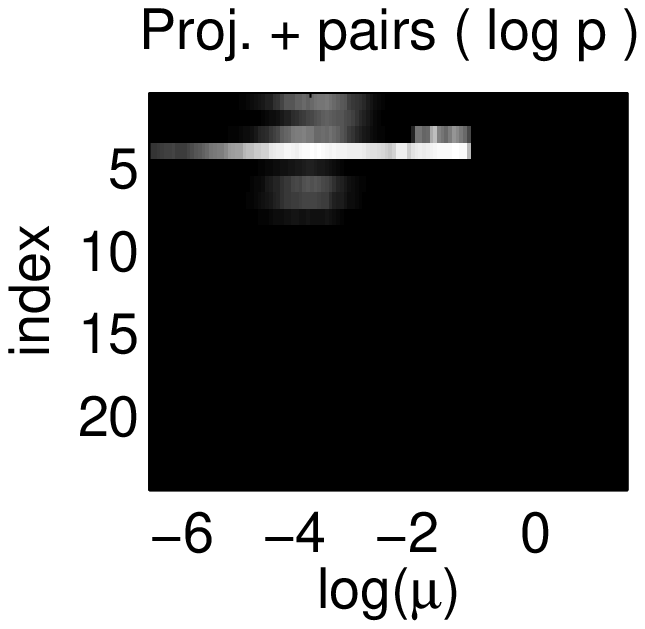}  \hspace*{.5cm}
\includegraphics[scale=.55]{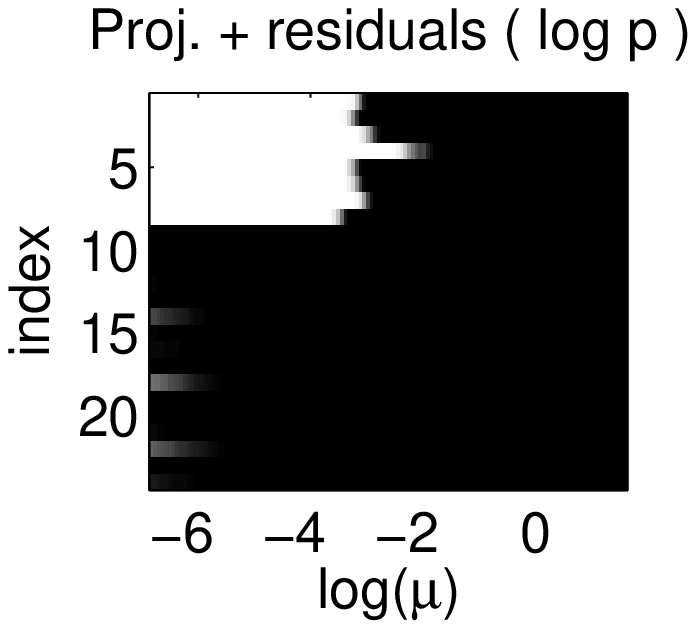}

\vspace*{-.45cm}

\end{center}
\caption{Probability of selecting each variable vs. regularization parameter  $\mu$ (high-dimensional setting) for the Lasso (left column) and the Bolasso (middle and right columns). White values correspond to probability equal to one, and black values correspond to probability equal to zero  (model consistency corresponds to while on the top 8 variables and black on the rest). See text for details.}
\label{fig:highdim-intersection}

\end{figure}

\begin{figure}
\begin{center}

\vspace*{-.05cm}
\includegraphics[scale=.55]{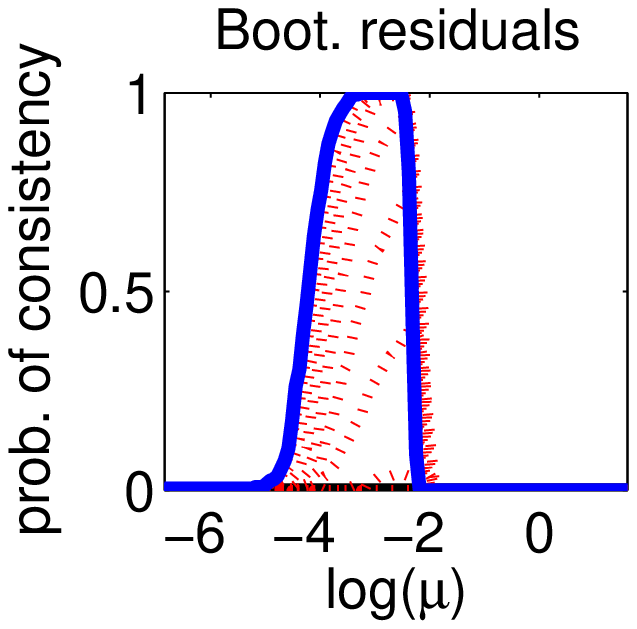} \hspace*{1cm}
\includegraphics[scale=.55]{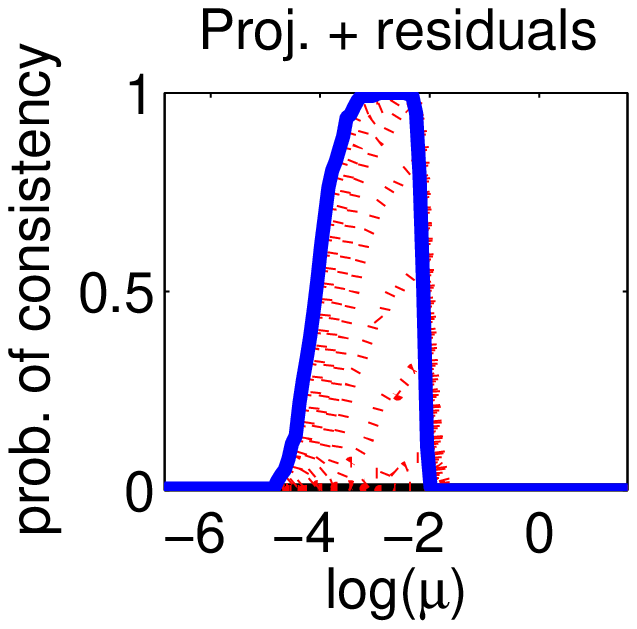}

\vspace*{.05cm}

\includegraphics[scale=.55]{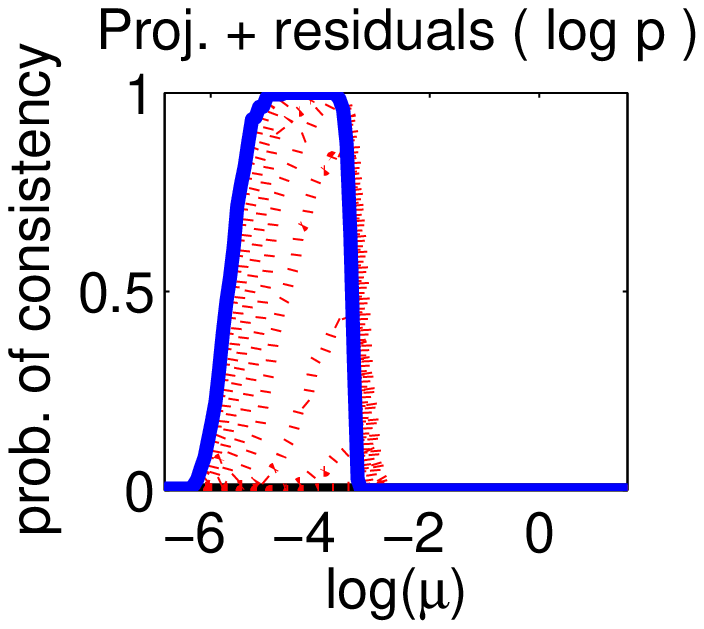} \hspace*{1cm}
\includegraphics[scale=.55]{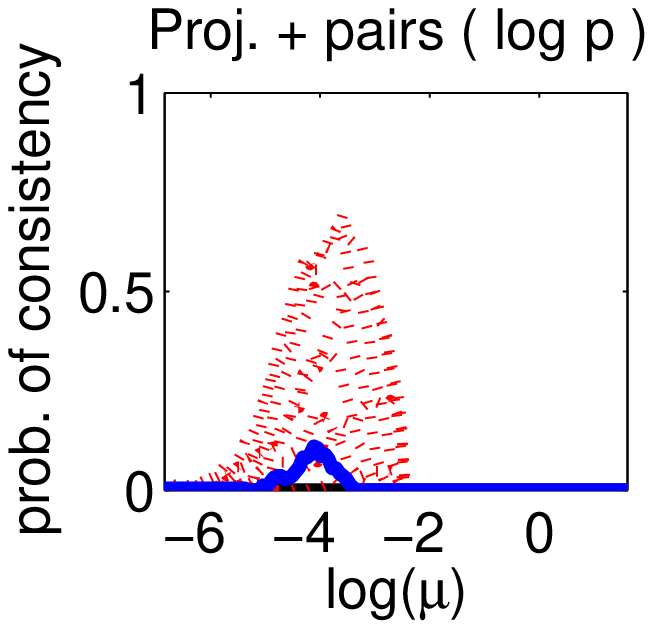}

\vspace*{-.35cm}

\end{center}
\caption{Probability of correct pattern selection with various   numbers $m$ of replications in 
$\{1 \mbox{ in plain black},2,4,8,16,32,64,128,256,  \mbox{ all in dashed red}, 512 \mbox{ in plain blue}\}$ (high-dimensional setting). Top: consistency condition of the Lasso is satisfied, Bottom: consistency condition not satistfied. Note that only one replication (plain black) is very similar to the regular Lasso.}
\label{fig:highdim-effectofm}

\end{figure}

  The current work could be extended in various ways: 
  first, we have not proved yet that bootstrapping residuals, while giving nice empirical performance, is consistent in terms of model selection. Second, a similar analysis could be applied to    other settings than least-square regression with the $\ell^1$-norm, namely regularization by block $\ell^1$-norms~\cite{grouped}, multiple kernel learning~\cite{grouped}, more general hierarchical norms~\cite{cap,hkl}, and other losses such as general convex classification losses; in particular, an extension of our results to well-specified generalized linear models is straightforward, as they are locally equivalent to a problem like in \eq{lasso-eq}, i.e., locally they are equivalent to  minimizing $ \frac{1}{2} (w -\w)^\top Q (w -\w)- q^\top(w-\w) + \reg  \| w\|_1$, with $q$ being random and having as covariance matrix a multiple of $Q$.  
  
  Moreover, extensions to general misspecified models or models with heteroscedastic additive noise could be carried through. Also, theoretical and practical connections could be made with other work on resampling methods and boosting~\cite{boosting}. In particular, using the bootstrap to both select the model and estimate the regularization parameter is clearly of interest.
  Finally, applications of such resampling techniques for signal processing and compressed sensing~\cite{cs1,cs2} remain to be explored, both in the context of basis pursuit ($\ell^1$-norm regularization, \cite{chen}) and matching pursuit (greedy selection, \cite{Mallat93matching}).

\iftrue
\appendix

\section{Probability results}
\label{app:proba}

In this appendix, we review concentration inequalities that we will need throughout the proofs.

\subsection{Multivariate Berry-Esseen Inequalities}
\label{app:berry}
If $X_1,\dots,X_n \in \rb^p$ are $n$ independent (but not indentically distributed) random vectors, with finite third-order moments, and normalized second-order moments, i.e., such that ${\rm var}(n^{-1/2} \sum_{i=1}^n X_i ) = \idm$, then for all convex sets $C$, we have the multivariate Berry-Esseen inequality~\cite{bentkus,gotze}:
 \BEQ
\label{eq:berry}
\left| \P \!\left( \frac{1}{n^{1/2}}   \sum_{i=1}^n \displaystyle X_i \in C \right) \! -\! \P( u \in C) \right| 
 \leqslant C^{\rm BE}_1 \frac{p^{1/2}}{  n^{1/2}} \left( \frac{1}{n} \sum_{i=1}^n \E  \| X_i \|_2^3 \right),
\EEQ
where $u$ is a standard normal random vector and $C^{\rm BE}_1 $ is a universal constant.

We can also derive from \cite{gotze} another version for expectation of bounded Lipschitz functions, i.e, if $f(x)$ is bounded by $M_1$ and Lipschitz, with Lipschitz constant $M_2$, then, we have:
\BEQ
\label{eq:berryL}
\left| \E f \! \left(\! n^{-1/2}   \sum_{i=1}^n \displaystyle X_i  \! \right) \! - \! \E f(u) \right| \leqslant C^{\rm BE}_2   ( M_1 \!+ \! M_2 )   \frac{p^{1/2}}{  n^{1/2}}  \left( \frac{1}{n} \sum_{i=1}^n \E  \| X_i \|_2^3 \right),
\EEQ
where $ C^{\rm BE}_2$ is a universal constant. Note that better bounds (with better scalings in $p$) exist in the i.i.d. case~\cite{bentkus}.  Any improvement on Berry-Esseen inequalities would lead to an improvement of our results.

In this paper, we will consider convex sets corresponding to selecting a given sign pattern (among the $3^p$ available ones), making use of \eq{berry}. When considering leaving out a given variable (like in Appendix~\ref{app:proofs-lowdim-missing}), we will design a specific Lipschitz function and apply \eq{berryL}.
 
 \subsection{Concentration Inequalities for Subgaussian Variables}
 \label{app:quadratic} 
 We consider $n$ independent real random variables  $Y_1,\dots,Y_n$, which are subgaussian with zero mean and uniform subgaussian constant, i.e., there exists $\tau >0 $ such that for all $i \in \{1,\dots,n\}$ and all $s\in \rb$,
$\E( e^{sY_i}) \leqslant e^{s^2 {\tsig}^2/2}$. Then, we have~\cite{massart-concentration,concentration}:
\BEQ
\label{eq:concentration}
\P\left( \frac{1}{n} \sum_{i=1}^n Y_i \geqslant t \right) \leqslant e^{-n t^2/2\tsig^2}.
\EEQ
Note  that the variance of $Y$ is always then less than $\tsig^2$ (with equality if and only if $Y$ is normally distributed). We will also use Hoeffding inequality for bounded variables, which amounts to use the fact that if $|Y|\leqslant M$, then $Y$ is subgaussian with constant $\tau^2 = M^2/4$~\cite{concentration,massart-concentration}.

We will also use concentration inequalities for quadratic forms~\cite{wright} in independent random subgaussian variables, with universal  strictly positive constants $C_1^{\rm q}$, $C_2^{\rm q}$, $C_3^{\rm q}$: for all symmetric matrices $A$, if $|A|$ denotes the matrix of absolute values of elements of $A$, then
\BEQ
\label{eq:quadratic}
\P \!\left(  Y^\top A Y \! -\! \E ( Y^\top A Y ) \! \geqslant t \right)
\! \leqslant \! C_1^{\rm q} \exp \!\left(\! -\! \min \! \left\{ \!\frac{C_2^{\rm q} t \tau^{-2}}{  \| |A| \|_2},
 \frac{C_3^{\rm q} t^2 \tau^{-4}}{ \| A\|_F^2 } \right\} \! \right) \! .
\EEQ

\section{Perturbation of positive matrices}
\label{app:perturbation}
In this appendix, we review known results of perturbation of positive matrices. Let $Q$ and $R$ be two positive matrices of size $p$, $A$ and $B$ two disjoint subsets of $\{1,\dots,p\}$ such that $A\cup B = \{1,\dots,p\}$. We have~\cite{horn}:
\BEAS
\| Q^{-1} - R^{-1} \|_{\rm 2} & \!\!\leqslant \!\! & \frac{1}{\lmin(Q) \lmin(R)} \| Q- R \|_{ 2 } ,\\
\| Q^{1/2} - R^{1/2} \|_{\rm 2} & \!\!\leqslant \!\! & \frac{ \max \{\lmax(Q),\lmax(R)\}^{1/2}}{
2 \max\{\lmin(Q),\lmin(R)\}}
\| Q- R \|_{ 2 }, \\
\| Q^{-1/2} - R^{-1/2} \|_{\rm 2} & \!\!\leqslant \!\! & \frac{ 1 }{
2\max\{\lmin(Q),\lmin(R)\}^{3/2}}
\| Q- R \|_{ 2 } ,
\EEAS
\BEAS
 \| Q_{A,B}Q_{B,B}^{-1}\! - \! R_{A,B}R_{B,B}^{-1} \|_{\rm 2} & \!\!\leqslant \!\! & \! \!\!
\frac{ \| Q_{A,B}\! -\! R_{A,B} \|_{ 2 }}{ \lmin(Q_{B,B})} 
\! + \! \frac{ \lmax(R_{A,A})^{1/2}}{ \lmin(R_{B,B})^{3/2}} \| Q_{B,B}\!-\! R_{B,B} \|_{ 2 }. 
\EEAS

\section{Optimization lemmas}

 The following three lemmas give error bounds on the Lasso estimates and conditions for a sign pattern $s \in \{-1,0,1\}^p$ to be the one of the unique solution  $\hat{w}$ to \eq{lasso} or \eq{lasso-eq}.

\begin{lemma}
\label{lemma:opt}
Assume \hypref{model} and  \hypref{inv}.
Let $s \in \{0,-1,1\}^p$  and $J = \{ j, s_j \neq 0 \}$. Then $s$ is selected (i.e., $\sign( \hat{w}) = s$) if and only if:
\BEA
  \label{eq:A00}  
 &&   \|   Q_{J^c,J}Q_{J, J}^{-1} q_J-  q_{J^c} - Q_{J^c, J^c} \w_{J^c}      - \reg Q_{J^c,J} Q_{J, J}^{-1}  s_J
   \  \|_\infty \leqslant \reg , \\
\label{eq:B00}  &&      \sign  (  \w_J + Q_{J,J}^{-1} q_J - \reg Q_{J, J}^{-1}  s_J   ) = s_J.
\EEA 
The solution then satisfies $\hat{w}_J = \w_J + Q_{J,J}^{-1}( q_J - \mu s_J)$.
\end{lemma}

\begin{proof}
Following standard results in non-smooth convex optimization~\cite{boyd,bonnans}, $w$ is optimal for \eq{lasso} or \eq{lasso-eq}, if and only if, for all $j \in \{1,\dots,p\}$ such that $w_j \neq 0$, then,
$ [Q(w - \w)]_j- q_j + \mu \sign( w_j) = 0$, and for all other $j$, $|[Q(w - \w)]_j- q_j | \leqslant \mu$. We thus get $\hat{w}_J = \w_J +  Q_{J,J}^{-1}( q_J - \mu s_J)$, and the result follows from expressing that $\hat{w}_J$ should have the right sign on $J$---\eq{B00}---and that the directional derivatives along other directions are positive---\eq{A00}.
\end{proof}

When the sign pattern is consistent on $\J$ with $\w$, then we can further refine the conditions of  Lemma~\ref{lemma:opt}:

\begin{lemma}
\label{lemma:opt2}
Assume \hypref{model} and  \hypref{inv}.
Let $s \in \{0,-1,1\}^p$ such that $s_{\J} = \sign( \w_\J)$ and let $J = \{ j, s_j \neq 0 \} \supset \J$. Then $s$ is selected if and only if:
\BEA
  \label{eq:A}  
 &&   \|   Q_{J^c,J}Q_{J, J}^{-1} q_J-  q_{J^c}     - \reg Q_{J^c,J} Q_{J, J}^{-1}  s_J
   \  \|_\infty \leqslant \reg , \\
\label{eq:B}  &&      \sign  (  \w_\J + ( Q_{J,J}^{-1} q_J - \reg Q_{J, J}^{-1}  s_J)_\J  ) = \sign( \w_\J),
\\
\label{eq:C}   &&      \sign  (  Q_{J,J}^{-1} q_J - \reg Q_{J, J}^{-1}  s_J)_{J \backslash \J}   = s_{J \backslash \J}.
\EEA 
The solution then satisfies $w_J = \w_J + Q_{J,J}^{-1}( q_J - \mu s_J)$.

\end{lemma}

\begin{lemma}
\label{lemma:bound}
Assume \hypref{model} and \hypref{inv}. We have
$ \|  \hw - \w
\|_2 \leqslant  \frac{p^{1/2} \mu + \| q\|_2 }{\lminQ} $ and
$ \| Q^{1/2}(\hw - \w)
\|_2 \leqslant  \frac{p^{1/2} \mu + \| q\|_2 }{\lminQ^{1/2}} $.
\end{lemma}
\begin{proof}
From optimality conditions, we have $\| Q ( \hw - \w) - q \|_\infty \leqslant \mu$, from which we get $
\| Q^{1/2}(\hw - \w)
\|_2 \leqslant  \lminQ^{-1/2} \left(
\| Q(\hw - \w) - q
\|_2 + \|   q\|_2\right)$.
The results follow from the identity $\|a\|_2 \leqslant p^{1/2} \|a\|_\infty$ for any $a \in \rb^p$.\end{proof}

The following lemma relates the solutions of \eq{lasso-eq} for different values of $Q$ and $q$. This will be used in Appendix~\ref{app:proofs-lowdim-missing} to prove the Lipschitz continuity of the solution of \eq{lasso-eq} as a function of $q$.

\begin{lemma}
\label{lemma:lipschitz}
If $\hat{w}$ is solution of \eq{lasso-eq} for $Q,q$, and $\hat{w}'$ is solution for $Q',q'$, then we have:
$$
\|Q^{1/2}(\hat{w} - \hat{w}') \|_2
\!\leqslant \!  2 \| Q^{-1}(q - q')\|_2  
+\frac{2 \| (Q')^{-1/2}( Q - Q')Q^{-1/2}\|_{ 2 }}{ \lambda_{\min}(Q')^{1/2}}
 \! \left[ p^{1/2} \mu\! + \!\|q'\|_2
\right] .
$$
Let $\gamma = \frac{Q^{1/2} ( \hat{w} - \w - Q^{-1} q)}{\mu}$, then if $Q=Q'$,
$
\|\gamma - \gamma'\|_2 \leqslant \frac{ 3 \| Q^{-1}(q - q')\|_2 }{\mu }
$, and if $q = q'$, then
$
\|Q^{-1/2}\gamma - (Q')^{-1/2} \gamma'\|_2 \leqslant \frac{2 \| (Q')^{-1/2}( Q - Q')Q^{-1/2}\|_{ 2 }}
{\mu \lambda_{\min}(Q)^{1/2}\lambda_{\min}(Q')^{1/2}}  \left[ p^{1/2} \mu + \|q'\|_2
\right] 
$.
\end{lemma}
\begin{proof}
  We let denote $J(w) = \frac{1}{2} (w - \w)^\top Q ( w - \w) - q^\top (w - \w) + \mu \| w\|_1$ the Lasso cost function. A short calculation shows that for all $z$ such that $z^\top Q z = 1$, $J(\hat{w}' + \alpha z)- J(\hat{w}') $ is larger than
$$
\frac{\alpha^2 }{2} - \left( \| Q^{-1/2} ( q - q')\|_2   +
\| (Q')^{1/2} ( \hat{w}' - \w) \|_2 \| (Q')^{-1/2}( Q - Q')Q^{-1/2}\|_{ 2 } \right) \alpha.
$$
If the last expression is nonnegative,
since $J$ is convex, the (unique, because $Q$ is invertible) minimum $\hat{w}$ of $J$ must occur within the convex set $\{ w, \| Q^{1/2}(w - w') \|_2 \leqslant \alpha\}$.
The first result follows, using Lemma~\ref{lemma:bound}. Other results are direct consequences of using results from Appendix~\ref{app:perturbation}.
\end{proof}

\section{Proofs for low-dimensional results}
 \label{app:proofs-lowdim}

Note that assumption \hypref{bounded} implies a bound on the largest eigenvalue of the matrix $Q$, i.e., 
$\lmaxQ \leqslant p\| Q \|_\infty \leqslant p M^2$. Moreover, we have
 for all $J \subset \{1,\dots,p\}$, $k \in \{1,\dots,n\}$ and $j \in J^c$:
$$
|x_{kj} - Q_{j,J}Q_{J,J}^{-1} x_{kJ}|
 \leqslant  M + \frac{ M }{\lminQ^{1/2}} \times |J|^{1/2} M \leqslant  \frac{ 2M |J|^{1/2}}{\tl^{1/2}} ,
$$
which leads to 
\BEQ
\label{eq:qJCJ}
\P( \| q_{J^c} - Q_{J^c,J}Q_{J,J}^{-1} q_J \|_\infty \geqslant t M\sigma ) \leqslant 2 p \exp
\left( - \frac{  t^2  \tl}{  8 \ttau^2     } \frac{n}{|J|} \right).
\EEQ

 \subsection{Proof of Proposition~\ref{prop:lowdim-heavy}}
 \label{app:proofs-lowdim-heavy}
 The null vector $\hat{w}=0$ is solution of \eq{lasso-eq}, if and only if
$\| Q \w + q \|_\infty \leqslant \mu$, which is the case, as soon as
$\mu \geqslant M^2 \|\w\|_1 + \| q \|_\infty$ (because
$\| Q \w \|_\infty \leqslant M^2 \| \w\|_1$), and, thus with the additional assumption $\mu \geqslant 2 M^2 \| \w \|_1$, as soon as $ \|q \|_\infty \leqslant \mu / 2$.
We have, by the union bound:
$$
\P( \| q \|_\infty \!
\leqslant \mu/2 )   \!\geqslant\!   1 - \sum_{j=1}^p \P ( |q_j| 
\geqslant   \mu/2 ) 
 \geqslant   1 - 2 p \exp \!\left( \! -   \frac{  n \tmu^2}{8 \ttau^2} \right) \!,
$$
because we have $\E( e^{s x_{ij} \varepsilon_i}) \leqslant e^{s^2 \tsig^2 M^2 / 2}$ for all $s \in \rb$ and $j \in \{1,\dots,p\}$, and by  \eq{concentration}.

 \subsection{Proof of Proposition~\ref{prop:lowdim-fixed}}
 \label{app:proofs-lowdim-fixed}
If $J(w) = \frac{1}{2}(w-\w)^\top Q ( w - \w) - q^\top( w-\w) + \mu \| w\|_1$ is the Lasso cost function, we have, for all $z\in \rb^p$ such that $\| z\|_2 = 1$ and $\alpha>0$,
\BEAS
J(w_0 + \alpha z)
& \!\!\geqslant\!\! & J(w_0 ) + \lminQ \alpha^2 / 2 - q^\top \alpha z + ( \mu\!-\!\mu_0)
( \| w_0\! +\! \alpha z\|_1 \!-\! \| w_0\|_1 ), \\
&\!\! \geqslant \!\! & J(w_0 ) + \lminQ  \alpha^2 / 2 - \alpha ( \|   q\|_2 + | \mu - \mu_0|
 p^{1/2}),
\EEAS
 which implies $\|   \hat{w} - w_0 \|_2
\leqslant 2 \lminQ ^{-1} \|  q \|_2 + 2 |\mu - \mu_0| \lminQ^{-1} p^{1/2}$. The first inequality follows from $
\P( \|   q \|_2 \geqslant t ) \leqslant 2 p
\exp( - t^2 n  / 2 p M^2  \tsig^2 )
$, applied with $t = \lminQ \beta \sigma / 4 M $.

We let denote $s$ and $J$ the sign and support patterns of $v$. We have from optimality conditions of the noiseless problem, $(w_0)_J = \w_J  - \mu_0 Q_{J,J}^{-1} s_J$ and $ \| ( Q( w_0 - \w) )_{J^c} \|_\infty \leqslant 
\mu_0 - \eta M \sigma$. We now need sufficient conditions
for \eq{A00} and \eq{B00} in Lemma~\ref{lemma:opt}. For \eq{B00}, 
we need that
$ \sign  (  (w_0)_J +    Q_{J,J}^{-1} q_J + ( \mu_0 - 
\mu)  Q_{J, J}^{-1}  s_J   ) = s_J$. If $|\mu- \mu_0| \leqslant \frac{\lminQ \minf{w_0}}{  2 p^{1/2}} 
= \frac{ M\sigma \tl \minf{ w_0 M / \sigma }}{ 2p^{1/2}}$, 
then
$$
\| ( \mu_0 - 
\mu)  Q_{J, J}^{-1}  s_J   \|_\infty \leqslant  |\mu- \mu_0| \lminQ^{-1} p^{1/2}
\leqslant \minf{w_0}/2,
$$
and
then \eq{B00} is satisfied as soon as  $(Q_{J,J}^{-1}q_J)_j s_j \geqslant - \minf{w_0}/2$,
 for all $j \in J$, which occurs with probability
greater than
$
 1 - p \exp\left(
\frac{ -   n \minf{w_0 M/\sigma}^2  \tl^2}{ 8 \ttau^2  p }\right)
$.

For \eq{A00}, we assume that $\|  q_{J^c} - Q_{J^c,J}Q_{J,J}^{-1} q_J\|_\infty \leqslant \eta M \sigma / 2$, which occurs with probability obtained  from \eq{qJCJ} (with $t = \eta/2$). Also, $ |\mu - \mu_0|  \leqslant  
\frac{ \eta \lminQ^{1/2} M\sigma }{ 4 p^{1/2} M}
\leqslant 
 \frac{   \eta M\sigma/2}{ 2 p^{1/2}
 \frac{  M^{1/2}}{\lminQ^{1/2}} }
\leqslant  \frac{  M \sigma \eta/2}{ 1 + 
 p^{1/2} \frac{M^{1/2}}{
 \lminQ^{1/2}}}$. This implies 
$$
 \| q_{J^c} - Q_{J^c,J}Q_{J,J}^{-1} q_J\|_\infty \leqslant  \eta M\sigma - |\mu - \mu_0|  ( 1 + 
 p^{1/2} M^{1/2}
 \lminQ^{-1/2} ),
$$
$$ \| q_{J^c} - Q_{J^c,J}Q_{J,J}^{-1} q_J \|_\infty +   
\| Q ( w_0 - \w) \|_\infty + |\mu - \mu_0| p^{1/2} M^{1/2}
 \lminQ^{-1/2} \leqslant \mu,$$
 because $\| Q_{J,J}^{-1} s_J \|_\infty \leqslant M p^{1/2} \lminQ^{-1/2}$; hence the desired result.

 \subsection{Proof of Proposition~\ref{prop:lowdim-high}}
 \label{app:proofs-lowdim-high}
Note that  $\mu \leqslant  \frac{ \minf{\w} \lminQ }{p^{1/2}}$,
and $\| \Delta\|_2 \leqslant \lminQ^{-1} p^{1/2}$ implies that 
$\sign(\w_\J + \mu \Delta_\J) = \sign(\w_\J)$. Thus, if $\| z\|_2 = 1$ and $\alpha>0$, we have:
\BEAS
J( \w + \mu\Delta + \alpha z)
& \geqslant & J(\w + \mu \Delta ) +  \lminQ \alpha^2 / 2 - q^\top \alpha z + 
(  \mu \Delta)^\top Q \alpha z +  \\
& & 
\mu( \| \w + \mu \Delta + \alpha z\|_1 - \| \w + \mu \Delta \|_1 ), \\
& \geqslant & J(\w + \mu \Delta ) +  \lminQ \alpha^2 / 2 - q^\top \alpha z, \EEAS which implies 
$
\| \hat{w} - \w - \mu \Delta \|_2
\leqslant 2 \lminQ ^{-1} \|   q \|_2$,
ans thus the first inequality.

We let denote $s$ the sign pattern of $\Delta$ and $J$ its support.
Since, by assumption $\tmu \leqslant \frac{\minf{\tw} \tl
}{2 p^{1/2}}$, we have $ \mu \| Q_{JJ}^{-1} s_J )_\J \|_\infty \leqslant \minf{\w}/2$, if 
$\| (Q_{J,J}^{-1} q_J)_\J \|_2 \leqslant \frac{1}{2}\minf{\w}$, which occurs with probability greater than
$ 1 - 2 |\J| \exp( -  n   \frac{ \minf{\tw} \tl^2  n }{ 8  \ttau^2 p }) 
$,
then \eq{B} is satisfied.

If  $\| q_{J^c} - Q_{J^c , J} Q_{J,J}^{-1} q_J \|_\infty \leqslant \mu \eta   $, then \eq{A} is satisfied, and this occurs with probability
greater than $ 1- 2 p   \exp
\left( - \frac{\tl n \eta^2 \tmu^2}{8 |J| \ttau^2    } \right)$
(from \eq{qJCJ}). Finally, if
for all $j \in J \backslash \J$, $(Q_{J,J}^{-1} q_J)_j s_j \geqslant - \mu |\Delta_j|$, then \eq{C} follows. 
This occurs with probability greater than
$1 - p \exp( -      \tl^2 \minf{M^2 \Delta}^2 \tmu^2 n /2\ttau^2p)$.
The result follows by the union bound.

 \subsection{Proof of Proposition~\ref{prop:lowdim-medium}}
 \label{app:proofs-lowdim-medium}

The optimality condition in \eq{B} from Lemma~\ref{lemma:opt2} is satisfied as long as
$\tmu \leqslant  \frac{ \minf{\tw} \tl }{2 p^{1/2}}$, and
$\| (Q_{J,J}^{-1} q_J)_\J \|_2 \leqslant \frac{1}{2}\minf{\w}$, which occurs with probability greater than
$ 1 - 2 |\J| \exp( -  n   \minf{\tw} \tl^2  n / 8  \ttau^2 p ) 
$,
while the intersection of events in \eq{A} and \eq{C}, by the Berry-Esseen inequalities, converges to  the probability that
$\P( u \in C )$, where
$u$ is normal with zero mean and covariance matrix
$Q$ and $C$ is the convex set defined as the intersection of
$$
\left\{ \left[  ( Q_{J,J}^{-1} u_J)_{J\backslash \J} 
- \mu n^{1/2} \sigma^{-1} 
 (Q_{J,J}^{-1} s )_{J \backslash \J} \right]\circ s_{J \backslash \J}  \geqslant 0 \right\}
$$
and
$$
\left\{
\| u_{J^c} - Q_{J^c,J} Q_{J,J}^{-1} u_J  - \mu n^{1/2} \sigma^{-1} Q_{J^c,J} Q_{J,J}^{-1} s_J \|_ \infty   \leqslant \mu n^{1/2} \sigma^{-1} \right\}.
$$
The set $C$ and its complement have non-empty interior and since $Q$ is full-rank, the probability is strictly inside the interval $(0,1)$.
Moreover, by \eq{berry},  the error bound is upperbounded by $C^{\rm BE}_1$ times
\BEAS
  \frac{p^{1/2}}{ n^{1/2}}
\left( \frac{1}{n} \sum_{i=1}^n  \E \| (\sigma^2 Q)^{-1/2} \varepsilon_i x_i \|_2^3 \right)
& \! \!\!\!\leqslant \!\!\!  \! & 
  \frac{p^{1/2}}{ n^{1/2}} \frac{ 4  M p^{1/2} \tau^3 }
 { \sigma^3 \lminQ^{1/2}} \! \left( \frac{1}{n} \sum_{i=1}^n  \E \|  Q^{-1/2}  x_i \|_2^2 \!\right) ,\\
& \! \!\!\!\leqslant \! \!\! \! & 
   \frac{p^{1/2}}{ n^{1/2}} \frac{ 4  M p^{1/2} \tau^3 }
 { \sigma^3 \lminQ^{1/2}} p
 = 
  \frac{p^{2}}{ n^{1/2}} \frac{ 4  \ttau^3  }
 { \tl^{1/2}},
 \EEAS
 because $\E |\varepsilon_i|
^3 \leqslant 4 \tau^3$, which leads to the desired result.

 \subsection{Proof of Proposition~\ref{prop:lowdim-medium2}}
 \label{app:proofs-lowdim-medium2}
We simply use Lemma~\ref{lemma:bound}: $\| \hw - \w
\|_2 \leqslant  \frac{p^{1/2}  \mu  + \|  q\|_2}{\lminQ }. $
Thus if $\minf{\w}  >  
 \frac{p^{1/2} \mu + \|   q\|_2}{\lminQ } $, the result follows from concentration inequalities in Appendix~\ref{app:proba}.

 \subsection{Proof of Proposition~\ref{prop:lowdim-low}}
 \label{app:proofs-lowdim-low}
We have, by considering all patterns consistent with the total absence of zeros:
$$ \P( \exists j \in \{1,\dots,p\}, \hat{w}_j = 0 )
\leqslant \sum_{s,  \exists j \in \{1,\dots,p\}, s_j = 0 } 
\P( \sign(\hat{w}) = s).
$$
We now consider such a pattern and its support (strictly included in $\{1,\dots,p\}$). 
From optimality conditions in \eq{A00}, we get that $\sign(\hat{w}) = s$ implies that
$\|q_{j } - Q_{j,J}Q_{J,J}^{-1} q_J - \mu Q_{j,J}Q_{J,J}^{-1}s_J \|_\infty \leqslant \mu$ for some $j \in J^c \neq \varnothing$. 
Note that the covariance matrix of $
q_{j } - Q_{j,J}Q_{J,J}^{-1} q_J $ is equal to $\sigma^2 Q_{j,j|J} /n$ and has a lowest eigenvalue greater than
$\sigma^2 \lminQ/n$.
Thus, by the Berry-Esseen inequality,
$$
\P( \sign(\hat{w}) = s)
\leqslant C^{\rm BE}_1    \frac{ 4  \ttau^3  }
 { \tl^{1/2}} \frac{p^{2}}{ n^{1/2}}   +  \frac{\tmu n^{1/2}}{\tl^{1/2}} ,
$$
which implies the desired result, since there are at most $3^{p}$ allowed patterns.

We can get a better bound (with respect to $p$), but with a weaker dependence in $\mu$, that is, we consider:
$$
\P( \exists j \in \{1,\dots,p\}, \hat{w}_j = 0 )  \leqslant
\P( \exists j \in \J , \hat{w}_j = 0 ) + 
\P( \exists j \in \J^c, \hat{w}_j = 0 ) .
$$
The first term is upper bounded by $ 2 |\J| \exp( -  n   \minf{\tw} \tl^2  n / 8  \ttau^2 p ) $, while the second one is upper-bounded using Proposition~\ref{prop:missingone}. This leads to the global desired upper bound,
which scales better in $p$ but worse in $n$. In particular, it requires that $\mu n$ tends to infinity, i.e., $\mu$ is not too small.

 \subsection{Proof of Proposition~\ref{prop:missingone}}
 \label{app:proofs-lowdim-missing}

 We first start with a simple elementary lemma:
 \begin{lemma}
\label{lemma:special}
If $u \in \rb$ is a standard normal random variable, then
$$  \alpha   \geqslant \P( | u - \beta | \leqslant \alpha )
\geqslant \frac{\alpha}{1+\alpha} e^{ -  \beta^2/2  }.$$
\end{lemma}
 When minimizing \eq{lasso-eq}, with the constraint that
 $w_{{j}}=0$, we get the solution (with the notation ${j}^c = \{ {j} \}^c$):
 $$
 w_{{j}^c} =\w_{ {j}^c } +  Q_{{j}^c , {j}^c}^{-1} q_{{j}^c} + \mu 
 Q_{{j}^c , {j}^c}^{-1/2} \gamma_{\mu,Q}^{{j}}(q_{{j}^c }),
 $$
 for a certain $\gamma_{\mu,Q}^{{j}}(q_{{j}^c })$ such that
    $\|  Q_{{j}^c , {j}^c}^{1/2}  \gamma_{\mu,Q}^{{j}}(q_{{j}^c }) \|_\infty \leqslant 1$. It is optimal for the full problem if and only if (because $\w_{{j}} = 0$),
 $
| Q_{ {j} , {j}^c} ( w_{{j}^c} - \w_{ {j}^c } ) - q_{{j}} | \leqslant \mu$,
 i.e.,
 $$|
- q_{{j}} + Q_{{j}, {j}^c} Q_{{j}^c , {j}^c}^{-1}q_{{j}^c }   + \mu  Q_{ {j} , {j}^c} 
   Q_{{j}^c , {j}^c}^{-1/2}  \gamma_{\mu,Q}^{{j}}(q_{{j}^c })
| \leqslant \mu.
 $$
 We consider the ``soft indicator'' function (triangle-shaped) $f_{\mu,Q}^{{j}}(q)$ of $q$ defined as
 $$
 f_{\mu,Q}^{{j}}(q) =   \left( 1 -  \mu^{-1} \left|  q_{{j}} - Q_{{j}, {j}^c} Q_{{j}^c , {j}^c}^{-1}q_{{j}^c }   -  Q_{ {j} , {j}^c} 
 \mu  Q_{{j}^c , {j}^c}^{-1/2} \gamma_{\mu,Q}^{{j}}(q_{{j}^c })
\right| \right)_+ .
 $$
 The function $f_{\mu,Q}^{{j}}$ is upper bounded by $1$, moreover, from Lemma~\ref{lemma:lipschitz},  $\gamma_{\mu,Q}^{{j}}$ is Lipschitz with constant $L = \mu^{-1} \frac{3}{ \lambda_{\min}(Q)^{1/2}}$ . Thus, $f_{\mu,Q}^{{j}}$ is Lipshitz with constant
$ 2 \mu^{-1} \frac{M^{1/2}}{ \lambda_{\min}(Q)^{1/2}} + L M
\leqslant  5 \mu^{-1} \frac{M }{ \lambda_{\min}(Q)^{1/2}}
$.

 Moreover (by design) we have
 $$
 f_{\mu,Q}^{{j}}(q) \leqslant 1_{ |
- q_{{j}} + Q_{{j} , {j}^c}^{-1/2}  Q_{{j}^c , {j}^c}^{-1/2} q_{  {j}^c }   + Q_{ {j} , {j}^c} 
 \mu Q_{{j}^c , {j}^c}^{-1/2}  \gamma_{\mu,Q}^{{j}}(q_{{j}^c })
| \leqslant \mu},
$$
 thus 
 $\E  f_{\mu,Q}^{{j}}(q) \leqslant \P( {j} \notin \hat{J})$. This implies by the Berry-Esseen bound~(see Appendix~\ref{app:berry}), that, if $q_{\rm G}$ denotes the Gaussian approximation:
 $$
 \P( {j} \notin \hat{J})
 \geqslant \E f_{\mu,Q}^{{j}}(q_{\rm G} ) -  C^{\rm BE}_2 \frac{p^{1/2}}{ 
 n^{ 1/2}}
\left( \frac{5 p^{1/2}   n^{-1/2}}{\tmu \tl^{1/2}} \! +\! 1\! \right) 
\frac{     4 \ttau^3 p^{3/2}}{  \tl^{1/2}
 }  ,
 $$
 because  the average third order moment of the normalized variable is equal to $\frac{     4 \ttau^3 p^{3/2}}{  \tl^{1/2}
 } $ and the Lipshitz constant of the function of the normalized variable is equal to 
$ \frac{4 \mu^{-1} M }{ \lambda_{\min}(Q)^{1/2}}  \times
n^{-1/2} \sigma p^{1/2} M =  \frac{5 p^{1/2}   n^{-1/2}}{\tmu \tl^{1/2}}  $. 

  Moreover, we can lower bound, for any $q$,
$$
\E f_{\mu,Q}^{{j}}(q )  \geqslant \frac{1}{2} \P( |
- q_{{j}} + Q_{{j}, {j}^c} Q_{{j}^c , {j}^c}^{-1}q_{{j}^c }   + \mu Q_{ {j} , {j}^c} 
 Q_{{j}^c , {j}^c}^{-1/2} \gamma_{\mu,Q}^{{j}}(q_{{j}^c })
| \leqslant \mu / 2 ).
$$
When applied to the Gaussian limiting distribution $q_{\rm G} $, we know that the random variable
$n^{1/2} \sigma^{-1} ( - q_{{j}} + Q_{{j}, {j}^c} Q_{{j}^c , {j}^c}^{-1}q_{{j}^c } ) $ is asymptotically normal with mean zero and covariance
 $\kappa^2 = Q_{{j} ,{j} | {j}^c}$. We get by applying  Lemma~\ref{lemma:special} with
$\beta = \frac{\mu n^{1/2} \sigma^{-1}}{
  \kappa  } Q_{{j}, {j}^c} Q_{{j}^c , {j}^c}^{-1/2}  \gamma_{\mu,Q}^{{j}}(q_{{j}^c } )$ and
  $\alpha = 
   \frac{\mu n^{1/2} \sigma^{-1}}{
  \kappa}
  $, which are such that
  $|\beta| \leqslant 
\frac{\mu n^{1/2} \sigma^{-1} M p^{1/2}}{
  \kappa \lambda_{\min}(Q)^{1/2}} $:
$$
\E  \left[ f_{\mu,Q}^{{j}}( q_{\rm G} ) | ( q_{\rm G} )_{{j}^c} \right] \geqslant 
\frac{\frac{\mu n^{1/2} \sigma^{-1}}{
2 \kappa }}{
1+ \frac{\mu n^{1/2} \sigma^{-1}}{
2 \kappa }
}
\frac{1}{2}  \exp \left[ - \frac{\mu^2 n}{ 2\sigma^2 \kappa^2}
 M  \lminQ^{-1 } p
\right],
$$
which leads to, with $\kappa \leqslant  M$ and $\kappa \geqslant \lminQ^{1/2}$,
$$
\E  f_{\mu,Q}^{{j}}( q_{\rm G} ) \geqslant
\frac{\frac{\mu n^{1/2} \sigma^{-1}}{
2 M }}{
1+ \frac{\mu n^{1/2} \sigma^{-1}}{
2 \lminQ^{1/2}}
} \frac{1}{2} \exp  \left[- \frac{ \mu^2 np }{2
\sigma^2  }
 \frac{M^2}{\lminQ^2}
\right]. $$

Similarly, we can get an upper bound on the probability of not selecting  the variable ${j}$. We consider the same technique, but we now need to upperbound a probability of the type
$
\E  f_{\mu,Q}^{{j}}( q_{\rm G} ) $, which leads to the desired result.

 \subsection{Proof of Proposition~\ref{prop:lowdim-copies}}
 \label{app:proofs-lowdim-copies}
Following the analysis in \mysec{support}, we need to upper bound
$\P(  j \in  \hat{J}^\ast )^m $ and $\P(      (\hat{J}^\ast)^c \cup \J   \neq \varnothing )$.
We obtain $\P(      (\hat{J}^\ast)^c \cup \J   \neq \varnothing )  \leqslant 2p
 \exp\left( - \frac{\minf{\tw} \tl^2 }{8\ttau^2 p} n\right)$ from 
Proposition~\ref{prop:lowdim-medium2}. From
Proposition~\ref{prop:missingone}, we get
$$ \P(  j \in  \hat{J}^\ast )
\geqslant  \frac{   \tmu n^{1/2} / 4  }{
1+  \tmu n^{1/2} / 2 \tl^{1/2}} 
 \exp \left( -  \frac{   \tmu^2 }{2\tl^2} n p    \right) -   \frac{10 C^{\rm BE}_2 }{\ttau^3 \tl^{1}} \frac{p^{3}}{\tmu n p^{1/2}}     -    \frac{     4 C^{\rm BE}_2 \ttau^3 p^{5/2}}{  \tl^{1/2}
 }
  .$$
  We let $h(c) = \frac{1}{2} \frac{   c / 4  }{
1+  c / 2 \tl^{1/2}} 
 \exp \left( -  \frac{ 2 c^2 }{\tl^2}      \right)$, and
 $g(c) = \left( \frac{8 C^{\rm BE}_2 }{\ttau^3 \tl^{1}} \frac{1}{c}  +     \frac{4C^{\rm BE}_2 }{\ttau^3 \tl^{1/2}}  \right)^{2} h(c)^{-2}$, and $f(c) = - \log( 1 - h(c) )$, to get the desired result.

 \subsection{Proof of Proposition~\ref{prop:lowdim-cutting}}
 \label{app:proofs-lowdim-cutting}
Using the same reasoning as in Appendix~\ref{app:proofs-lowdim-copies}, we get the same $f(c) $
and $a(c) = \frac{10 C^{\rm BE}_2 }{\ttau^3 \tl^{1}} \frac{1}{c}  +     \frac{4C^{\rm BE}_2 }{\ttau^3 \tl^{1/2}}$.

\section{Proofs for boostrapping pairs}
\label{app:proofs-bolasso}

\subsection{Concentration inequalities}
 \label{app:concentration-bolasso}

 We now assume that we have a bootstrap sample $X^\ast$ and $y^\ast$, which leads to $Q^\ast$ and $q^\ast$. We now derive concentration inequalities for $q^\ast$ and $Q^\ast$, that we use in 
 Appendix~\ref{app:proofs-bolasso-2}.

 For all $a,b\in \{1,\dots,p\}$, $ Q^\ast_{ab} $ is an average of variables bounded by $M^2 $. Thus, by Hoeffding's inequality~\cite{concentration} and the union bound:
\BEQ
\label{eq:QQast}
\P(  \| Q^\ast - Q\|_\infty \! \geqslant  \! t M^2 ) 
\leqslant 2p^2 \exp \left(- 2 n t^2  \right).
\EEQ
 Similarly, we bound the deviation between $q$ and $q^\ast$:
\BEQ
\label{eq:qqast}
\P(   \|   q-q^\ast \|_\infty  \geqslant t M\sigma | \varepsilon ) 
\leqslant 2 p \exp \left(- 2 n t^2  \frac{\sigma^2}{\| \varepsilon \|_\infty^2} \right).
\EEQ
Also, by the central limit theorem, given $\varepsilon$, $n^{1/2}( q^\ast - q)$ converges in distribution to a normal variable with mean zero and covariance matrix
$$
\sigma^2 \widetilde{Q} = 
\E  \left[ (\varepsilon_1^\ast)^2
 x_1^\ast   (x_1^\ast)^\top | \varepsilon \right] - \E [ \varepsilon_1^\ast x_1^\ast| \varepsilon ]
 \E [ \varepsilon_1^\ast x_1^\ast| \varepsilon ]^\top
 = \frac{1}{n} \sum_{i=1}^n \varepsilon_i^2 x_i x_i^\top - qq^\top.
$$
We can derive concentration inequalities of   $\widetilde{Q}$ around $Q$, by using Appendix~\ref{app:quadratic} and \eq{quadratic}:
\begin{lemma}
 \label{lemma:concentrationbolasso}
Assume \hypreff{model}{inv}. We have:
 $$
\P( \| \widetilde{Q} - Q\|_\infty\! \geqslant \! t M^2 ) \! \leqslant \!
 2 p^2 C_1^{\rm q} \exp \!\left(\!\! -\! \min \! \left\{ \!\frac{C_2^{\rm q} t n }{\ttau^{2}},
 \frac{C_3^{\rm q} n t^2  }{ \ttau^{4}} \right\} \! \right) 
 \!+ \!
 2 p \exp\left( - \frac{ nt  }{2 \ttau^2} \right).
 $$
\end{lemma}
\begin{proof}
From \eq{quadratic} applied with a diagonal matrix for each pair of coordinates $a,b$ (and using the union bound):
$$
\P \left( \left\|\frac{1}{n} \sum_{i=1}^n  \sigma^{-2} \varepsilon_i^2 x_{i} x_{i}^\top
- Q \right\|_\infty\!\!\!\!\! \geqslant t M^2 \right)
\leqslant 2 p^2 C_1^{\rm q} \exp \!\left(\! -\! \min \! \left\{ \!\frac{C_2^{\rm q} t n }{\ttau^{2}},
 \frac{C_3^{\rm q} n t^2  }{ \ttau^{4}} \right\} \! \right) \! .
$$

If we use the inequality $
\P(  \| q\|_\infty \geqslant z ) \leqslant 2 p \exp( - nz^2/2M^2 \tau^2)
$, with $z = (t/2)^{1/2} \sigma M $, we get the desired result.
 \end{proof}
 
\subsection{Proof of Proposition~\ref{prop:lowdim-bolasso-pairs}}
\label{app:proofs-bolasso-2}
Following the analysis from \mysec{support}, we need to upper bound
$\P(  j \in  \hat{J}^\ast | \varepsilon)$ (probability of including a certain irrelevant variable into one of the replicated active sets),
and $
\P(      (\hat{J}^\ast)^c \cup \J   \neq \varnothing )$ (probability of missing none of the relevant variables). We first prove two lemmas about each of them. 

\begin{lemma}
\label{lemma:bolasso1}
Assume \hypreff{model}{inv}, $\tmu \leqslant \frac{\minf{\tw} \tl}{2 p^{1/2}}$
and $ \frac{n}{p} \geqslant \frac{ 256 \ttau^2}{\minf{\w}^2 \tl^2}$. We have:
$$\P(      (\hat{J}^\ast)^c \cup \J   \neq \varnothing )  \!
\leqslant 2p^2 \exp \left(- \frac{\tl^2}{2} \frac{ n}{p^2}  \right)
 +8 p n^{1/2}
 \exp\left(-
 \frac{ \minf{\tw} \tl} {8\ttau } \frac{ n^{1/2}}{p^{1/2}}
 \right).
$$
\end{lemma}
\begin{proof}
This lemma shows that all relevant variables will be selected with overwhelming probability.
From Lemma~\ref{lemma:bound}, we have that $\J \subset \hat{J}^\ast$ as soon as $\| \hw - \w
\|_2 \leqslant  \frac{p^{1/2}  \mu  + \|  q^\ast \|_2}{\lambda_{\min}(Q^\ast) }$.
Thus, if $\minf{\w}  >    
 \frac{2\mu p^{1/2}}{\lminQ } $, $
 \lambda_{\min}(Q^\ast) \geqslant \lminQ/2$, 
 $\|q - q^\ast\|_2 \leqslant \minf{\w} \lminQ / 8 $,
 and  $\|q\|_2 \leqslant \minf{\w} \lminQ / 8 $, then $\J \subset \hat{J}^\ast$.
 Thus,  we have:
 \BEAS
\P(      (\hat{J}^\ast)^c \cup \J   \neq \varnothing ) \!\!\!  
 & \leqslant &  \!\!\!
 \P\left( \lambda_{\min}(Q^\ast) \leqslant \frac{\lminQ}{2}\right) \!+ \!
 \P\left( \|q\|_2 \geqslant \frac{\minf{\w} \lminQ }{8}\right) 
 \\
 & & \hspace*{3cm} + \P\left(\|q - q^\ast\|_2 \geqslant  \frac{\minf{\w} \lminQ }{8}\right) ,
 \\
 & \leqslant & \!\!\!  2p^2 \exp \left(- \frac{ n \tl^2}{ 2 p^2 }\right)
 + 2 p \exp \left( - \frac{ n \minf{\tw}^2 \tl^2}{ 128 p \ttau^2 } \right) \\
 & & 
 \hspace*{3cm} +
  2 p \E \exp \left(- \frac{ n \minf{\tw}^2 \tl^2 \sigma^2 }{  32 p  \| \varepsilon \|_\infty^2 } \right).
 \EEAS
We thus need to bound, for some $A>0$,
\BEAS
\textstyle
\E \exp \left(- \frac{ A }{  \| \varepsilon \|_\infty^2 }\right)\!\!\!
&= \!\!\!& \textstyle
\E \exp \left(- \frac{A}{  \| \varepsilon \|_\infty^2} \right) 1_{\| \varepsilon \|_\infty \leqslant M}
+
\E \exp \left(- \frac{A}{    \| \varepsilon \|_\infty^2} \right)
1_{\| \varepsilon \|_\infty > M} ,\\
& \leqslant \!\!\!& \textstyle \exp \left(- \frac{A}{ M^2 }\right)  + \P( \| \varepsilon \|_\infty > M ) , \\
& \leqslant \!\!\!& \textstyle \exp \left(- \frac{A}{ M^2 }\right)  + 2 n \exp( - \frac{ M^2 }{ 2 \tau^2} ) 
\leqslant 3n^{1/2} \exp( - \frac{A^{1/2}}{ \tau \sqrt{2}}),
\EEAS
for $A/M^2 = A^{1/2} / \tau 2^{1/2} + \log(n^{1/2})$, leading to
$$
 2 p \E \exp\left(- \frac{ n \minf{\tw}^2 \tl^2 \sigma^2 }{  32 p  \| \varepsilon \|_\infty^2 } \right)
 \leqslant 6 p n^{1/2}
 \exp\left(-
 \frac{n^{1/2} \minf{\tw} \tl} {8\ttau p^{1/2}} 
 \right).
$$
The condition $  n \geqslant \frac{ 256 \ttau^2 p}{\minf{\w}^2 \tl^2}$ allows to combine two terms into one, leading to the desired result.
\end{proof}

\begin{lemma}
\label{lemma:toto}
\label{lemma:bolasso2}
Assume \hypreff{model}{inv} and ${j} \in \J^c$. Moreover, assume
that $\| q\|_\infty \leqslant \beta_1 M \sigma /2$ and $ \| Q - \widetilde{Q} \|_\infty \leqslant 
\frac{\beta_2 M^2}{2}$,
with  $\beta_1 \geqslant \tmu$,    $ \beta_1 \beta_2 \leqslant \frac{\tmu \tl^2 }{40 p^2 }$, 
and $\beta_2 \leqslant \tl$.
 We have:
\BEAS
\! \P(  j \notin  \hat{J}^\ast | \varepsilon) \!\!& \!\! \geqslant \!\!\! & \frac{\frac{ \tmu n^{1/2}}{32}}{
1 + \frac{ \tmu n^{1/2}}{4 \tl^{1/2}}} \exp\left[ - \frac{1}{2}
\left( \frac{8 \tmu n^{1/2} p^{1/2}   }{ \tl} 
+ \frac{ | q_{{j}} - Q_{{j},{j}^c}Q_{{j}^c,{j}^c}^{-1}q_{{j}^c }|}{\sigma n^{-1/2} Q_{{j}{j}|{j}^c}^{1/2}}
\right)^2 \right] \\
& &\!\!\!\!\!\!\! \!\!\!\!\!\!   -  \frac{16 C^{\rm BE}_2 }{\ttau^3 \tl^{1}} \frac{p^{5/2}}{\tmu n}   -   \frac{10C^{\rm BE}_2 }{\ttau^3 \tl^{1/2}} \frac{p^2}{n^{1/2}}  \! -  \!
 2 p  \exp\! \left( -  \frac{n \beta_1^2 \sigma^2}{4 \| \varepsilon \|_\infty^2}
\right)
 \!-\!   2p^2 \exp\! \left( \! -  \frac{n \beta_2^2 }{2}\right).
 \EEAS
 \end{lemma}

\begin{proof}
We follow the same approach as in the proof of Proposition~\ref{prop:missingone} in Appendix~\ref{app:proofs-lowdim-missing}.
We first assume that  
 $\| q - q^\ast \|_\infty \leqslant  \beta_1 M \sigma / 2$ and 
 $  \| Q^\ast - Q\|_\infty \leqslant \beta_2 M^2/2$  (on top of the assumptions made on $\widetilde{Q} $ and $q$).
Following the same reasoning as in Appendix~\ref{app:proofs-lowdim-missing},
${j}$ is not included if
 $$|
- q^\ast_{{j}} + Q^\ast_{{j}, {j}^c} ( Q^\ast_{{j}^c , {j}^c})^{-1} q^\ast_{{j}^c}   + \mu  Q^\ast_{ {j} , {j}^c} 
(Q^\ast_{ {j}^c , {j}^c}) ^{-1/2}
 \gamma_{\mu,Q^\ast}(q^\ast_{{j}^c })
| \leqslant \mu.
 $$
 In order to apply Berry-Esseen inequality given $\varepsilon$, we first need to
 upper bound  
$|  
Q^\ast_{ {j} , {j}^c}
(Q^\ast_{{j}^c , {j}^c})^{-1} q^\ast_{ {j}^c } -
\widetilde{Q}_{ {j} , {j}^c}
\widetilde{Q}_{{j}^c , {j}^c}^{-1} q^\ast_{ {j}^c } |
 $, using Appendix~\ref{app:perturbation}, by 
 $$
 \| Q^\ast - \widetilde{Q}\|_{\infty} \| q^\ast \|_2 \times \left( \frac{ 2 p^{1/2}}{\lminQ} + \frac{4M p }{\lminQ^{3/2}  } \right) \leqslant  \beta_1 \beta_2   \frac{6 p^{3/2} M\sigma }{\tl^{3/2}  }. $$
  Also, we need to bound, by Lemma~\ref{lemma:lipschitz} and Appendix~\ref{app:perturbation}:
  \BEAS
  & & |  
Q^\ast_{ {j} , {j}^c} (Q^\ast_{ {j}^c , {j}^c}) ^{-1/2}
  \gamma_{\mu,Q^\ast}^{{j}}(q^\ast_{{j}^c}) -
 {Q}_{ {j} , {j}^c} (Q_{ {j}^c , {j}^c}) ^{-1/2}
 \gamma_{\mu,Q}^{{j}}(q^\ast_{{j}^c}) | \\
 &  
 \leqslant & p^{1/2} \frac{\beta_2 M^2 }{2}  \lminQ^{-1} p^{1/2}
 + M^2 p^{1/2} \frac{ 4 p \| Q^\ast - Q \|_\infty}{ \mu\lminQ^2} ( p^{1/2} \mu + \|q\|_2)
  \\
  &   \leqslant & 
p\beta_2 \tl^{-1}  \left( 1  + \frac{4 p  }{\tl} 
+ \frac{ 4 }{\tl} \frac{\beta_1 p }{\tmu} \right) 
 \leqslant 
p^2 \beta_2  \frac{5  }{\tl^2} 
 \left( 1 +  \frac{\beta_1}{\tmu}  \right) 
\leqslant 
  \frac{10 \beta_1 \beta_2 p^2 }{\tl^2 \tmu}.
\EEAS
Since $\beta_1 \geqslant \tmu$,    $ \beta_1 \beta_2 \leqslant \frac{\tmu \tl^2 }{40 p^2 }$, 
and $\beta_2 \leqslant \tl$, we thus have:
$$
 |  
Q^\ast_{ {j} , {j}^c}
(Q^\ast_{{j}^c , {j}^c})^{-1} q^\ast_{ {j}^c } -
\widetilde{Q}_{ {j} , {j}^c}
\widetilde{Q}_{{j}^c , {j}^c}^{-1} q^\ast_{ {j}^c } | \leqslant \mu / 4,
$$
$$
 |  
Q^\ast_{ {j} , {j}^c} (Q^\ast_{{j}^c , {j}^c})^{-1/2}
  \gamma_{\mu,Q^\ast}^{{j}}(q^\ast_{{j}^c}) -
 {Q}_{ {j} , {j}^c} Q_{{j}^c , {j}^c}^{-1/2} 
 \gamma_{\mu,Q}^{{j}}(q^\ast_{{j}^c}) |  \leqslant \mu / 4.$$
If we let denote $A$ the event  $\{ |
- q^\ast_{{j}} +  \widetilde{Q}_{ {j} , {j}^c}
\widetilde{Q}_{{j}^c , {j}^c}^{-1} q^\ast_{ {j}^c }  + \mu  Q_{ {j} , {j}^c} 
 Q_{{j}^c , {j}^c}^{-1/2} \gamma_{\mu,Q}^{{j}}(q^\ast_{{j}^c })
| \leqslant \mu/2 \}$ and $B$ the event $\{  \| q - q^\ast \|_\infty \leqslant  \beta_1 M \sigma  / 2 \} \cap \{ \| Q^\ast - Q\|_\infty \leqslant \beta_2 M^2/2\}$, this implies that
\BEQ
\label{eq:toto1}
\P(  j \notin  \hat{J}^\ast | \varepsilon) 
  \geqslant  
\P( A   | \varepsilon ) - \P(B^c|\varepsilon).
\EEQ
We have, by concentration inequalities from Appendix~\ref{app:concentration-bolasso}:
\BEQ
\label{eq:toto2}
\P(B^c|\varepsilon) \leqslant
 2 p  \exp\left( -  \frac{n \beta_1^2 \sigma^2}{4 \| \varepsilon \|_\infty^2}
\right)
 +   2p^2 \exp \left(-  \frac{n \beta_2^2 }{2}\right).
\EEQ
Overall, if we assume the various bounds on $q$, $q^\ast$, $Q^\ast$ and $\widetilde{Q}$, to have ${j}$ excluded from the active set for the bootstrap sample, it is sufficient that $A$ is satisfied.
As in Appendix~\ref{app:proofs-lowdim-missing}, we consider a smooth version of the indicator function, and we get that the
probability of $A$, given $\varepsilon$, is greater
than
\BEQ
\label{eq:toto3}
\frac{1}{2} \P\left(  |
u_{{j}} \! -\!  \widetilde{Q}_{ {j} , {j}^c}
\widetilde{Q}_{{j}^c , {j}^c}^{-1} u_{ {j}^c }  \!- \! \mu  Q_{ {j} , {j}^c} 
 Q_{{j}^c , {j}^c}^{-1/2} \gamma_{\mu,Q}^{{j}}(u_{{j}^c })
| \leqslant \mu/4 \right)\! - \! R,
\EEQ
where $u $ is normal with mean $q$   and covariance matrix $\sigma^2\widetilde{Q}/n$,
and, from Proposition~\ref{prop:missingone},
$
R \leqslant   \frac{16 C^{\rm BE}_2 }{\ttau^3 \tl^{1}} \frac{p^{5/2}}{\tmu n}    +   \frac{10 C^{\rm BE}_2 }{\ttau^3 \tl^{1/2}} \frac{p^2}{n^{1/2}}  $ (note that we have used that $\widetilde{Q}$ is close to $Q$).

We have that given $u_{{j}^c}$, $- u_{{j}} +  \widetilde{Q}_{ {j} , {j}^c}
\widetilde{Q}_{{j}^c , {j}^c}^{-1} u_{ {j}^c }$ is normal with mean $
- q_{{j}} +  \widetilde{Q}_{ {j} , {j}^c}
\widetilde{Q}_{{j}^c , {j}^c}^{-1} q_{ {j}^c }$ and covariance matrix
$\sigma^2\widetilde{Q}_{{j}, {j} | {j}^c } /n
$. Thus, we get, using Lemma~\ref{lemma:special}:
\BEA
\label{eq:toto111}
& & 
\frac{1}{2} \P\left(  |
- u_{{j}} +  \widetilde{Q}_{ {j} , {j}^c}
\widetilde{Q}_{{j}^c , {j}^c}^{-1} u_{ {j}^c }  + \mu  Q_{ {j} , {j}^c} 
Q_{{j}^c , {j}^c}^{-1/2} \gamma_{\mu,Q}^{{j}}(u_{{j}^c })
| \leqslant \mu/4 \right) \\
\nonumber & = &  \frac{1}{2}  \E 
\P\left(  |
- u_{{j}} +  \widetilde{Q}_{ {j} , {j}^c}
\widetilde{Q}_{{j}^c , {j}^c}^{-1} u_{ {j}^c }  + \mu  Q_{ {j} , {j}^c} 
 Q_{{j}^c , {j}^c}^{-1/2} \gamma_{\mu,Q}^{{j}}(u_{{j}^c })
| \leqslant \mu/4 | u_{{j}^c} \right)  \\
\nonumber& \geqslant & \frac{1}{2} 
\frac{ \frac{ \mu / 4} { 2 \sigma n^{-1/2} \widetilde{Q}_{{j} {j} | {j}^c }^{1/2}} }{1 + 
\frac{ \mu / 4} { 2 \sigma n^{-1/2} \widetilde{Q}_{{j} {j} | {j}^c }^{1/2}} }   \times \\
\nonumber & & 
\exp\!\left[ \!
\frac{-1}{2} \!\left( 
  \left|
\frac{ \mu  Q_{ {j} , {j}^c} 
( Q_{{j}^c , {j}^c})^{-1/2} \gamma_{\mu,Q}^{{j}}(u_{{j}^c })
  } {  \sigma n^{-1/2} \widetilde{Q}_{{j}, {j} | {j}^c }^{1/2}}
\right| \!+ \!
\left|
\frac{ q_{{j}} -  \widetilde{Q}_{ {j} , {j}^c}
\widetilde{Q}_{{j}^c , {j}^c}^{-1} q_{ {j}^c }}{\sigma n^{-1/2} \widetilde{Q}_{{j}, {j} | {j}^c }^{1/2}}
\right|
\right)^2
\right].
\EEA
We have using our assumptions regarding $q$ and $\widetilde{Q}$:
$
\lminQ / 2 \leqslant  \widetilde{Q}_{{j} {j} | {j}^c }  \leqslant   2 M^2$ and  $ 
\left| Q_{ {j} , {j}^c} 
( Q_{{j}^c , {j}^c})^{-1/2} \gamma_{\mu,Q}^{{j}}(u_{{j}^c }) \right| \leqslant   M  \lminQ^{-1/2} p^{1/2} $. Moreover,
\BEAS 
|   \widetilde{Q}_{ {j} , {j}^c}
\widetilde{Q}_{{j}^c , {j}^c}^{-1} q_{ {j}^c }   -   {Q}_{ {j} , {j}^c}
 {Q}_{{j}^c , {j}^c}^{-1} q_{ {j}^c }|  & \!\!\!\leqslant \!\!\!& 
 \frac{ p^{1/2} \beta_1 M \sigma}{2}   \left(         \frac{p^{1/2} \beta_2 M^2}{\lminQ} \! +\!   \frac{ 4 M^3 p \beta_2  }{\lminQ^{3/2}}     \!\right)
 \\
  &\!\! \!\leqslant \!\!\! & 
 \frac{ 3p^{3/2} M^3 \beta_2 \beta_1 M \sigma} {\lminQ^{3/2}}    \leqslant \mu/8 ,
  \\
 | {Q}_{{j} {j} | {j}^c }^{-1/2}
 - \widetilde{Q}_{{j} {j} | {j}^c }^{-1/2}| & \!\!\!\leqslant \!\!\! & 4 \lminQ^{-3/2}
 \| Q - \widetilde{Q} \|_{\rm 2} \leqslant \frac{ 2\beta_2 M^2}{\lminQ^{3/2}}.
\EEAS
 This leads to a lower bound of the form:
\BEQ 
\label{eq:toto4}
\frac{\frac{ \tmu n^{1/2}}{32}}{
1 + \frac{ \tmu n^{1/2}}{4 \tl^{1/2}}} \exp\left[ - \frac{1}{2}
\left( \frac{8 \tmu n^{1/2} p^{1/2}   }{ \tl} 
+ \frac{ | q_{{j}} - Q_{{j}, {j}^c} Q_{{j}^c , {j}^c}^{-1}q_{{j}^c }  |}{\sigma n^{-1/2} Q_{{j}, {j}|{j}^c}^{1/2}}
\right)^2 \right].
\EEQ
By combining \eq{toto1}, \eq{toto2}, \eq{toto3}, \eq{toto111} and \eq{toto4}, we get the desired result.
\end{proof}

We can now consider the full bound using the analysis outlined in \mysec{support}, using Lemma~ \ref{lemma:concentrationbolasso}, \ref{lemma:bolasso1} and \ref{lemma:bolasso2}:
\BEAS
 \P( \hat{J}^\cap \neq \J )
&\!\!\!\!  \leqslant  \!\!\!\! & 
 m 
\P(      (\hat{J}^\ast)^c \cup \J   \neq \varnothing )  + 
 \sum_{j \in \J^c} \E ( 
\P(  j \in  \hat{J}^\ast | \varepsilon)^m ),
\\
& \!\!\!\!\leqslant \!\!\!\! & 
 2p^2 m \exp \left(- \frac{ n \tl^2}{ 2 p^2 }\right)
 +8 p n^{1/2} m
 \exp\left(-
 \frac{n^{1/2} \minf{\tw} \tl} {8\ttau p^{1/2}} 
 \right)
 +  2 p \exp \left(-  \frac{n \beta_1^2 }{2\ttau^2}  \right)
 \\
 & & + 2 p^2 C_1^{\rm q} \exp \!\left(\! -\! \min \! \left\{ \!\frac{C_2^{\rm q} \beta_2 n }{2 \ttau^{2}},
 \frac{C_3^{\rm q} n \beta_2^2  }{ 2 \ttau^{4}} \right\} \! \right) 
 + 
 2 p \exp\left( - \frac{ n \beta_2  }{4 \ttau^2} \right) \\
 & & + 
 \sum_{j \in \J^c} \E \left[ 
\P(  j \in  \hat{J}^\ast | \varepsilon)^m 1_{ \| q\|_\infty \leqslant \beta_1 M \sigma /2}
1_{ \| \widetilde{Q}-Q\|_\infty \leqslant \beta_2 M^2 / 2}  \right].
 \EEAS
 We consider $\beta_1 = \tmu p^{-1} n^{3/10}$ and $\beta_2 = p^{-1} n^{-3/10}$. We truncate
 $\| \varepsilon\|_\infty$
at $n^{1/10} \sigma$  and $ \frac{ | q_{{j}|{j}^c}|}{\sigma n^{-1/2} Q_{{j}{j}|{j}^c}^{1/2}}$ at $z$ and use Berry-Esseen inequality for  $q_{{j}|{j}^c}$, to obtain:
\begin{multline*}
 \P( \hat{J}^\cap \neq \J )
  \leqslant   m p \exp \left( - A_0 \frac{n^{1/2}}{p^{1/2}} \right) + A_1 \frac{ p^{3/2}} {n^{1/2}}
+  \exp( - z^2 / 2)
 + 
 \\
 + p \left(  1 - \frac{A_2(c)}{p^{1/2}} \exp\left[ - \frac{1}{2}
\left( \frac{8 c  }{ \tl} 
+  z 
\right)^2 \right]   + A_3(c) \frac{p^{3}}{n^{1/2}} \right)^m ,
 \end{multline*}
 with 
 $$ 
 2p\exp \left(- \frac{ n \tl^2}{ 2 p^2 }\right)
 +8  n^{1/2}
 \exp\left(-
 \frac{n^{1/2} \minf{\tw} \tl} {8\ttau p^{1/2}} 
 \right) 
\leqslant  p \exp \left( - A_0 \frac{n^{1/2}}{p^{1/2}} \right) ,
 $$
 $$
 2 p \exp \left(-  \frac{n \tmu^2 p n^{3/5}   }{2p^2\ttau^2} \right) 
  + 2 p^2 C_1^{\rm q} \exp \!\left(\! -\! \min \! \left\{ \!\frac{C_2^{\rm q} p^{-1} n^{-3/10} n }{2 \ttau^{2}},
 \frac{C_3^{\rm q} n p^{-2} n^{-3/5}  }{ 2 \ttau^{4}} \right\} \! \right) $$
 $$
 + 
 2 p \exp\left( - \frac{ n p^{-1} n^{-3/10}  }{4 \ttau^2} \right) 
 + \frac{   \ttau^3 p^{3/2}}{\tl^{3/2}} n^{-1/2}  + 2 n \exp( - n^{1/5} / 2 \ttau^2 )
\leqslant   A_1 \frac{ p^{3/2}} {n^{1/2}},
 $$
 $$
 A_3(c) \frac{p^{3}}{n^{1/2}}  =  \frac{16 C^{\rm BE}_2 }{\ttau^3 \tl^{1}} \frac{p^{3}}{\tmu n p^{1/2}}   +  \frac{10C^{\rm BE}_2 }{\ttau^3 \tl^{1/2}} \frac{p^2}{n^{1/2}}
   + 
 2 p  \exp\! \left( -  \frac{n \tmu^2 n^{3/5} p p^{-3}}{4 n^{1/5}}
\right)
 +   2p^2 \exp\! \left( \! -  \frac{n  p^{-2} n^{-3/5}}{2}\right).
 $$
 $$A_2(c) p^{-1/2}\leqslant  \frac{\frac{ \tmu n^{1/2}}{32}}{
1 + \frac{ \tmu n^{1/2}}{4 \tl^{1/2}}} .
$$ 
All these constraints lead to  the constraint tha $n p^{-6}$ should be larger than a function of $c$.
 We now need to optimize over $z$ the following quantity:
$$
 p \left(  1 - \frac{A_2(c)}{p^{1/2}} \exp\left[ - \frac{1}{2}
\left( \frac{8 c  }{ \tl} 
+  z 
\right)^2 \right]   + A_3(c) \frac{p^{3}}{n^{1/2}} \right)^m  + e^{-z^2/2}.
  $$
 If we select $z$ such that $\frac{8 c  }{ \tl} +z = \left( 2 \log \frac{\frac{A_2(c)}{p^{1/2}}}{ A_3(c) p^{3}  n^{-1/2} + \frac{\log  m}{m}} \right)^{1/2}$, which is possible if $m$  and $n$ large enough, i.e.,
 if $ m \geqslant e^{(\frac{8 c  }{ \tl} )^2}(\frac{A_2(c)}{p^{1/2}})^{-2}$ and $n \geqslant e^{(\frac{8 c  }{ \tl} )^2} ( A_3(c) p^{3} )^2 (\frac{A_2(c)}{p^{1/2}})^{-2}$,
  then we have the bound:
$$
 \left( 1 - \frac{A_2(c)}{p^{1/2}}\exp( - \frac{1}{2}(\frac{8 c  }{ \tl} +z)^2) + A_3(c) p^{3} n^{-1/2} \right)^m \leqslant \frac{1}{m},
$$
and 
$$
  e^{-z^2/2} \leqslant 
\frac{  A_3(c) p^{3}  n^{-1/2} + \frac{ \log m }{m}}{\frac{A_2(c)}{p^{1/2}}} e^{  \frac{(\frac{8 c  }{ \tl} )^2}{2}}\exp\left( - \frac{\frac{8 c  }{ \tl}}{2}\left( 2 \log \frac{\frac{A_2(c)}{p^{1/2}}}{ A_3(c) p^{3}  n^{-1/2} + \frac{\log  m}{m}} \right)^{1/2} \right).
$$
This leads to the desired bound.

\section{Proofs for bootstrapping residuals}
\label{app:proofs-residuals}

We use the following notation for the solution of the Lasso:
 $\hat{w} - \w = Q^{-1} q + \mu \hat{\alpha}$, where $\| Q \hat{\alpha} \|_\infty \leqslant 1$.
 We also denote $\Pi_X = X (X^\top X)^{-1} X^\top \in \rb^{n \times n}$ the projection matrix on the data, which leads to the following expression for the non-centered estimated residuals:
 $$
 \te = y - X \hw = X( \w - \hw) + \varepsilon = ( \idm - \Pi_X ) \varepsilon - \mu X \hat{\alpha}.
 $$
 We let denote $\hat{\nu} = \frac{1}{n} \sum_{i=1}^n \te_i$. The boostrapped responses are thus
 $y_i^\ast = \te_{i^\ast} - \hat{\nu} + \hw^\top x_i$. The bootstrapped residuals are thus
 $y_{i^\ast} + ( \hat{w}- \w)^\top x_i$, i.e.:
 $$\varepsilon_i^\ast = \left[
 \Pi_X \varepsilon + \mu X \hat{\alpha}
 \right]_i + \te_{i^\ast} - \hat{\nu}.
$$

We have the following expectations:
 \BEAS
 \E( \te_{k^\ast} | \varepsilon ) & =
 & \frac{1}{n} 1^\top \te = \frac{1}{n} 1^\top ( \idm - \Pi_X ) \varepsilon
 - \frac{\mu}{n} 1^\top X \hat{\alpha} = \hat{\nu}, \\
 \E( \varepsilon^\ast | \varepsilon ) & = & \Pi_X \varepsilon + X \mu \hat{\alpha},
 \\
 \var( \varepsilon^\ast_k | \varepsilon ) 
 & = &  \var( \te_{k^\ast} | \varepsilon ) 
 = \frac{1}{n} \te^\top \te  - \hat{\nu}^2
 = \frac{1}{n} \varepsilon^\top ( \idm - \Pi_X)\varepsilon + \mu^2 \hata^\top Q \hata   - \hat{\nu}^2,
 \\
 \E( q^\ast | \varepsilon ) & = & \frac{1}{n} \sum_{k=1}^n  \E( \te_{k^\ast} | \varepsilon )  x_k =   q + \mu Q \hata,   \\
\frac{\sigma^2}{n} \widetilde{Q} \!=\!  \var( q^\ast | \varepsilon ) \!\!& = &\!\! \frac{1}{n^2} \sum_{k=1}^n 
 \var( \hate_{k^\ast} | \varepsilon )  x_k x_k^\top 
 =  \frac{Q}{n} \left[
 \frac{1}{n} \varepsilon^\top ( \idm - \Pi_X)\varepsilon \!+ \! \mu^2 \hata^\top Q \hata \!  - \! \hat{\nu}^2
 \right].
 \EEAS
 We let denote $\gamma = 
 \frac{\sigma^{-2}}{n} \varepsilon^\top ( \idm - \Pi_X)\varepsilon + \sigma^{-2}\mu^2 \hata^\top Q \hata   - \sigma^{-2} \hat{\nu}^2$ so that $\var( q^\ast | \varepsilon ) = \sigma^2 \gamma Q / n$.

 \subsection{Concentration inequalities}
 \label{app:concentration-residuals}
 
 We need concentration inequalities for $q^\ast$ around $q$ (given $\varepsilon$) and of $s$ around 1, as well as $\hat{\nu}$ around zero.
 
 \begin{lemma}
 Assume \hypreff{model}{inv} and $t \geqslant \frac{2 \tmu p}{\tl}$. We have:
 $$
 \P \left( | \hat{\nu} | \geqslant t \sigma \right)
 \leqslant   2 \exp\left(
 \frac{-nt^2 \tl }{32 p^2 \ttau^2}
 \right).  $$
 \end{lemma}
 \begin{proof}
 We have: 
 $ \frac{1}{n} 1^\top ( \idm - \Pi_X ) \varepsilon
 = \frac{1}{n} \sum_{i=1}^n \varepsilon_i  [ ( \idm - \Pi_X ) 1 ]_i
 $
 with $ [ ( \idm - \Pi_X ) 1 ]_i
= 1 - x_i^\top Q^{-1} \left( \frac{1}{n} \sum_{k=1}^n x_k \right)
 $ is such that
 $$
 |
 [ ( \idm - \Pi_X ) 1 ]_i | \leqslant 1 + \lminQ^{-1} M^2p \leqslant  2  \lminQ^{-1} M^2 p
 = \frac{2 p}{\tl}.
 $$
 Thus, we get:
 $$
 \P \left(  \left|\frac{1}{n} 1^\top ( \idm - \Pi_X ) \varepsilon \right| \geqslant t \sigma\right)
 \leqslant 2 \exp\left(
 \frac{-nt^2 \lminQ^2 \sigma^2}{8\tau^2 M^4 p^2}
 \right) =  2 \exp\left(
 \frac{-nt^2 \tl }{8 p^2 \ttau^2}
 \right) 
 $$
  We also have
 $
\left| \frac{\mu}{n} 1^\top X \hat{\alpha}  \right| = \left| 
\mu n^{-1} \sum_{k=1}^n x_i^\top \hata \right| \leqslant \mu p  M \lminQ^{-1}
 $, hence the desired result with the extra condition on $t$.
 \end{proof}
 
 \begin{lemma}
  Assume \hypreff{model}{inv}. We have:
   $$\P( \| q^\ast - q  - \mu Q \hata \|_\infty \geqslant t M\sigma | \varepsilon )
 \leqslant 2 p \exp\left(
 \frac{-2 nt^2 } { \left(
  \frac{2 p}{\tl} \| \varepsilon \|_\infty/\sigma + \tmu  \frac{n^{1/2} p^{1/2}}{\tl^{1/2}} \right)^2 }
 \right).$$
 \end{lemma}
 \begin{proof}
 We have
 $q^\ast = q + \mu Q \hata + \frac{1}{n} \sum_{i=1}^n x_i \hate_{i^\ast} 
 $.
 Moreover, we have
 $$\| \hate \|_\infty  \leqslant \| \te \|_\infty  \leqslant \left ( 1 + \frac{M^2p}{\lminQ} \right) \| \varepsilon \|_\infty + \mu \| X \hata \|_2 
 \leqslant 
  \frac{2 p}{\tl} \| \varepsilon \|_\infty +    \frac{\tmu n^{1/2} p^{1/2} \sigma }{\tl^{1/2}}.
 $$
 We get the result by Hoeffding's inequality.
 \end{proof}
 
 \begin{lemma}
 \label{lemma:s}
  Assume \hypreff{model}{inv}, $ \frac{p}{n} \leqslant \frac{t}{4}$ and $ p \frac{\tmu^2}{\tl} \leqslant t/4$.  We have:
  $$
  \P( | \gamma - 1| \geqslant t ) \leqslant 
  2 C_1^{\rm q} \exp \!\left(\! -\! \min \! \left\{ \!\frac{C_2^{\rm q} t \ttau^{-2}}{   
\frac{1}{n} + \frac{ p^2}{n^2 \tl}
},
 \frac{C_3^{\rm q} n t^2 \ttau^{-4}}{ ( 1 - p/n)  } \right\} \! \right) + 
  2 \exp\left(
 \frac{-nt \tl }{64 p^2 \ttau^2}
 \right)
  $$
  \end{lemma}
 \begin{proof}
 We first need to derive concentration for $\frac{1}{n} \varepsilon^\top ( \idm - \Pi_X) \varepsilon$ using \eq{quadratic}. We have:
$
 \lmax( \frac{1}{n} | \idm - \Pi_X | )   \leqslant   \frac{1}{n} + \frac{ p}{n} \| \Pi_X \|_\infty
\leqslant \frac{1}{n} + \frac{ p^2}{n^2 \tl} $ and $
 \|  \idm - \Pi_X\|_F^2 = \frac{1}{n}$,
 because  $|(\Pi_X)_{ij} | = \frac{1}{n} | x_i^\top Q^{-1} x_j| \leqslant p / n \tl$. We thus obtain from \eq{quadratic}:
$$ \P \!\left( \left|  \frac{\sigma^{-2}}{n} \varepsilon^\top ( \idm - \Pi_X) \varepsilon - (1 - p/n) \right| \geqslant t \right)
\! \leqslant \! 2 C_1^{\rm q} \exp \!\left(\! -\! \min \! \left\{ \!\frac{C_2^{\rm q} t \ttau^{-2}}{   
\frac{1}{n} + \frac{ p^2}{n^2 \tl}
},
 \frac{C_3^{\rm q} n t^2 \ttau^{-4}}{ ( 1 - p/n)  } \right\} \! \right) \! .
$$
Together with $\hata^\top Q \hata \leqslant p / \lmin(Q)$, we get the desired result.
\end{proof}

\subsection{Proof of Proposition~\ref{prop:bolasso-residual}}
 \label{app:bolasso-residuals}

Following the analysis from \mysec{support}, we need to upper bound
$\P(  j \in  \hat{J}^\ast | \varepsilon)$ (probability of including a certain irrelevant variable into one of the replicated active sets),
and $
\P(      (\hat{J}^\ast)^c \cup \J   \neq \varnothing )$ (probability of missing none of the relevant variables). We first prove two lemmas about each of them. 

\begin{lemma}
\label{lemma:bolasso1-res}
Assume \hypreff{model}{inv} and $\tmu \leqslant \frac{\minf{\tw} \tl}{ p^{1/2}}$. We have:
\begin{multline*}
\P(      (\hat{J}^\ast)^c \cup \J   \neq \varnothing )  \!
\leqslant 2 p \exp \left( - \frac{ n \minf{\tw}^2 \tl^2}{ 32 p \ttau^2 } \right) +
2n\exp\left( - \frac{ n^{1/2}}{2 \ttau^2 p^{1/2}} \right)  \\
+ 2 p \exp\left(
 \frac{-   n \tl^2 \minf{\tw}^2 } { 8 \left(
  \frac{2 p^{3/4}}{\tl}n^{1/4}  + \tmu  \frac{n^{1/2} p^{1/2}}{\tl^{1/2}} \right)^2 }
 \right).
\end{multline*}
\end{lemma}
\begin{proof}
This lemma shows that all relevant variables will be selected with overwhelming probability.
From Lemma~\ref{lemma:bound}, we have that $\J \subset \hat{J}^\ast$ as soon as $\| \hw - \w
\|_2 \leqslant  \frac{p^{1/2}  \mu  + \|  q^\ast \|_2}{\lambda_{\min}(Q) }$.
Thus, if $\minf{\w}  >    
 \frac{\mu p^{1/2}}{\lminQ } $,  
 $\|q - q^\ast\|_2 \leqslant \minf{\w} \lminQ / 4 $
 and  $\|q\|_2 \leqslant \minf{\w} \lminQ / 4 $, then $\J \subset \hat{J}^\ast$.
 Thus,  we have (using results from Appendix~\ref{app:concentration-residuals}):
 \BEAS
\P(      (\hat{J}^\ast)^c \cup \J   \neq \varnothing ) \!\!\!  
 & \!\!\leqslant \!\!  &  
 \P\left( \|q\|_2 \geqslant \frac{\minf{\w} \lminQ }{4}\right) 
+ \P\left(\|q - q^\ast\|_2 \geqslant  \frac{\minf{\w} \lminQ }{4}\right) ,
 \\
 &\!\! \leqslant \!\! &  2 p \exp \left( - \frac{ n \minf{\tw}^2 \tl^2}{ 32 p \ttau^2 } \right) +
2 p \exp\left(
 \frac{-   n \tl^2 \minf{\tw}^2 } { 32 \left(
  \frac{2 p}{\tl} \| \varepsilon \|_\infty/\sigma + \tmu  \frac{n^{1/2} p^{1/2}}{\tl^{1/2}} \right)^2 }
 \right).
 \EEAS
If we truncate $\| \varepsilon \|_\infty$ at $ \sigma n^{1/4} p^{-1/4} $, then we have the bound
$$
 2 p \exp \left( - \frac{ n \minf{\tw}^2 \tl^2}{ 32 p \ttau^2 } \right) +
\exp\left( - \frac{ n^{1/2}}{2 \ttau^2 p^{1/2}} \right)
+ 2 p \exp\left(
 \frac{-   n \tl^2 \minf{\tw}^2 } { 32 \left(
  \frac{2 p}{\tl}n^{1/4} p^{-1/4} + \tmu  \frac{n^{1/2} p^{1/2}}{\tl^{1/2}} \right)^2 }
 \right),
$$
hence the desired result.
\end{proof}

\begin{lemma}
\label{lemma:toto-res}
\label{lemma:bolasso2-res}
Assume \hypreff{model}{inv} and ${j} \in \J^c$.   We have:
\BEAS
\! \P(  j \notin  \hat{J}^\ast | \varepsilon) \!\!& \!\! \geqslant \!\!\! &   -  \frac{16 C^{\rm BE}_2 }{\ttau^3 \tl^{1}} \frac{p^{5/2}}{\tmu n}   -   \frac{10C^{\rm BE}_2 }{\ttau^3 \tl^{1/2}} \frac{p^2}{n^{1/2}}  +
\frac{1}{2} 
\frac{ \frac{ \mu / 4} { 2 \sigma n^{-1/2} \widetilde{Q}_{{j} {j} | {j}^c }^{1/2}} }{1 + 
\frac{ \mu / 4} { 2 \sigma n^{-1/2} \widetilde{Q}_{{j} {j} | {j}^c }^{1/2}} }   \times \\
\nonumber & & 
\exp\!\left[ \!
\frac{-1}{2} \!\left( 
  \left|
\frac{ \mu  Q_{ {j} , {j}^c} 
( Q_{{j}^c , {j}^c})^{-1/2} \gamma_{\mu,Q}^{{j}}(u_{{j}^c })
  } {  \sigma n^{-1/2} \widetilde{Q}_{{j}, {j} | {j}^c }^{1/2}}
\right| \!+ \!
\left|
\frac{ q_{{j}} -   {Q}_{ {j} , {j}^c}
 {Q}_{{j}^c , {j}^c}^{-1} q_{ {j}^c }}{\sigma n^{-1/2} \widetilde{Q}_{{j}, {j} | {j}^c }^{1/2}}
\right|
\right)^2
\right].
 \EEAS
 \end{lemma}

\begin{proof}
We follow the same approach as in the proof of Proposition~\ref{prop:missingone} in Appendix~\ref{app:proofs-lowdim-missing} and of Proposition~\ref{prop:lowdim-bolasso-pairs} in Appendix~\ref{app:proofs-bolasso-2}:
${j}$ is not included if (note that $\widetilde{Q} = s Q$ and $Q$ are proportional matrices)
 $$|
- q^\ast_{{j}} + \widetilde{Q}_{{j}, {j}^c} ( \widetilde{Q}_{{j}^c , {j}^c})^{-1} q^\ast_{{j}^c}   + \mu  Q_{ {j} , {j}^c} 
(Q_{ {j}^c , {j}^c}) ^{-1/2}
 \gamma_{\mu,Q}(q^\ast_{{j}^c })
| \leqslant \mu.
 $$
As before, we consider a smooth version of the indicator function, and we get that the
probability of not selecting ${j}$, given $\varepsilon$, is greater
than
\BEQ
\label{eq:toto3-res}
\frac{1}{2} \P\left(  |
u_{{j}} \! -\!  \widetilde{Q}_{ {j} , {j}^c}
\widetilde{Q}_{{j}^c , {j}^c}^{-1} u_{ {j}^c }  \!- \! \mu  Q_{ {j} , {j}^c} 
 Q_{{j}^c , {j}^c}^{-1/2} \gamma_{\mu,Q}^{{j}}(u_{{j}^c })
| \leqslant \mu/4 \right)\! - \! R,
\EEQ
where $u $ is normal with mean $q$   and covariance matrix $\sigma^2\widetilde{Q}/n$,
and, from Proposition~\ref{prop:missingone},
$
R \leqslant   \frac{16 C^{\rm BE}_2 }{\ttau^3 \tl^{1}} \frac{p^{5/2}}{\tmu n}    +   \frac{10C^{\rm BE}_2 }{\ttau^3 \tl^{1/2}} \frac{p^2}{n^{1/2}}  $.

We have that given $u_{{j}^c}$, $- u_{{j}} +  \widetilde{Q}_{ {j} , {j}^c}
\widetilde{Q}_{{j}^c , {j}^c}^{-1} u_{ {j}^c }$ is normal with mean $
- q_{{j}} +  \widetilde{Q}_{ {j} , {j}^c}
\widetilde{Q}_{{j}^c , {j}^c}^{-1} q_{ {j}^c }$ and covariance matrix
$\sigma^2\widetilde{Q}_{{j}, {j} | {j}^c } /n
$. Thus, we get, using Lemma~\ref{lemma:special}:
\BEA
\label{eq:toto4-res}
& & 
\frac{1}{2} \P\left(  |
- u_{{j}} +  \widetilde{Q}_{ {j} , {j}^c}
\widetilde{Q}_{{j}^c , {j}^c}^{-1} u_{ {j}^c }  + \mu  Q_{ {j} , {j}^c} 
Q_{{j}^c , {j}^c}^{-1/2} \gamma_{\mu,Q}^{{j}}(u_{{j}^c })
| \leqslant \mu/4 \right) \\
\nonumber & = &  \frac{1}{2}  \E 
\P\left(  |
- u_{{j}} +  \widetilde{Q}_{ {j} , {j}^c}
\widetilde{Q}_{{j}^c , {j}^c}^{-1} u_{ {j}^c }  + \mu  Q_{ {j} , {j}^c} 
 Q_{{j}^c , {j}^c}^{-1/2} \gamma_{\mu,Q}^{{j}}(u_{{j}^c })
| \leqslant \mu/4 | u_{{j}^c} \right)  \\
\nonumber& \geqslant & \frac{1}{2} 
\frac{ \frac{ \mu / 4} { 2 \sigma n^{-1/2} \widetilde{Q}_{{j} {j} | {j}^c }^{1/2}} }{1 + 
\frac{ \mu / 4} { 2 \sigma n^{-1/2} \widetilde{Q}_{{j} {j} | {j}^c }^{1/2}} }   \times \\
\nonumber & & 
\exp\!\left[ \!
\frac{-1}{2} \!\left( 
  \left|
\frac{ \mu  Q_{ {j} , {j}^c} 
( Q_{{j}^c , {j}^c})^{-1/2} \gamma_{\mu,Q}^{{j}}(u_{{j}^c })
  } {  \sigma n^{-1/2} \widetilde{Q}_{{j}, {j} | {j}^c }^{1/2}}
\right| \!+ \!
\left|
\frac{ q_{{j}} -   {Q}_{ {j} , {j}^c}
 {Q}_{{j}^c , {j}^c}^{-1} q_{ {j}^c }}{\sigma n^{-1/2} \widetilde{Q}_{{j}, {j} | {j}^c }^{1/2}}
\right|
\right)^2
\right].
\EEA
By combining  \eq{toto3-res} and \eq{toto4-res}, we get the desired result. Note that if $|s-1|\leqslant 1/2$, we have
$$ \left|
\frac{ \mu  Q_{ {j} , {j}^c} 
( Q_{{j}^c , {j}^c})^{-1/2} \gamma_{\mu,Q}^{{j}}(u_{{j}^c })
  } {  \sigma n^{-1/2} \widetilde{Q}_{{j}, {j} | {j}^c }^{1/2}}
\right| 
\leqslant 2  \mu M \lminQ^{-1} p^{1/2} \sigma^{-1} n^{1/2} =
2 \tmu n^{1/2} p^{1/2}.
$$
and
$$\frac{ \frac{ \mu / 4} { 2 \sigma n^{-1/2} \widetilde{Q}_{{j} {j} | {j}^c }^{1/2}} }{1 + 
\frac{ \mu / 4} { 2 \sigma n^{-1/2} \widetilde{Q}_{{j} {j} | {j}^c }^{1/2}} } 
\geqslant 
\frac{ \frac{ \mu / 4} { 4 \sigma n^{-1/2} M  }}{1 + 
\frac{ \mu / 4} {  \sigma n^{-1/2} \tl^{1/2}} }  =
\frac{ \frac{ \tmu n^{1/2}} {16  }}{1 + 
\frac{ \tmu n^{12}} {4  \tl^{1/2}} } .
$$
\end{proof}

We can now consider the full bound using the analysis outlined in \mysec{support}, using Lemma~\ref{lemma:bolasso1-res}, \ref{lemma:bolasso2-res}, and Appendix~\ref{app:concentration-residuals}, together with truncating on the events $\|\varepsilon \|_\infty \leqslant \sigma n^{1/4}$ and $ | \gamma - 1 | \leqslant  n^{-1/3}$ (note that we can apply from Lemma~\ref{lemma:s} for $n$ large enough).
First, we need a bound on 
\BEAS
\P\!\left(
\left|
\frac{ q_{{j}} -   {Q}_{ {j} , {j}^c}
 {Q}_{{j}^c , {j}^c}^{-1} q_{ {j}^c }}{\sigma n^{-1/2} \widetilde{Q}_{{j}, {j} | {j}^c }^{1/2}}
\right| \!\geqslant\!  z \right)
& \!\!\leqslant\!\! & 
\P \!\left(
\left|
\frac{ q_{{j}} -   {Q}_{ {j} , {j}^c}
 {Q}_{{j}^c , {j}^c}^{-1} q_{ {j}^c }}{\sigma n^{-1/2}{Q}_{{j}, {j} | {j}^c }^{1/2}}
\right| \!\geqslant \! z (1- t) ^{1/2} \right) + P(| \gamma - 1| \leqslant t ), \\
& \leqslant & 
e^{-z^2 (1-t) / 2}
 + \frac{   \ttau^3 p^{3/2}}{\tl^{3/2}} n^{-1/2} + P(| \gamma - 1| \leqslant t ). 
\EEAS
 We have following the reasoning from \mysec{support}:
\BEAS
 \P( \hat{J}^\cap \neq \J )
&\!\!\!\!  \leqslant  \!\!\!\! & 
 \sum_{j \in \J^c} \E  
\P(  j \in  \hat{J}^\ast | \varepsilon)^m )
+ m 
\P(      (\hat{J}^\ast)^c \cup \J   \neq \varnothing ), \\
& \!\!\!\!\leqslant \!\!\!\! & 
p \left(  1 - \frac{A_2(c)}{p^{1/2}} \exp\left[ - \frac{1}{2}
\left(  B_0
+  z 
\right)^2 \right]   + A_3(c) \frac{p^{3}}{n^{1/2}} \right)^m  +   A_1 \frac{ p^{3/2}} {n^{1/2}} + e^{-z^2(1-t) /2},\EEAS
  with 
 \begin{multline*}
  2 p   \exp \left( - \frac{ n \minf{\tw}^2 \tl^2}{ 32 p \ttau^2 } \right) +2n\exp\left( - \frac{ n^{1/2}}{2 \ttau^2 p^{1/2}} \right) \hspace*{4cm}  \\
    + 2  p \exp\left(
 \frac{-   n \tl^2 \minf{\tw}^2 } { 32 \left(
  \frac{2 p^{3/4}}{\tl}n^{1/4}  + \tmu  \frac{n^{1/2} p^{1/2}}{\tl^{1/2}} \right)^2 }
 \right) 
\leqslant  p \exp \left( - A_0 \frac{n^{1/2}}{p^{1/2}} \right) 
 \end{multline*}
 \begin{multline*}
   2 C_1^{\rm q} \exp \!\left(\! -\! \min \! \left\{ \!\frac{C_2^{\rm q} t \ttau^{-2}}{   
\frac{1}{n} + \frac{ p^2}{n^2 \tl}
},
 \frac{C_3^{\rm q} n t^2 \ttau^{-4}}{ ( 1 - p/n)  } \right\} \! \right) + 
  2 \exp\left(
 \frac{-nt \tl }{64 p^2 \ttau^2}
 \right)  \\
 + \frac{   \ttau^3 p^{3/2}}{\tl^{3/2}} n^{-1/2}  + 2 n \exp( - n^{1/2} / 2 \ttau^2 )
\leqslant   A_1 \frac{ p^{3/2}} {n^{1/2}},
  \end{multline*}
 $$
 A_3(c) \frac{p^{3}}{n^{1/2}}  =  \frac{16 C^{\rm BE}_2 }{\ttau^3 \tl^{1}} \frac{p^{3}}{\tmu n p^{1/2}}   +  \frac{10C^{\rm BE}_2 }{\ttau^3 \tl^{1/2}} \frac{p^2}{n^{1/2}}
  $$
 $$A_2(c) p^{-1/2}\leqslant 
\frac{ \frac{ \tmu n^{1/2}} {16  }}{1 + 
\frac{ \tmu n^{12}} {4  \tl^{1/2}} } 
\mbox{ and }
B_0 \leqslant  2 \tmu n^{1/2} p^{1/2}.
$$
All these constraints lead to  the constraint tha $n p^{-6}$ should be larger than a function of $c$. The rest of proof follows along the lines of the proof of Proposition~\ref{prop:lowdim-bolasso-pairs} (note that the term 
$e^{-z^2(1-t) /2}$ instead of $e^{-z^2 /2}$ only affects the constant $A_5$).

\section{Proofs of high-dimensional results}
\label{app:proofs-highdim}

\subsection{Proof of Proposition~\ref{prop:highdim}}

From Lemma~\ref{lemma:opt2}, we obtain optimality  conditions for the solution of \eq{lasso-eq} to have the sign pattern $\t$:
\BEAS
 &&   \|   Q_{\L^c, \L}Q_{\L, \L}^{-1} q_{\L}-  q_{\L^c}     - \reg Q_{\L^c,\L} Q_{\L, \L}^{-1}  \t_{\L}
   \  \|_\infty \leqslant \reg , \\
 &&      \sign  (  \w_\J + ( Q_{\L,\L}^{-1} q_{\L} - \reg Q_{\L, \L}^{-1}  \t_{\L} )_\J  ) = \t_{\J},
\\
    &&      \sign   [ (  Q_{\L,\L}^{-1} q_{\L} - \reg Q_{\L, \L}^{-1}  \t_{\L})_{\K} ]   = \t_{\K}.
\EEAS
It is thus sufficient for $\t$ to be the sign pattern that
\BEA
 \label{eq:Aloc}  
  \forall k \in \L^c, & &   | Q_{k, \L}Q_{\L, \L}^{-1} q_{\L}-  q_{k} | \leqslant    \reg \boldsymbol{ \theta} ,\\
  \label{eq:Bloc}  
 \forall k \in \J, & &  | (Q_{\L \L} ^{-1} q_\L)_{k} | \leqslant \frac{1}{2}\mu \minf{\w},\\
 \label{eq:Cloc}  
 \forall k \in \K, & &  | (Q_{\L \L} ^{-1} q_\L)_{k} | \leqslant \mu \boldsymbol{ \theta}  Q_{kk}^{-1} .
\EEA
\eq{Aloc} occurs with probability greater than
$ 1 - 2 |\L^c| \exp \left( - \frac{ n \tmu^2  \boldsymbol{ \theta}^2 \tl_\L }
{ 8 \ttau^2 |\L|}
\right)$. \eq{Bloc} occurs with probability greater than
$ 1 - 2 |\J| \exp\left( - \frac{ n  \minf{\tw}^2 \tl^2_\L }
{ 4 \ttau^2 |\L|}
\right)$, and \eq{Cloc} occurs with probability greater than 
$ 1 - 2 |\K| \exp\left( - \frac{ n  \tmu^2 \boldsymbol{ \theta}^2   \tl^2_\L }
{ 4 \ttau^2 |\L|}
\right)$. This leads to the desired result by the union bound.

\subsection{Proof of Proposition~\ref{prop:bolasso-highdim}}
The bound is obtained simply from Proposition~\ref{prop:highdim}, Proposition~Proposition~\ref{prop:lowdim-bolasso-pairs} and Proposition~Proposition~\ref{prop:bolasso-residual}, using the union bound.

\section*{Acknowledgements}

 I would like to thank Za\"id Harchaoui, Jean-Yves Audibert and Sylvain Arlot for fruitful discussions related to this work. This work was supported by a grant from the Agence Nationale de la Recherche, France (MGA project, BLAN07-3-198092).  

 \fi

\bibliographystyle{abbrv}
\bibliography{bolasso}

\begin{thebibliography}{10}

\bibitem{bolasso}
F.~Bach.
\newblock Bolasso: model consistent {L}asso estimation through the bootstrap.
\newblock In {\em Proceedings of the International Conference on Machine
  Learning (ICML)}, 2008.

\bibitem{grouplasso}
F.~Bach.
\newblock Consistency of the group {L}asso and multiple kernel learning.
\newblock {\em Journal of Machine Learning Research}, 8:1179--1225, 2008.

\bibitem{hkl}
F.~Bach.
\newblock Exploring large feature spaces with hierarchical multiple kernel
  learning.
\newblock In {\em Advances in Neural Information Processing Systems (NIPS)},
  2008.

\bibitem{cs1}
R.~Baraniuk.
\newblock Compressive sensing.
\newblock {\em IEEE Signal Processing Magazine}, 24(4):118--121, 2007.

\bibitem{bentkus}
V.~Bentkus.
\newblock On the dependence of the {B}erry--{E}sseen bound on dimension.
\newblock {\em Journal of Statistical Planning and Inference}, 113:385--402,
  2003.

\bibitem{tsyb}
P.~J. Bickel, Y.~Ritov, and A.~Tsybakov.
\newblock Simultaneous analysis of {L}asso and {D}antzig selector.
\newblock {\em Annals of Statistics}, 2008.
\newblock To appear.

\bibitem{bonnans}
J.~F. Bonnans, J.~C. Gilbert, C.~Lemar{\'e}chal, and C.~A. Sagastizábal.
\newblock {\em Numerical Optimization Theoretical and Practical Aspects}.
\newblock Springer, 2003.

\bibitem{concentration}
S.~Boucheron, G.~Lugosi, and O.~Bousquet.
\newblock Concentration inequalities.
\newblock In {\em Advanced Lectures on Machine Learning}, volume 3176 of {\em
  Lecture Notes in Artificial Intelligence}. Springer, 2004.

\bibitem{boyd}
S.~Boyd and L.~Vandenberghe.
\newblock {\em Convex Optimization}.
\newblock Cambridge Univ. Press, 2003.

\bibitem{arcing}
L.~Breiman.
\newblock Arcing classifier.
\newblock {\em Annals of Statistics}, 26(3):801--849, 1998.

\bibitem{boosting}
P.~B{\"u}hlmann.
\newblock Boosting for high-dimensional linear models.
\newblock {\em Annals of Statistics}, 34(2):559--583, 2006.

\bibitem{bunea}
F.~Bunea, A.~Tsybakov, and M.~Wegkamp.
\newblock Sparsity oracle inequalities for the {L}asso.
\newblock {\em Electronic Journal of Statistics}, 1:169--194, 2007.

\bibitem{cs2}
E.~Cand\`es and M.~Wakin.
\newblock An introduction to compressive sampling.
\newblock {\em IEEE Signal Processing Magazine}, 25(2):21--30, 2008.

\bibitem{chen}
S.~S. Chen, D.~L. Donoho, and M.~A. Saunders.
\newblock Atomic decomposition by basis pursuit.
\newblock {\em SIAM Review}, 43(1):129--159, 2001.

\bibitem{cohen}
A.~Cohen, W.~Dahmen, and R.~DeVore.
\newblock Compressed sensing and best k-term approximation.
\newblock Technical report, IGPM Report, RWTH-Aachen, 2006.

\bibitem{lars}
B.~Efron, T.~Hastie, I.~Johnstone, and R.~Tibshirani.
\newblock Least angle regression.
\newblock {\em Annals of Statistics}, 32:407, 2004.

\bibitem{efron}
B.~Efron and R.~J. Tibshirani.
\newblock {\em An Introduction to the Bootstrap}.
\newblock Chapman \& Hall, 1998.

\bibitem{freedman}
D.~Freedman.
\newblock Bootstrapping regression models.
\newblock {\em Annals of Statistics}, 9(6):1218--1228, 1981.

\bibitem{descent}
J.~Friedman, T.~H. T, and R.~Tibshirani.
\newblock Pathwise coordinate optimization.
\newblock {\em Annals of Applied Statistics}, 1(2):302--332, 2007.

\bibitem{shooting}
W.~Fu.
\newblock Penalized regressions: the bridge vs. the {L}asso.
\newblock {\em Journal of Computational and Graphical Statistics},
  7(3):397--Ð416, 1998).

\bibitem{fu}
W.~Fu and K.~Knight.
\newblock Asymptotics for {L}asso-type estimators.
\newblock {\em Annals of Statistics}, 28(5):1356--1378, 2000.

\bibitem{fuchs}
J.-J. Fuchs.
\newblock On sparse representations in arbitrary redundant bases.
\newblock {\em IEEE Transactions on Information Theory}, 50(6):1341--1344,
  2004.

\bibitem{garrigues}
P.~Garrigues and L.~E. Ghaoui.
\newblock An homotopy algorithm for the {L}asso with online observations.
\newblock In {\em Advances in Neural Information Processing Systems (NIPS) 21},
  2009.

\bibitem{gotze}
F.~G{\"o}tze.
\newblock On the rate of convergence in the multivariate central limit theorem.
\newblock {\em Annals of Probability}, 19(2):724--739, 1991.

\bibitem{horn}
R.~Horn and C.~Johnson.
\newblock {\em {Matrix Analysis}}.
\newblock Cambridge University Press, 1985.

\bibitem{adaptivehighdim}
J.~Huang, S.~Ma, and C.-H. Zhang.
\newblock Adaptive {L}asso for sparse high-dimensional regression models.
\newblock {\em Statistica Sinica}, 18:1603--1618, 2008.

\bibitem{lounici}
K.~Lounici.
\newblock Sup-norm convergence rate and sign concentration property of {L}asso
  and {D}antzig estimators.
\newblock {\em Electronic Journal of Statistics}, 2, 2008.

\bibitem{Mallat93matching}
S.~Mallat and Z.~Zhang.
\newblock Matching pursuits with time-frequency dictionaries.
\newblock {\em IEEE Transactions on Signal Processing}, 41:3397--3415, 1993.

\bibitem{markowitz}
H.~M. Markowitz.
\newblock The optimization of a quadratic function subject to linear
  constraints.
\newblock {\em Naval Research Logistics Quarterly}, 3:111--133, 1956.

\bibitem{massart-concentration}
P.~Massart.
\newblock {\em Concentration Inequalities and Model Selection: Ecole d'\'et\'e
  de Probabilit\'es de Saint-Flour 23}.
\newblock Springer, 2003.

\bibitem{relaxedlasso}
N.~Meinshausen.
\newblock Relaxed {L}asso.
\newblock {\em Computational Statistics and Data Analysis}, 52(1):374--393,
  September 2007.

\bibitem{glasso}
N.~Meinshausen and P.~B{\"u}hlmann.
\newblock High-dimensional graphs and variable selection with the {L}asso.
\newblock {\em Annals of Statistics}, 34(3):1436--1462, 2006.

\bibitem{stability}
N.~Meinshausen and P.~B{\"u}hlmann.
\newblock Stability selection.
\newblock Technical Report 0809.2932, ArXiv, 2008.

\bibitem{yuinfinite}
N.~Meinshausen and B.~Yu.
\newblock {L}asso-type recovery of sparse representations for high-dimensional
  data.
\newblock {\em Annals of Statistics}, 37(1):246--270, 2008.

\bibitem{osborne}
M.~R. Osborne, B.~Presnell, and B.~A. Turlach.
\newblock On the lasso and its dual.
\newblock {\em Journal of Computational and Graphical Statistics},
  9(2):319--337, 2000.

\bibitem{lasso}
R.~Tibshirani.
\newblock Regression shrinkage and selection via the {l}asso.
\newblock {\em Journal of The Royal Statistical Society Series B},
  58(1):267--288, 1996.

\bibitem{martin}
M.~J. Wainwright.
\newblock Sharp thresholds for noisy and high-dimensional recovery of sparsity
  using $\ell_1$-constrained quadratic programming.
\newblock Technical Report 709, Department of Statistics, UC Berkeley, 2006.

\bibitem{wright}
F.~T. Wright.
\newblock A bound on tail probabilities for quadratic forms in independent
  random variables whose distributions are not necessarily symmetric.
\newblock {\em Annals of Probability}, 1(6):1068--1070, 1973.

\bibitem{grouped}
M.~Yuan and Y.~Lin.
\newblock Model selection and estimation in regression with grouped variables.
\newblock {\em Journal of The Royal Statistical Society Series B},
  68(1):49--67, 2006.

\bibitem{yuanlin}
M.~Yuan and Y.~Lin.
\newblock On the non-negative garrotte estimator.
\newblock {\em Journal of The Royal Statistical Society Series B},
  69(2):143--161, 2007.

\bibitem{ch_zhang}
C.-H. Zhang and J.~Huang.
\newblock The sparsity and bias of the {L}asso selection in high-dimensional
  linear regression.
\newblock {\em Annals of Statistics}, 36(4):1567--1594, 2008.

\bibitem{tong_zhang}
T.~Zhang.
\newblock Some sharp performance bounds for least squares regression with
  $\ell^1$-regularization.
\newblock {\em Annals of Statistics}, 2009.
\newblock to appear.

\bibitem{cap}
P.~Zhao, G.~Rocha, and B.~Yu.
\newblock Grouped and hierarchical model selection through composite absolute
  penalties.
\newblock {\em Annals of Statistics}, To appear, 2008.

\bibitem{Zhaoyu}
P.~Zhao and B.~Yu.
\newblock On model selection consistency of {L}asso.
\newblock {\em Journal of Machine Learning Research}, 7:2541--2563, 2006.

\bibitem{zou}
H.~Zou.
\newblock The adaptive {L}asso and its oracle properties.
\newblock {\em Journal of the American Statistical Association},
  101:1418--1429, December 2006.

\end{thebibliography}

\end{document}